\documentclass[twoside]{article}

\usepackage[dvipsnames]{xcolor}

\usepackage[accepted]{aistats2025}

\usepackage{researchpack}
\usepackage{pgfplots}
\usepackage{graphicx}

\usepackage{enumitem}
\usepackage{xspace}
\usepackage{multirow}

\usepackage[sort]{natbib}
\newcommand{\citeayp}[1]{%
  [%
    \def\citeaypseparator{}%
    \def\do##1{%
      \citeaypseparator%
      \citeauthor{##1},~\citeyear{##1}%
      \def\citeaypseparator{; }%
    }%
    \docsvlist{#1}%
  ]%
}

\usepackage{hyperref}
\hypersetup{
    colorlinks=true,
    citecolor=blue!75!black,
    linkcolor=red!75!black,
    filecolor=red!75!black,
    urlcolor=red!75!black,
}

\usepackage{dsfont}
\usepackage{comment}

\usepackage[capitalize,nameinlink,noabbrev]{cleveref}
\usepackage{soul}

\usepackage[toc,page,header]{appendix}
\usepackage{minitoc}

\usepackage{pifont}

\usepackage{thmtools}
\usepackage{thm-restate}

\declaretheorem[name=Theorem]{theorem}
\declaretheorem[name=Lemma, numberlike=theorem]{lemma}
\declaretheorem[name=Proposition, numberlike=theorem]{proposition}

\declaretheorem[name=Remark, numberlike=theorem]{remark}
\declaretheorem[name=Definition, numberlike=theorem]{definition}
\declaretheorem[name=Example]{example}

\declaretheorem[name=Definition, numbered=no]{definition*}
\declaretheorem[name=Theorem, numbered=no]{theorem*}
\declaretheorem[name=Corollary, numbered=no]{corollary*}
\declaretheorem[name=Lemma, numbered=no]{lemma*}
\declaretheorem[name=Proposition, numbered=no]{proposition*}

\crefname{equation}{}{Equations}

\newcommand{\notion}{\ensuremath{\Gamma}{\sc lr}\xspace}

\definecolor{darkgreen}{rgb}{0.0, 0.5, 0.0} 

\newcommand{\LG}[1]{}
\newcommand{\luigi}[1]{}
\newcommand{\seb}[1]{}
\newcommand{\EM}[1]{}
\newcommand{\SW}[1]{}

\newcommand{\changed}[1]{%
{#1}}
\newcommand{\cchanged}[1]{%
{#1}}

\newcommand{\HL}[1]{\textcolor{black}{#1}}

\definecolor{BrightTeal}{rgb}{0.70, 0.85, 1.00} 

\makeatletter
\def\@fnsymbol#1{\ensuremath{\ifcase#1\or \spadesuit\or \diamondsuit\or
   \mathsection\or \mathparagraph\or \|\or **\or \dagger\dagger
   \or \ddagger\ddagger \else\@ctrerr\fi}}
    \makeatother

\newcommand{\SIM}[1]{\ensuremath{\mathrm{SIm} (#1)} \xspace}

\newcommand{\SeqA}{\ensuremath{\mathsf{Seq}(\calA)}\xspace}

\newcommand{\cat}{\mathbin{\smallfrown}}

\usepackage{microtype}  %
\usepackage{mathtools}
\mathtoolsset{centercolon}

\begin{document}

\runningtitle{%
Identifiable Linear Properties of Next-token Predictors in Language Modeling}

\runningauthor{E. Marconato, S. Lachapelle, S. Weichwald, L. Gresele}

\twocolumn[

\aistatstitle{All or None: Identifiable Linear Properties of \\Next-Token Predictors in Language Modeling}

\aistatsauthor{Emanuele Marconato$^\clubsuit$
\And Sébastien Lachapelle \And  Sebastian Weichwald$^\spadesuit$ \And Luigi Gresele$^\spadesuit$
}

\aistatsaddress{University of Trento %
\And   SAIT AI Lab, Montreal \And  University of Copenhagen \And  University of Copenhagen }  ]

\begin{abstract}

\looseness-1 We analyze %
identifiability as a possible explanation for the ubiquity of linear properties across language models,
such as the vector difference between the representations of “easy” and “easiest” being parallel to that between “lucky” and “luckiest”.
For this, we ask 
whether finding a linear property in one model implies that any model that induces the same distribution has that property, too.
To answer that, we first prove an identifiability result to characterize distribution-equivalent next-token predictors, 
lifting
a diversity requirement
of previous results.
Second, based on a refinement of
relational linearity
\citeayp{paccanaro2001learning, hernandez2023linearity},
we show how many 
notions of linearity 
are amenable to our analysis.
Finally, we show that under suitable conditions, these linear properties either hold in 
all or none
distribution-equivalent 
next-token predictors.
\end{abstract}

\section{Introduction}

\begin{figure*}[!t]
    \centering
    \includegraphics[width=0.85\textwidth]{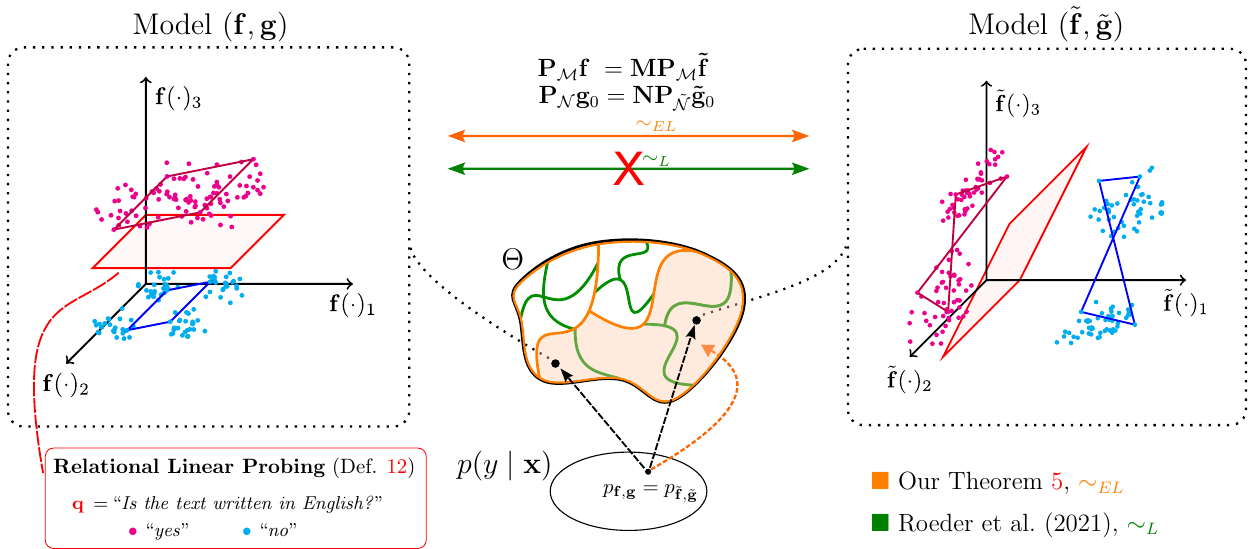}
    \caption{
    \textbf{Identifiability of linear properties.}
    Plots in the left and right dotted squares show the embeddings of two next-token predictors $(\vf, \vg), (\tilde \vf, \tilde \vg) \in \Theta$ that generate the same distribution $p_{\vf, \vg} = p_{\tilde \vf, \tilde \vg}$
    within a set of conditionals distributions
    $p(y \mid \vx)$. %
    \cref{thm:partial-identifiability} %
    proves a one-to-one correspondence (the {dashed} \textcolor{orange}{\textbf{orange}} arrow) between conditional %
    distributions and \textcolor{orange}{$\sim_{EL}$}-equivalent models (the \textcolor{orange}{\textbf{orange}} partitions of $\Theta$).
    This extends a %
    result by \citet{roeder2021linear} characterizing
    \textcolor{darkgreen}{$\sim_L$}-equivalent models 
    (\textcolor{darkgreen}{\textbf{green}} partitions of $\Theta$).
    Here $(\vf, \vg) \textcolor{orange}{\sim_{EL}} (\tilde \vf, \tilde \vg)$ 
    {while} the embedding representations  are not equal up to a linear transformation (thus $(\vf, \vg)  \textcolor{darkgreen}{\not \sim_{L}} (\tilde \vf, \tilde \vg)$), as shown by how the \textcolor{purple}{\textbf{purple}} and \textcolor{blue}{\textbf{blue}} {parallelograms} in the embeddings of the left model $(\vf, \vg)$ get distorted in those of the right model $(\tilde \vf, \tilde \vg)$%
    .
    Both models display {relational linear probing} %
    for the query
    \textcolor{red}{$\vq$}$=$``\textit{Is the text written in English?}'': one can linearly separate the embeddings of textual inputs 
    which, when concatenated with \textcolor{red}{$\vq$}, %
    have ``\textcolor{magenta}{\textit{yes}}'' as the likeliest next token, from those that yield ``\textcolor{cyan}{\textit{no}}''. 
    In \cref{prop:part-id-lin-rep-tentative}, we provide conditions under which all or none of the models in the $\textcolor{orange}{\sim_{EL}}$ equivalence class share the same linear property.}
    \label{fig:second-page}
\end{figure*}

In natural language processing, it is well-established that linear relationships between high-dimensional, real-valued vector representations of textual inputs reflect semantic and syntactic patterns.
This was motivated in seminal works~\citep{rumelhart1973model, hinton1986distributed, hinton1986learning, rumelhart1986learning,  bengio2000neural} and extensively validated in word embedding models~\citep{mikolov2013distributed, mikolov2013linguistic, pennington2014glove} as well as modern large language models trained for next-token prediction~\citep{%
burns2022discovering,
merullo2023language,
tigges2023linear,
pal2023future,
gurnee2023language,
bricken2023monosemanticity}. %

This ubiquity is puzzling, as different internal representations can produce identical next-token distributions, resulting in distribution-equivalent but internally distinct models. This raises a key question: {\bf Are the observed linear properties shared across all models with the same next-token distribution?}
Our \textbf{main result} is a mathematical proof that, under suitable conditions, %
certain linear properties hold for either all or none of the equivalent models generating a given next-token distribution. %
We demonstrate this through three main contributions.

The \textbf{first main contribution} (\cref{sec:identifiability})
is %
an identifiability result characterizing distribution-equivalent next-token predictors.
Our result is a generalization of the main theorems by~\citet{roeder2021linear} and~\citet{khemakhem2020ice}, relaxing the assumptions of diversity and equal representation dimensionality.
This result is
of independent interest for research on identifiable representation learning since our analysis is applicable to several discriminative models beyond next-token prediction%
~\citep{roeder2021linear}.

\footnotetext{\hspace{-0.89em} \textsuperscript{$\clubsuit$} Work done while at the University of Copenhagen.}  
\footnotetext{\hspace{-0.89em} \textsuperscript{$\spadesuit$} Shared last author.}

Our \textbf{second main contribution} (\cref{sec:linearity}) is to subsume several linear properties in a common framework.
We start by defining an analogue to {\em relational linearity}~\citep{paccanaro2001learning}, where the definition only relies on terms appearing in our identifiability result.
The %
key idea is to represent entities as vectors, binary relations as matrices, and to model the operation of applying a relation to an entity through matrix-vector multiplication, which yields the vector corresponding to the related entity. 
For example, in the sentence \textit{``Jimi Hendrix plays the guitar''}, the relation between the entities \textit{``Jimi Hendrix''} and \textit{``the guitar''} is signified by the word \textit{``plays''} and encoded as a matrix-vector multiplication in representation space. %
We then define relational counterparts to  linearity properties 
described and analyzed in previous %
works~\citep{
arora2016latent,
allen2019analogies,
park2023linear,
heinzerling2024monotonic}, %
thus making them amenable to our analysis.

Our \textbf{third main contribution}  (\cref{sec:implications}) is to show that under suitable conditions, these linear properties either hold in all or none of the %
models generating a given distribution.
For this, we combine the definitions in~\cref{sec:linearity} and our characterization of distribution-equivalent next-token prediction models in~\cref{sec:identifiability}. 
Identifiability theory thus
enables us to explain what linearity properties are shared across language models which are equivalent next-token predictors.
We illustrate 
this result in~\cref{fig:second-page}.

Lastly, in~\cref{sec:discussion}
we discuss
implications of our findings
and 
in~\cref{sec:related_work}
we discuss connections to related works
and future research directions%
. %

\section{Preliminaries}

\textbf{Notation}. Italic font letters denote scalars, \eg $a$; bold font lower-case letters denote vectors and sequences, \eg $\vx$; and bold font upper-case letters denote matrices, \eg $\vM$. %
{We use $\vM^+$ to denote the pseudo-inverse of $\vM$.}
We use the short-hand $[k] = \{1, \ldots, k\}$.
Given a {finite} dictionary of tokens $\calA$, the space of all possible finite sequences  (or sentences) is denoted by $\SeqA$\changed{, which is the power set of $\calA$}. 
With $\vx_{1:t}$, we denote the sub-sequence $(\vx_1, ..., \vx_t)$ of a sequence $\vx$, \ie $\vx = (\vx_1, \vx_2, ..., \vx_t, ..., \vx_T) \in \SeqA$. %
We use the symbol $\cat$ to denote the concatenation of two elements of $\SeqA$, \eg 
$\vx_1 \cat \vx_2 = (\vx_{1,1}, ..., \vx_{1,k}) \cat (\vx_{2,1}, ..., \vx_{2,l}) = (\vx_{1,1}, ..., \vx_{1,k},\vx_{2,1}, ..., \vx_{2,l}) \in \SeqA$.
For a function $\vh$, we denote its image by $\text{Im}(\vh)$.
For a $k$-dimensional subspace $\calH \subseteq \bbR^d$ spanned by an orthonormal basis $\{ \vs_1,...,\vs_k\}$, we use ${\vP_\calH= \sum_{l=1}^k \vs_l \vs_l^\top \in\bbR^{d\times d}}$ to denote the orthogonal projector onto that space.

\textbf{Next-token predictors}.
Here, we introduce the general form of next-token predictors used in our analysis.\footnote{In \cref{sec:app-transformers} we show that decoder-only transformer models {can be expressed in this form}~\citep{roeder2021linear}. %
} 
We consider models which take text sequences $\vx \in \SeqA$ of tokens $\calA$ as input. 
A next-token predictor $(\vf, \vg)$ consists of two functions: $\vf: \SeqA \to \bbR^d$ maps sequences to their representations, called \textit{embeddings}, and $\vg: \calA \to \bbR^d $ maps tokens to their representations, called \textit{unembeddings}. Let $\Theta_d$ be the set of all tuples $(\vf, \vg)$ with representation dimensionality $d$ and let $\Theta := \bigcup_{d=1}^\infty \Theta_d$ be the set of all tuples with any dimensionality $d$. 
A next-token predictor models the conditional distribution of the next-token $x_{t+1}$ given the context $\vx_{1:t}$ as
\[  \label{eq:next-token-predictor}
    p_{\vf, \vg}(x_{t+1} \mid  \vx_{1:t}) := \frac{ 
    \exp( \vf(\vx_{1:t})^\top \vg(x_{t+1}) ) 
    }{Z(\vx_{1:t})}, 
\]
where $ Z( \vx_{1:t}) := \sum_{y \in \calA} \exp[\vf(\vx_{1:t})^\top \vg(y)]$ is a normalizing constant. 
Models of the form in Equation~\eqref{eq:next-token-predictor} are trained to maximize the {conditional} log-likelihood of the data. For a data distribution $p_\calD$ over sequences $\vx \in \SeqA$, its log-likelihood is given by:
\[  \label{eq:log-likelihood}
    \textstyle
    \calL (\vf, \vg) := \bbE_{{\vx \sim}p_\calD}\left[  \sum_{t=1}^{T(\vx)-1} \log p_{\vf, \vg} (x_{t+1} | \vx_{1:t}) \right],
\]
where $T(\vx)$ indicates the length of the sequence $\vx$. 
For a fixed representation dimensionality $d$, the next-token prediction objective  can be written as
\[  \textstyle
    \label{eq:learning-objective}
    \max_{(\vf, \vg) \in \Theta_d} \;  \calL(\vf,\vg) \, . %
\]
\changed{We consider a setting where both 
$\vf$ and $\vg$ are nonparametric functions, so the model’s expressivity is determined solely by the parameter
$d$, corresponding to the dimensionality of the representation space.}

\textbf{Representation dimensionality and approximation capacity.}
In theory and for real-valued inputs, models of the form in Equation~\eqref{eq:next-token-predictor} have been proven to be universal approximators,
\ie they can approximate any conditional distribution $p(x_{t+1} \mid \vx_{1:t})$ to arbitrary precision, given a sufficiently large representation dimensionality $d$
\citep{khemakhem2020ice}; similar results may apply to next-token predictors.
In practice, even if a representation
dimensionality of $d$
may be sufficient to represent 
a given distribution well, 
different practitioners may choose models with representations of different dimensionality. The linear identifiability results by \citet{khemakhem2020ice, roeder2021linear} cannot be applied in this setting, since they consider models with equal representation dimensionality. 
Our \cref{thm:partial-identifiability} alleviates this tension between theory and practice: It characterizes identifiability of next-token predictors modeling the same conditional distribution irrespective of their representation dimensionality.

\section{Identifiability of next-token predictors}
\label{sec:identifiability}

We introduce a novel identifiability analysis for the model in Equation~\eqref{eq:next-token-predictor}. In general, a statistical model $p_\vtheta(\vx)$ parameterized by $\vtheta$ is said to be \textit{identifiable up to an equivalence relation} $\sim$ %
\changed{in the model class $\Theta$}  if 
$p_\vtheta = p_{\tilde\vtheta} \implies \vtheta \sim \tilde\vtheta$. In other words, if two parametrizations $\vtheta, \tilde\vtheta$ yield the same distribution, then they coincide under the equivalence relation $\sim$. The precise notion of equivalence depends on the problem setting. Although it is less commonly discussed, this implication can often be shown to hold also in the other direction so that $\sim$ is effectively a characterization of distribution-equivalent models, \ie $p_\vtheta = p_{\tilde\vtheta} \iff \vtheta \sim \tilde\vtheta$. 
In this section, we define an equivalence relation over tuples $(\vf, \vg) \in \Theta$ that characterizes models that entail the same next-token distribution. In other words, we want to define an equivalence relation $\sim$ over $\Theta$ such that $p_{\vf, \vg} = p_{\tilde\vf, \tilde\vg} \iff (\vf, \vg) \sim (\tilde\vf, \tilde\vg)$; we then say $(\vf, \vg)$, $(\tilde\vf, \tilde\vg)$ are $\sim$-equivalent. %
Our characterization applies to pairs of models having different dimensionalities, i.e., $d \not=\tilde d$, as opposed to previous works only considering $d =\tilde d$ \citep{khemakhem2020variational, khemakhem2020ice, roeder2021linear, lachapelle2023synergies}. 

Previous works %
have shown that, \changed{under an assumption known as {\em variability}
\citep{khemakhem2020variational, khemakhem2020ice, lachapelle2023synergies} or {\em diversity condition} \citep{roeder2021linear}}, %
the representations extracted by distributionally-equivalent models
are equal up to a linear invertible transformation.
Intuitively, %
requires that at least one model $(\vf, \vg)$ ``spans'' the whole representation space. To formally state the condition, we define 
the linear space spanned by the 
image $\text{Im}(\vh) \subseteq \bbR^d$
of a function
$\vh$ as
$\SIM \vh :=  \mathrm{span}\{\vv \mid \vv \in\mathrm{Im}(\vh)\}$.
Additionally, for the unembeddings, we choose an arbitrary token $y_0 \in \calA$ as a {\em pivot} for the remainder of the paper and define:
\[
    \vg_0(y) := \vg(y) - \vg(y_0)
\]
for all tokens $y\in \calA$. 
The diversity condition %
can then be defined as follows:
\begin{definition}[Diversity condition]
\label{def:diversity-condition}
    We say that a model $(\vf, \vg)$ with representation dimensionality $d$ satisfies the diversity condition if $\SIM{\vf} = \SIM{\vg_0} = \bbR^d$.
\end{definition}

Intuitively, the diversity condition states that the dimension of the spaces spanned by the output of $\vf$ and $\vg_0$ match the dimension of the representation space. When both the diversity condition and $d=\tilde{d}$ hold, existing identifiability results for models of the form in \changed{Equation~\eqref{eq:next-token-predictor}} (presented in \cref{cor:linear-id}) guarantee equivalence of representations up to an invertible linear transformation \citep{khemakhem2020ice, roeder2021linear}. 
Next, we show how to relax these two conditions via a more permissive equivalence relation.

\begin{figure*}[!t]
    \centering

    \includegraphics[width=\textwidth]{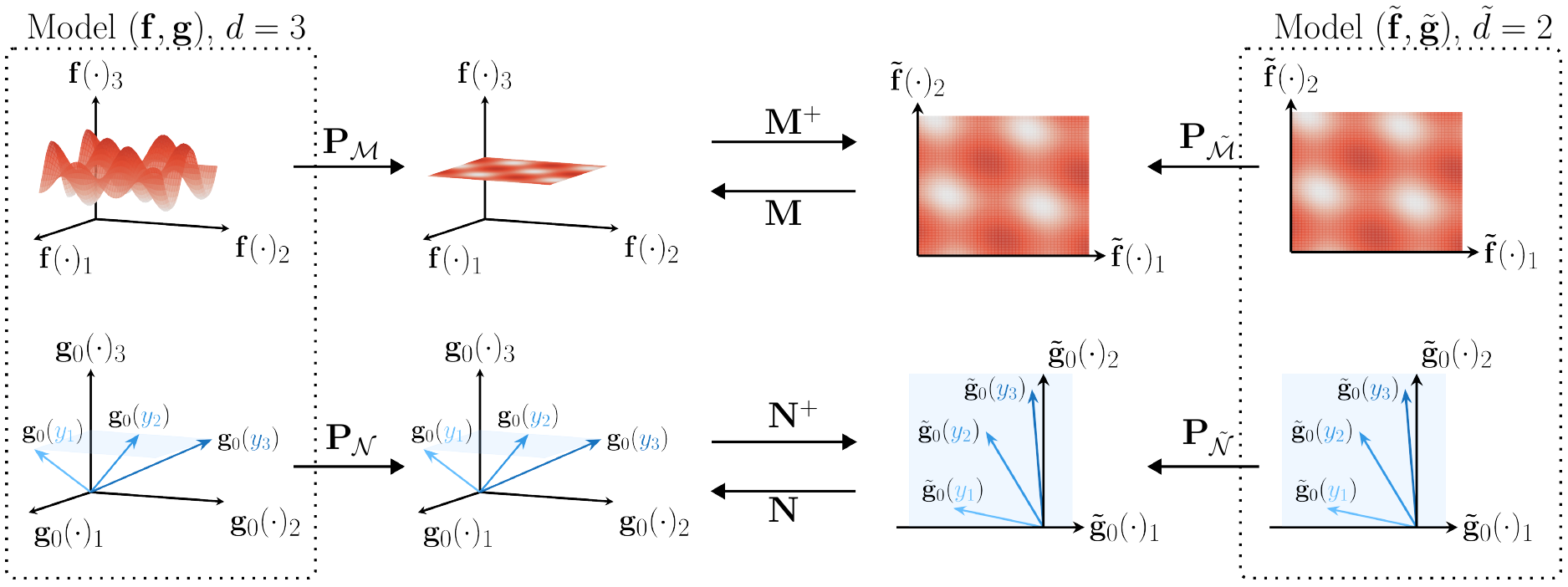}
    
    \caption{\textbf{Illustration of the $\sim_{EL}$ equivalence relation.} \looseness-1
    (\textbf{Left}) In the leftmost model, $(\vf, \vg)$, the embeddings lie on a manifold $\mathrm{Im}(\vf)\subsetneq \bbR^3$, yielding $\SIM{\vf} = \bbR^3$. 
    To ease visualization, $\mathrm{Im}(\vf)$ is plotted as a \changed{continuous} manifold in the figure, although in practice textual inputs are discrete.
    The unembeddings 
    lie on a two-dimensional space, $\SIM{\vg_0} \cong \bbR^2$, drawn in \textcolor{BrightTeal}{\textbf{light blue}}. 
    Consequently the projectors $\vP_\calM$ and $\vP_\calN$ map onto a two-dimensional subspace, \ie $\vP_\calM = \vP_\calN = \vP_\calG$. 
    (\textbf{Right}) The rightmost model, $(\tilde{\vf}, \tilde{\vg})$, represents both the embeddings and the unembeddings in a two-dimensional space.  We therefore have $\SIM{\tilde{\vf}} = \SIM{\tilde{\vg_0}} = \bbR^2$, which implies that $\vP_{\tilde{\calM}} = \vP_{\tilde{\calN}} = \vI$. Thus applying these projection matrices to embeddings and unembeddings leaves them unchaged (top-right and bottom-right grids). 
    \textbf{(Center)}
    The equivalence relation $\sim_{EL}$ specifies that both $\vP_\calM \vf$ and $\vP_{\tilde{\calM}} \tilde{\vf}$, as well as $\vP_\calN \vg$ and $\vP_{\tilde{\calN}} \tilde{\vg}$, are related by linear invertible transformations defined by the matrices $\vM, \vN \in \bbR^{3 \times 2}$.
    }
    \label{fig:illustration-kidf}
\end{figure*}

\subsection{Effective complexity of the model}
\label{sec:eff-comp}
To generalize previous results to settings where the diversity condition may not hold, we introduce the notion of {\em effective complexity} of a model. %
Starting from the conditional distribution captured by the model $(\vf, \vg)$, we have that for every $y_0 \in \calA$,%
\begin{align*}
    p_{\vf,\vg}(y \mid \vx) &\propto \exp(\vf(\vx)^\top\vg(y)) \\
    &\propto \exp(\vf(\vx)^\top\vg(y))\exp(-\vf(\vx)^\top\vg(y_0)) \\ 
    & = \exp(\vf(\vx)^\top\vg_0(y))) \, ,
\end{align*}
indicating that the conditional distribution ${p_{\vf,\vg}(y \mid \vx)}$ is fully determined by the dot product $\vf(\vx)^\top\vg_0(y)$.  Denote by $\vP_\calF$ and $\vP_\calG$ the orthogonal projectors onto $\calF:=\SIM\vf$ and  $\calG:=\SIM {\vg_0}$, respectively. 
Since both $\vf(\vx) = \vP_{\calF} \vf(\vx)$ and $\vg_{0}(y) =\vP_{\calG} \vg_{0}(y)$, the dot product between the two can be evaluated as:
\begin{align}
    \vf(\vx)^\top \vg_{0}(y) &= (\vP_\calF \vf(\vx))^\top \vP_{\calG} \vg_{0}(y) \\
    &= \vf(\vx)^\top \vP_{\calF} \vP_{\calG} \vg_{0}(y) \, ,
\end{align}
where we used the fact that $\vP_\calF^\top = \vP_\calF$, as a property of orthogonal projectors.
In general, $\vP_{\calF}$ and $\vP_{\calG}$ may not commute \citep{rehder1980projections}. We consider the subspaces: %
\[  \textstyle  \label{eq:left-right-projectors}
    \calM := \mathrm{Im}(\vP_{\calF} \vP_{\calG}), \quad \calN := \mathrm{ker} (\vP_{\calF} \vP_{\calG})^\perp \, ,
\]
which will be central to our characterization. 
Only when $\vP_\calF$ and $\vP_\calG$ commute, we have that $\calM = \calN = \SIM{\vf} \cap \SIM{\vg_0}$ \citep{rehder1980projections}.
In general, $\vP_\calM \vf \neq \vf$ and  $\vP_\calN \vg_0 \neq \vg_0$, as shown in the following example. 

\begin{example}
    Let $\{\ve_1, \ve_2, \ve_3\}$ be an orthogonal basis of $\bbR^3$. Take $\calF = \mathrm{span}(\ve_1, \ve_2)$ and $\calG =\mathrm{span}(\ve_1, \ve_3)$. 
    \changed{Then it follows that
    $\vP_\calF = \ve_1 \ve_1^\top + \ve_2 \ve_2^\top$ and $\vP_\calG = \ve_2 \ve_2^\top + \ve_3 \ve_3^\top$.}
    In this case, $\vP_\calF$ and $\vP_\calG$ commute, so $\calM = \calN = \calF \cap \calG = \mathrm{span} (\ve_1)$. %
\end{example}

Also, we have the following properties: 
\begin{restatable}{lemma}{lemmaprojectors}
\label{lemma:projectors}
    Given the orthogonal projectors $\vP_\calF$ and $\vP_\calG$, and the orthogonal projectors $\vP_\calM$ and $\vP_\calN$ onto, respectively, the spaces $\calM$ and $\calN$, defined as in  Equation~\eqref{eq:left-right-projectors}, the following holds: (i) 
    $\mathrm{dim}(\calM) =\mathrm{dim} (\calN) = \mathrm{dim}(\SIM{\vf}) - \mathrm{dim}(\SIM{\vf} - \SIM{\vg_0}^\perp)
    $;
    (ii) $\calM \subseteq \SIM{\vf}$ and $\calN \subseteq \SIM{\vg_0}$;
    and (iii)
    \[  \textstyle  \label{eq:equiv-dot-prod}
    \begin{aligned}
        \vf(\vx)^\top \vg_{0}(y) &= \big(\vP_\calM \vf(\vx) \big)^\top  \vP_\calN \vg_{0}(y) \, . %
    \end{aligned}
    \]
\end{restatable}
All proofs can be found in \cref{sec:proofs-sec4}.
As a consequence of this lemma, we can view the projections $\vP_\calM \vf$ and $\vP_\calN \vg$ as the parts of $\vf$ and $\vg$ that are effectively retained when evaluating the dot product on the left-hand side of Equation \eqref{eq:equiv-dot-prod}. 
\cchanged{This also means the dot product depends solely on the projection of the embeddings onto $\calM$ and the unembeddings onto $\calN$; components of the embeddings and unembeddings which are orthogonal to these subspaces do not contribute.}
The dimensionality of $\calM$ (and $\calN$) can be viewed as a measure of model complexity since, intuitively, the larger $\dim(\calM)$ is, the more expressive the resulting model, and thus the more complex the relationship between $y$ and $\vx$, captured by $p(y \mid \vx)$, can be. For this reason, we call $\dim(\calM)$ the \textit{effective complexity} of the model $(\vf, \vg)$. %
Note that
the effective complexity of %
$(\vf, \vg)$ 
is less than or equal
to
the representation dimensionality $d$.

\subsection{Extended linear equivalence relation}

We now
introduce an equivalence relation among models with potentially different representation dimensionality, building on our notion of effective complexity.
We consider next-token prediction models
$(\vf, \vg)$ and $(\tilde{\vf}, \tilde{\vg})$, and spaces $\tilde{\calF} := \SIM{\tilde \vf}$, $\tilde{\calG} := \SIM {\tilde \vg_0}$, $\tilde{\calM} := \mathrm{Im} (\vP_{\tilde \calF} \vP_{\tilde \calG}) $ and $\tilde{\calN} := \mathrm{ker} (\vP_{\tilde \calF} \vP_{\tilde \calG})^\perp$, as introduced in~\cref{sec:eff-comp}.

\begin{restatable}[Extended linear equivalence]{definition}{klinequiv}
\label{def:klinear-equiv}

    Two models $(\vf, \vg)$ and $(\tilde{\vf}, \tilde{\vg})$ are extended-linearly equivalent, %
    if both (i) $\mathrm{dim}(\calM) = \mathrm{dim}(\tilde{\calM})$ and
    (ii) there exist 
    two full-rank matrices $\vM, \vN \in \bbR^{d \times \tilde d}$ defining, respectively, invertible transformations from
    $\calM$ to  $\tilde \calM$, and from $\calN$ to $\tilde \calN$, such that $\vM^\top \vN = \vP_{\tilde \calM} \vP_{\tilde \calN}$ and 
    \begin{align}
            \vP_{\calM} {\vf}(\vx)_{\phantom{0}} &= \vM  \vP_{\tilde \calM} \tilde{\vf}(\vx)_{\phantom{0}} \label{eq:linear-transf-f} \\
            \vP_{\calN} {\vg_0}(y) &= \vN \vP_{\tilde \calN} \tilde{\vg}_0(y) \, , \label{eq:linear-transf-g} 
    \end{align}
    for all $y \in \calA, \vx \in \SeqA$. We denote this relation by $(\vf, \vg) \sim_{EL} (\tilde{\vf}, \tilde{\vg})$.
\end{restatable}

The above equivalence relation generalizes the linear equivalence already known in the literature \citep{roeder2021linear, khemakhem2020ice} to that of two models which can be related to each other on subspaces of dimension ${ \mathrm{dim}(\calM) \leq \min\{d, \tilde d\}}$. 
It shows that, after projecting the representations to suitable equal-dimensional subspaces, namely $\calM, \calN, \tilde\calM$, and $\tilde\calN$, we can find an invertible linear transformation relating them. 
Here, the dimensions of $\calM$ and $\tilde \calM$ are equal, requiring that two equivalent models share the same \textit{effective complexity}. %
Furthermore, 
\cchanged{models that are $\sim_{EL}$-equivalent encode the same dot-product:}

\begin{restatable}{proposition}{elpreserves}
\label{prop:elpreserves}
    If $(\vf, \vg) \sim_{EL} (\tilde\vf, \tilde\vg)$, then
    \[
    \cchanged{
    \vf(\vx)^\top \vg_0(y) = \tilde \vf(\vx)^\top \tilde \vg_0(y)  \, .
    }
    \]
\end{restatable}

As a consequence, models in the $\sim_{EL}$ equivalence class also satisfy $p_{\vf, \vg} = p_{\tilde{\vf}, \tilde{\vg}}$. %
In \cref{sec:app-equivalence-relation},
we prove that~\Cref{def:klinear-equiv} is an equivalence relation and we provide the explicit form of the matrix $\vN$.
The extended linear equivalence relation is illustrated in \cref{fig:illustration-kidf}.

\subsection{Identifiability of next-token predictors}

The following theorem provides a characterization of models generating the same conditional probability distribution (i.e., distributionally-equivalent next-token predictors):

\begin{restatable}{theorem}{partialidf}
\label{thm:partial-identifiability}
    For all $(\vf, \vg), (\tilde\vf, \tilde\vg) \in \Theta$, with representation dimensions $d$ and $\tilde d$ (not necessarily equal), %
    \[  \textstyle
        p_{\vf, \vg} = p_{\tilde{\vf}, \tilde{\vg}} \iff
        (\vf, \vg) \sim_{EL}(\tilde \vf, \tilde \vg) \, .
    \]
\end{restatable}

In words, %
there is a one-to-one correspondence between the set of conditional probability distributions expressed in Equation~\eqref{eq:next-token-predictor} and the set of equivalence classes entailed by $\sim_{EL}$ (cf.\@ \cref{fig:second-page}): models $(\vf, \vg)$, $(\tilde\vf, \tilde\vg)$ which are {$\sim_{EL}$}-equivalent can be mapped to a single conditional probability distribution {$p_{\vf, \vg}$}. 

Interestingly,~\cref{thm:partial-identifiability} highlights the fact that our notion of effective complexity, %
defined in~\cref{sec:eff-comp} 
for next-token prediction models, is a well-defined complexity measure for a distribution $p_{\vf, \vg}$, in the sense that it does not depend on the specific choice of embedding and unembedding functions $(\vf, \vg)$. Indeed, the result implies that 
\[p_{\vf, \vg} = p_{\tilde\vf, \tilde\vg} \implies \dim(\calM) = \dim(\tilde\calM) \, .\]

Furthermore, as a special case of~\cref{thm:partial-identifiability}, when the diversity condition holds and $d=\tilde{d}$ we recover known results on linear identifiability: %

\begin{restatable}[Adapted from~\citep{roeder2021linear}]{corollary}{roedeo}
    \label{cor:linear-id}
    For all $(\vf, \vg), (\tilde\vf, \tilde\vg) \in \Theta_d$ such that $(\vf, \vg)$ satisfies the diversity condition (\cref{def:diversity-condition}), we have
    \[
        p_{\vf, \vg} = p_{\tilde{\vf}, \tilde{\vg}} \implies
        (\vf, \vg) \sim_{L} (\tilde{\vf}, \tilde{\vg}) \, ,
    \]
    where, by definition, $(\vf, \vg) \sim_L (\tilde\vf, \tilde\vg)$ if and only if there exists an invertible matrix $\vM \in \bbR^{d\times d}$ such that for all $y \in \calA, \vx \in \SeqA$ we have
    \[\vf(\vx) = \vM\tilde\vf(\vx)\ \text{and}\ \vg_0(y) = \vM^{-\top}\tilde\vg_0(y) \, .\]
\end{restatable}

\cchanged{In \cref{sec:counter-example-diversity}, we provide an example about non-linear distortions that can arise in models that are $\sim_{EL}$-equivalent but not diverse (\cref{def:diversity-condition}).}

\subsection{Implications for empirical practice} %
\label{sec:implications-empirical}
\paragraph{Implications for trained models.} %
Suppose that $(\vf, \vg) \in \Theta_d$ and $(\tilde\vf, \tilde\vg) \in \Theta_{\tilde d}$ are both global maximizers (in their respective model classes $\Theta_d$ and $\Theta_{\tilde d}$) of the objective in Equation~\eqref{eq:learning-objective}. If both models have enough capacity to represent the ground-truth data distribution $p_\calD$, then they necessarily represent the same distribution, i.e., $p_{\vf, \vg} = p_\calD = p_{\tilde{\vf}, \tilde{\vg}}$.
By \cref{thm:partial-identifiability}, we can thus conclude that $(\vf, \vg) \sim_{EL} (\tilde\vf, \tilde{\vg})$. If these assumptions held in practice, this would imply that all models trained with sufficient capacity on a given data distribution $p_\calD$ will be $\sim_{EL}$-equivalent.\footnote{We neglect optimiziation issues such as local minima for ease of exposition.}  %
This analysis relies crucially on the assumption that $\Theta$ has enough capacity to represent $p_\calD$. In practice, this might not hold for at least two reasons. First, we typically train with a fixed representation dimension $d$ which limits the expressivity of the model; moreover, despite the universal approximation guarantee, there might not even exist a sufficiently large $d$ such that the model can express $p_\calD$ \textit{exactly}. %
\changed{Secondly, all distributions that can be represented by $\Theta$ put nonzero probability mass on all text sequences (because of the exponential on the RHS of Equation~\eqref{eq:next-token-predictor}), whereas %
under the ground-truth distribution $p_\calD$ describing, e.g., text on the internet, several sequences will have zero probability. %
Models of the form in Equation~\eqref{eq:next-token-predictor} are thus inherently misspecified in such cases.}
Modeling these settings might thus require an extension of the current theoretical framework~\citep{nielsen2024challenges}.

\paragraph{Different token vocabularies.}
\cchanged{Our analysis is also restricted to next-token predictors that share the same token vocabulary $\calA$. We hypothesize that, for two models with different token vocabularies $\calA$ and $\tilde \calA$, 
our results may be extended to prove a $\sim_{EL}$-equivalence relation restricted to the shared tokens $\calA \cap \tilde \calA$, under suitable conditions on the next-token probabilities.}

\section{Linear properties}
\label{sec:linearity}

In the previous section, we established identifiability results for next-token predictors. Here, we turn to precisely defining the linear properties we will focus on.
These 
characterize how a given model $(\vf,\vg)$ represents different inputs---as in our opening example, describing a geometric relationship (parallelism) among the vector differences between the embeddings of two different inputs (\textit{``lucky''} and \textit{``luckiest''}) and that between two further inputs (\textit{``easy''} and \textit{``easiest''}). 
Importantly, these linear properties are not to be confused with the linear equivalence class $\sim_L$ in \cref{def:klinear-equiv}, which instead describes how {\em different} models represent the {\em same} data distribution.
In \cref{sec:implications}, we will combine the linear properties defined here with the identifiability results of~\cref{sec:identifiability} to determine which linear properties hold for all models in a given equivalence class.
See also \cref{fig:second-page} for an illustration.

Our analysis focuses on embeddings $\vf(\vs) \in \SIM{\vf}$ and unembeddings $\vg_0(y) \in \SIM{\vg_0}$, 
allowing us to define relational linear properties solely in terms of the quantities described in our identifiability result:
our analysis is thus agnostic to assumptions on the data-generating process underlying natural language and it does not require positing unobserved variables.
\changed{In principle,  the linear properties we will define can apply to any collection of strings---for example, the difference between the unembeddings of {\em ``1fv0sywi''} and {\em ``eg2op3te''} could be parallel to the difference between those of {\em ``tgsqil2h''} and {\em ``khdo5zof''}. 
As we will argue, such parallel structures imply certain symmetries in a model’s conditional next-token probabilities.
Our work is motivated by the commonly observed instances of linear properties involving collections of semantically meaningful strings, where symmetries in next-token probabilities likely reflect regularities in human-produced text.}

\subsection{Parallel vectors}
\label{sec:parallel-vecs}
\changed{We begin with a definition of vector parallelism. This is motivated by recent empirical findings that differences in semantically or syntactically related token unembeddings often exhibit parallelism, }
such as $\vg(\textit{``easy''}) -\vg (\textit{``easiest''})$ being parallel to $\vg(\textit{``lucky''}) - \vg(``\textit{luckiest''})$.\footnote{Similar properties had previously been observed in \changed{word embedding models}~\citep{mikolov2013efficient,park2023linear}.}
Central to our theory will be the following definition of parallelism in a subspace $\Gamma \subseteq \bbR^d$:

\begin{definition}[Parallelism in $\Gamma$]
\label{def:s-parallelism}
    We say that two vectors $\vgamma, \vgamma' \in \bbR^d$ are parallel in $\Gamma$ if there exists $\beta \neq 0$ (see \cref{remark:parallelism}) such that
    $\vP_\Gamma \vgamma = \beta \cdot \vP_\Gamma \vgamma' $.
\end{definition}

We next show that parallel vectors induce similar log ratios of conditional probabilities, as noted in \citep{park2023linear, jiang2024origins}:

\begin{restatable}{lemma}{lemparallelism}
    \label{lemma:parallelism}
    Consider a model $(\vf, \vg) \in \Theta$. For $y_0, y_1, y_2, y_3 \in \calA$, 
    the difference vectors $\vg (y_1) - \vg (y_0)$ and $\vg (y_3) - \vg (y_2)$ are parallel in $\calN$ if and only if
    there exists $\beta \neq 0, \text{ s.t. } \forall \vs \in \SeqA$ %
    \[  \label{eq:parallelism}
        \log \frac{p_{\vf, \vg}(y_0 \mid \vs)}{p_{\vf, \vg}(y_1 \mid \vs)} = \beta \cdot \log \frac{p_{\vf, \vg}(y_2 \mid \vs)}{p_{\vf, \vg}(y_3 \mid \vs)} \, .
    \]
\end{restatable}

That is, parallel difference vectors for token pairs
correspond to
proportional likelihood ratios between the tokens in each token pair. \changed{Notice that, as in \cref{def:s-parallelism}, the difference vectors are parallel only in the space $\calN$. 
This implies that the components outside $\calN$ for two $\calN$-parallel vectors are allowed to not be parallel.}
\changed{All proofs for the results presented in this section are provided in \cref{sec:proofs-sec3}.}

\begin{figure}[!t]
    \centering
    \resizebox{0.5\textwidth}{!}{

    \begin{tikzpicture}

    \pgfdeclarehorizontalshading{diagonal}{100bp}{color(18bp)=(cyan); color(49bp)=(magenta)}

    \shade[top color=magenta,  rounded corners, bottom color=magenta,  opacity=0.5, transform canvas={rotate around={15:(0,0)}}] (0.05,-0.2) rectangle (0.35,0.9);
    
    \shade[top color=magenta!80,  rounded corners, bottom color=magenta!80,  opacity=0.5, transform canvas={rotate around={15:(0,0)}}] (0.36,-0.3) rectangle (0.71,1.7);
    
    \shade[top color=magenta!50,  rounded corners, bottom color=magenta!50,  opacity=0.5, transform canvas={rotate around={15:(0,0)}}] (0.72,-0.5) rectangle (1.07,2.8);
    
    \shade[top color=magenta!20,  rounded corners, bottom color=magenta!20,  opacity=0.5, transform canvas={rotate around={15:(0,0)}}] (1.08,-0.6) rectangle (1.43,3.5);
    
    \shade[top color=cyan!20,  rounded corners, bottom color=cyan!20,  opacity=0.5, transform canvas={rotate around={15:(0,0)}}] (1.44,-0.65) rectangle (1.79,3.9);
    
    \shade[top color=cyan!50,  rounded corners, bottom color=cyan!50,  opacity=0.5, transform canvas={rotate around={15:(0,0)}}] (1.80,-0.7) rectangle (2.15,3.9);
    
    \shade[top color=cyan!80,   rounded corners, bottom color=cyan!80,  opacity=0.5, transform canvas={rotate around={15:(0,0)}}] (2.16,-0.75) rectangle (2.51,3.8);
    
    \shade[top color=cyan, rounded corners, bottom color=cyan,  opacity=0.5, transform canvas={rotate around={15:(0,0)}}] (2.52,-0.8) rectangle (2.87,3.7);

    \shade[top color=cyan!90!blue!80, rounded corners, bottom color=cyan!90!blue!80,  opacity=0.5, transform canvas={rotate around={15:(0,0)}}] (2.88,-0.9) rectangle (3.23,3.6);

    \shade[top color=cyan!80!blue!70, rounded corners, bottom color=cyan!80!blue!70,  opacity=0.5, transform canvas={rotate around={15:(0,0)}}] (3.24,-1) rectangle (3.59,3.5);

    \draw[thick,->] (0,0) -- (4,0) node[anchor=north west] {$\vf(\cdot \cat \vq)_1$};
    \draw[thick,->] (0,0) -- (0,4) node[anchor=south east] {$\vf(\cdot \cat \vq)_2$ };
    \draw[->, thick,blue] (0,0) -- (2.87*0.967,0.259*2.87) node[above right] {\large $\vg_{n}(\textit{``yes''})$};

    \fill[black] (0.5,1.5) circle (1.5pt) node[above] {$\vf(\vs_1 \cat \vq)$};
    \draw[thick, dashed] (0.5,1.5) -- (0.8415,0.2255)
    node[above right] {\footnotesize $0.3$}
    ;
    
    \fill[black] (1.25,3) circle (1.5pt) node[above] {$\vf(\vs_2\cat \vq)$};
    \draw[thick, dashed] (1.25,3) -- (1.9167,0.5134) node[above right] {\footnotesize $0.7$};

    \draw[->,thick,decorate,decoration={snake,amplitude=.4mm,segment length=2mm,post length=1mm}, red] 
    (4.3,1.5) -- (5.8,1.5) node[midway, above] {\color{black} \large (Def. \ref{def:relational-linear-subspaces})};

    \begin{scope}[shift={(6.5,0)}]

        \shade[top color=cyan!80, rounded corners, bottom color=cyan!80,  opacity=0.5, transform canvas={rotate around={260:(8,1.8)}}] (-0.21,-0) rectangle (0.31,4.1);
        
        \shade[top color=cyan!50,  rounded corners,
        bottom color=cyan!50,  opacity=0.5, transform canvas={rotate around={260:(8,1.8)}}] (0.32,0.05) rectangle (0.82,4.2);
        
        \shade[top color=cyan!20,  rounded corners, bottom color=cyan!20,  opacity=0.5, transform canvas={rotate around={260:(8,1.8)}}] (0.83,0.1) rectangle (1.33,4.3);
        
        \shade[top color=magenta!20,  rounded corners, bottom color=magenta!20,  opacity=0.5, transform canvas={rotate around={260:(8,1.8)}}] (1.34,0.2) rectangle (1.84,4.4);
        
        \shade[top color=magenta!50, bottom color=magenta!50, rounded corners,  opacity=0.5, transform canvas={rotate around={260:(8,1.8)}}] (1.85,0.3) rectangle (2.35,4.5);
        
        \shade[top color=magenta!80, bottom color=magenta!80,  rounded corners,  opacity=0.5, transform canvas={rotate around={260:(8,1.8)}}] (2.36,0.4) rectangle (2.86,4.55);
        
        \shade[top color=magenta, bottom color=magenta,    rounded corners, opacity=0.5, transform canvas={rotate around={260:(8,1.8)}}] (2.87,0.5) rectangle (3.37,3);
    
        \draw[->, thick,blue]  (0,0) -- (0.174*4.2, 0.985*4.2)  node[right] {\large $ \vg_{o}(\textit{``English''})
        $};
    
        \draw[thick,->] (0,0) -- (4,0) node[anchor=north west] {$\vf(\cdot)_1$ };
        \draw[thick,->] (0,0) -- (0,4) node[anchor=south east] {$\vf(\cdot)_2$};

        \fill[black] (2.5,2.5) circle (1.5pt) node[above] {$\vf(\vs_2)$};
        \draw[thick, dashed] (2.5, 2.5) -- (0.5042, 2.853) node[above right] {\footnotesize $0.7$} ;

        \fill[black] (3,0.9) circle (1.5pt) node[above] {$\vf(\vs_1)$};
        \draw[thick, dashed] (3,0.9) -- (0.245, 1.386) node[above right] {\footnotesize $0.3$} ;
    \end{scope}

    \node at (2.25,5) {\Large $\vf (\cdot \cat \vq)$};
    \node at (8.5,5) {\Large $\vf (\cdot)$};
    \node at (13,5) {\large $\vg_o(\textit{``English''})^\top \vf (\vq)$};

    \node at (5.25,-1) {\color{red} \Large $=$};
    \node at (2.25,-1) {\Large $\vg_n(\textit{``yes''})^\top \mathbf{f} (\mathbf s \cat \mathbf q) 
    $};
    \node at (8.25,-1) {\Large $\vg_o(\textit{``English''})^\top \mathbf{f} (\mathbf s)  $};

    \node at (5.25,-2) {\Large $\vq=$``\textit{Is the text written in English''}};
    
    \node at (2.225,-2.7) {\Large $\vs_1=$``\textit{Wow New Orleans \includegraphics[width=0.8em]{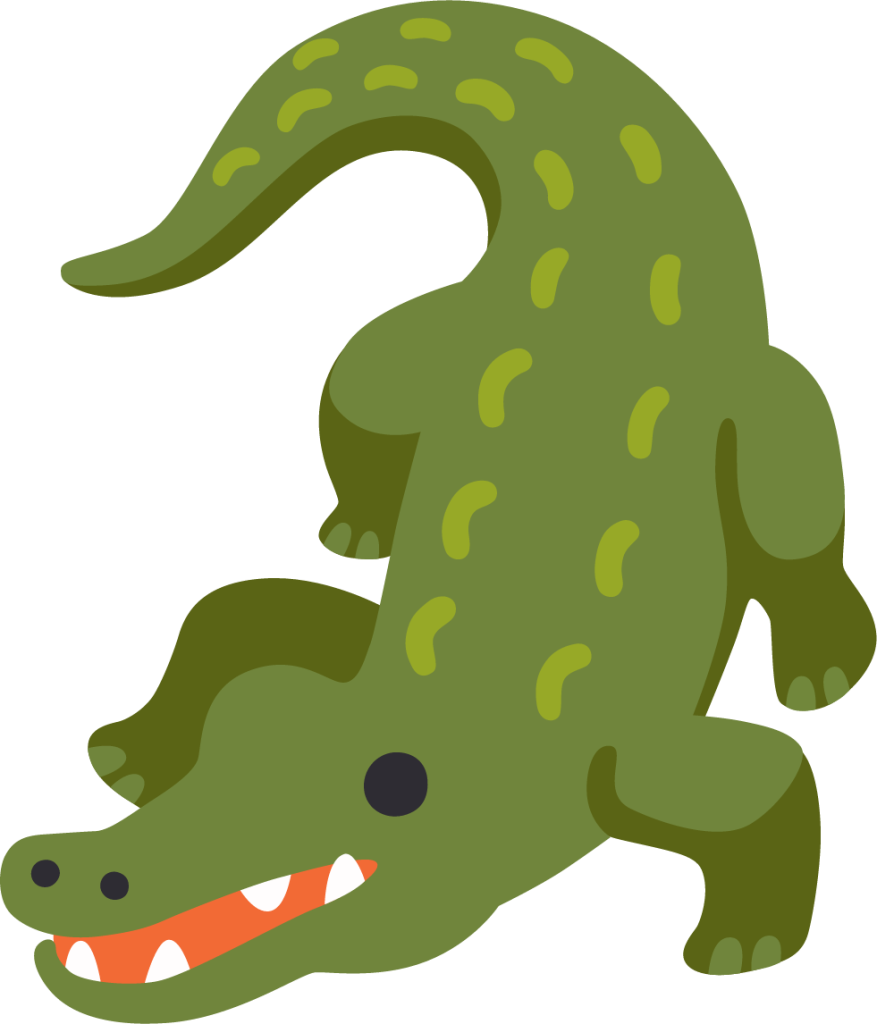}''}};

    \node at (8.275,-2.7) {\Large $\vs_2=$``\textit{Today pizza \includegraphics[width=0.8em]{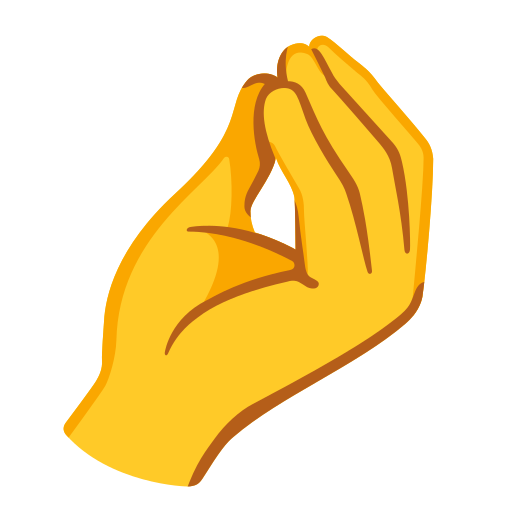}''}};

     \begin{scope}[shift={(12.5,-1)}]
        \draw[thick] (0, 0) rectangle (0.4, 5);
        \shade[top color=magenta, bottom color=magenta, opacity=0.5] (0, 0) rectangle (0.4, 0.5);
        
        \shade[top color=magenta!80, bottom color=magenta!80, opacity=0.5] (0, 0.5) rectangle (0.4, 1);
        
        \shade[top color=magenta!50, bottom color=magenta!50, opacity=0.5] (0, 1) rectangle (0.4, 1.5);
        
        \shade[top color=magenta!20, bottom color=magenta!20, opacity=0.5] (0, 1.5) rectangle (0.4, 2);
        
        \shade[top color=cyan!20, bottom color=cyan!20, opacity=0.5] (0, 2) rectangle (0.4, 2.5);
        
        \shade[top color=cyan!50, bottom color=cyan!50, opacity=0.5] (0, 2.5) rectangle (0.4, 3);
        
        \shade[top color=cyan!80, bottom color=cyan!80, opacity=0.5] (0, 3) rectangle (0.4, 3.5);

        \shade[top color=cyan, bottom color=cyan, opacity=0.5] (0, 3.5) rectangle (0.4, 4);

        \shade[top color=cyan!90!blue!80, bottom color=cyan!90!blue!80, opacity=0.5] (0,4) rectangle (0.4, 4.5);

        \shade[top color=cyan!80!blue!70, bottom color=cyan!80!blue!70, opacity=0.5] (0, 4.5) rectangle (0.4, 5);
        \draw[thick] (0.4,5) -- (0.6, 5) node[right] {$1.25$};
        
        \draw[thick] (0.4,4) -- (0.6, 4) node[right] {$1.0$};

        \draw[thick] (0.4,3) -- (0.6, 3) node[right] {$0.75$};
        
        \draw[thick] (0.4,2) -- (0.6,2) node[right] {$0.5$};
        
        \draw[thick] (0.4,1) -- (0.6,1) node[right] {$0.25$};

        \draw[thick] (0.4,0) -- (0.6,0) node[right] {$0.0$};
        
    \end{scope}
\end{tikzpicture}
    }

    \vspace{-0.5em}
    
    \caption{\textbf{Relational linear subspaces}. 
    The figure depicts the embedding function $\vf$ of a model $(\vf, \vg) \in \Theta$ with representation dimension $d=2$.
    Let $\vg_o(\textit{``English''}) := \vg(\textit{``English''}) - \vg(\textit{``other language''})$ and
    $\vg_{n}(\textit{``yes''}) := \vg(\textit{``yes''}) - \vg(\textit{``no''})$. %
    Here, $(\vf, \vg)$ linearly represents (\cref{def:relational-linear-subspaces})
    the subspace spanned by $\vg_o(\textit{``English''})$ for the query $\vq=$``\textit{Is the text written in English?''}.
    Accordingly, there exists a vector, here $\vg_n(\textit{``yes''})$, such that
    the dot product $\vg_o(\textit{``English''})^\top \vf(\vs)$, whose magnitude is represented through the color map on the right, matches the dot product
    ${\vg_n(\textit{``yes''})^\top\vf(\vs \cat \vq)}$, on the left.
    For ease of visualization, we set $\vg_n(\textit{``yes''})^\top \va_\vq = 0$ %
    \changed{and we display the values of the dot products for two input contexts $\vs_1, \vs_2$. %
    Intuitively, the dot product of a context's embedding $\vf(\vs)$ with 
    $\vg_{o}(\textit{``English''})$ %
    captures the log-probability ratio of ``\textit{yes}'' vs. ``\textit{no}'' as next tokens for the same context $\vs$ concatenated with the query $\vq$.}
    }
    \label{fig:relational-linearity}
\end{figure}

\subsection{Relational linear property}

\looseness-1
\changed{Beyond parallelism, the first property we define is relational linearity, introduced by \citet{paccanaro2001learning} and recently studied by \citet{hernandez2023linearity},
who found empirical evidence that hidden transformer layers in language models display this property.
}

\paragraph{Context-query-reply sequences.}
We
consider sequences $\vx \in \SeqA$ that can be decomposed as ${\vx = \vs \cat \vq \cat y}$, where $\vs \in \SeqA$ is termed \textit{context} (or \textit{subject)}, $\vq \in \SeqA$ is termed \textit{query} (or \textit{relation}), and $y \in \calA$ is termed \textit{reply} (or \textit{object}). The following example illustrates a semantically meaningful context-query-reply sequence.

\begin{example}
    \changed{Consider a sequence ${\vx = \vs \cat \vq \cat y}$ where $\vs=$``All roads lead to Rome'', $\vq=$``What is the written language?'', and $y=$``English''. We deliberately pick $y=$``English'' as the most likely next-token prediction following $\vs \cat \vq$  made by English speakers. 
    The string $\vs \cat \vq$ could also be provided as input to a language
    model to test whether it can recognize English language. Another example of context-query-reply sequence is $\vs=$``Rome'', $\vq=$``is the capital of'' and $y=$``Italy''.}      
\end{example}

\changed{As shown in the examples above, it is often possible to parse natural language expressions into strings $\vx = \vs \cat \vq \cat y$ which capture relational aspects encoded in substrings $\vs$ and $y$ through a substring $\vq$. 
In principle, one could also consider strings
$\vs \cat \vq$
involving queries whose expected reply is independent of the context, such as $\vq=$``\textit{Whatever follows reply with $42$}'', or paraphrases of the query, \eg $\vq'=$``\textit{Now reply with $42$}''.
{In \cref{sec:properties}, we disscuss how our analysis can capture these corner cases.}}
\looseness-1 %
\changed{Intuitively, relational linearity entails the property that all the information relevant for next-token prediction carried by the embeddings of the joint string $\vs \cat \vq$ (\ie $\vf(\vs \cat \vq)$) can be retrieved by considering the embeddings of $\vs$ (\ie $\vf(\vs)$) via an affine transformation.
To formalize this, we focus on the embeddings $\vf (\vs)$ of the model and on subspaces $\Gamma \subseteq \SIM{\vg_0}$ of the unembeddings, which contain the relevant tokens for $\vq$.%
\footnote{\looseness=-1 \changed{E.g., for a query $\vq=\textit{``What is the written language?''}$, next-tokens corresponding to different languages may be more probable and interesting to look at, thought it's ultimately a modeler's choice what subspace $\Gamma$ to focus on.}}
}

\begin{definition}[\notion: Relational linearity of $\vq$ in $\Gamma$] \label{def:linearity}
    For a model $(\vf, \vg) \in \Theta$,
    let {$\Gamma \subseteq \SIM{\vg_0}$} be a subspace.
    We say that $(\vf, \vg)$ linearly represents the query $\vq \in \SeqA$ on $\Gamma$, if there exist a matrix $\vA_{\vq} \in \bbR^{d\times d}$ and a vector $\va_{\vq} \in \bbR^d$ such that, for all $\vs \in \SeqA$,
    \[  \label{eq:linearity-proj}
        \vP_\Gamma \vf(\vs \cat \vq) =  \vP_\Gamma  \big( \vA_{\vq} \vf(\vs) + \va_{\vq} \big).
    \]
    When this holds, we define $\Gamma_\vq := \mathrm{Im}(\vA_\vq^\top \vP_\Gamma)$.
\end{definition}

Intuitively, all the information within $\vf(\vs \cat \vq)$ which is relevant to compute the probability of next-tokens in $\Gamma$ is captured, up to an affine transformation, by $\vf(\vs)$ in the subspace $\Gamma_\vq$. Indeed, one can show that, if $\vg_0(\vy) \in \Gamma$, then necessarily $\vf(\vs\cat\vq)^\top\vg_0(y) = ( \vA_{\vq} \vf(\vs) + \va_{\vq} )^\top \vg_0(y)$.
The spaces $\Gamma$ and $\Gamma_\vq$ are central to proving whether relational linearity holds for all $\sim_{EL}$-equivalent models, as we will show in \cref{sec:implications}.

\textbf{Connection to other linear properties}.
\changed{In the following, we show how to capture three additional linear properties %
building on the definition of relational linearity.}
We follow the taxonomy by \citet{park2023linear}.

\subsubsection{Linear subspaces ({\sc ls})}

\cchanged{%
Parallel vectors naturally define a one-dimensional subspace  $\Gamma \in \bbR^d$ that contains all of those vectors.}
 \cchanged{In language model representations, several such subspaces have been identified that appear to encode semantic and syntactic properties, %
for example translation across languages or the transformation of an adjective into its comparative or superlative form%
~\citep{mikolov2013distributed, park2023linear}.}
\cchanged{
Our relational formulation of this linear property focuses on whether these subspaces contain the information in the embeddings $\vf(\vs)$ which is relevant to predict the reply token to a query $\vq$ when appended to the context $\vs$.
E.g., this could happen if the embeddings projected in the subspace $\Gamma_{eng,ita}$, containing the vector $\vg(\textit{``Rome''}) - \vg(\textit{``Roma''})$,
encode whether the replies to the query $\vq=\textit{``Is written in English or Italian?''}$ are more likely to be $y=\textit{``English''}$ or $y'=\textit{``Italian''}$. %
}
We can capture this %
through the following definition:

\begin{definition}[{\sc ls}: Relational Linear Subspaces]
\label{def:relational-linear-subspaces}
    A model $(\vf, \vg) \in \Theta$ linearly represents a subspace ${\Gamma \subseteq \SIM{\vg_0}}$ relative to $\vq \in \SeqA$ 
    if for all \changed{pairs} of tokens $y_i, y_j \in \calA$ such that
    $\vg_i(y_j) := \vg(y_j) -\vg(y_i) \in \Gamma$, there exists a vector $\vgamma \in \SIM{\vg_0}$ such that
    ${\forall \vs \in \SeqA}$
    \[  \label{eq:lin-subspace}
        \vg_i(y_j)^\top \vf(\vs) = \vgamma^\top (\vf(\vs \cat \vq) - \va_\vq).
    \]
\end{definition}

We provide one example of this property in \cref{fig:relational-linearity}.
The \textsc{ls} property is implied by relational linearity (\cref{def:linearity}) in the following sense:

\begin{restatable}[\notion$\implies${\sc ls}]{proposition}{connectionls}
\label{prop:connection-lin-sub}
    Suppose that a model $(\vf, \vg) \in \Theta$
    (i) linearly represents $\vq$ on $\Gamma \subseteq \SIM{\vg_0}$, and 
    (ii) ${\Gamma_\vq \subseteq \SIM{\vg_0}}$, then 
    the model $(\vf, \vg)$ linearly represents %
    $\Gamma_\vq$ relative to $\vq$ (\cref{def:relational-linear-subspaces}).
\end{restatable}

\subsubsection{\textbf{Linear probing} ({\sc lp})}

\changed{There is empirical evidence that, in language models, sentence embeddings can be linearly separated with good accuracy based on the language of the corresponding sentences \citep{park2023linear,chang2022geometry}.}%
This property is also termed {\em linear probing}
\citep{alain2016understanding, kim2018interpretability}.
Below, we redefine linear probing as a relational property, based on~\cref{def:linearity}:

\begin{definition}[{\sc lp}: Relational Linear Probing] \label{def:linear-probing}
    We say that a model $(\vf, \vg) \in \Theta$ can be linearly probed for a query $\vq \in \SeqA$ and a collection $\calY_P \subseteq \calA$ of $\ell$ elements if %
    there exist $\vW \in \bbR^{\ell \times d}$ and $\vb \in \bbR^\ell$ such that for all $\vs \in \SeqA$ and $\forall i \in [\ell]$ 
    \[
        \mathrm{softmax}\big(\vW \vf(\vs) + \vb \big)_i = p_{\vf, \vg}(y_i \mid \vs \cat \vq; \calY_P),
    \]
    where $p(y \mid \cdot\,; \calY_P) = p(y \mid \cdot\,) / (\sum_{y'\in\calY_P}p(y'\mid \cdot\,))$ is
    \changed{the conditional %
    distribution} restricted to the set $\calY_P$.
    
\end{definition}

To illustrate why this is termed {\textit{linear probing}},
suppose a model given the query $\vq=$``\textit{Is the text written in English?}''
discriminates input sequences $\vs \in \SeqA$ between positive $y_0=$``\textit{yes}'' and negative examples $y_1=$``\textit{no}''---\changed{that is, it assigns high probability ${p_{\vf, \vg} (y_0 \mid \vs \cat \vq)}$ to sequences $\vs$ corresponding to English sentences, and high probability $p_{\vf, \vg} (y_1 \mid \vs \cat \vq)$ to non-English sentences.} %
Then,
these conditional distributions %
can be evaluated directly from $\vf(\vs)$ via a linear probe.
\cref{fig:second-page} includes an illustration of {\sc lp}. Below, 
we relate {\sc lp} (\cref{def:linear-probing}) to \notion (\cref{def:linearity}):  

\begin{restatable}[\notion$\implies${\sc lp}]{proposition}{connectionlp}
    \label{prop:connection-probing}
    If a model $(\vf, \vg) \in \Theta$ (i) linearly represents $\vq$ on $\Gamma$, and (ii) $\vg(y_i) -\vg(y_j) \in  \Gamma$ for all $y_i \in \calY_P$, then the model can be linear probed (\cref{def:linear-probing}) for $\vq$ and $\calY_P$, with parameters given by
    {$\vW = (\vw_1, \ldots, \vw_\ell)^\top$
    and $\vb =(b_1, \ldots, b_\ell)^\top$, where
    $ \vw_i := \vA_\vq^\top \vg(y_i) 
    $ and $
        b_i := (\va_\vq)^\top \vg(y_i)
    $}.
\end{restatable}

\subsubsection{\textbf{Linear Steering}}%
\label{sec:linear-steering}

\changed{Another %
property %
that has attracted considerable attention is the {\em linear steering property} \citep{stolfo2024improving}, 
also termed {\em linear intervening property} by \citet{park2023linear}. 
By knowing what queries are linearly represented by the model (as per \cref{def:linearity}), 
this property allows us to \textit{steer} the model embeddings such that the most-likely reply to a given query changes, while the replies to other queries remain unaffected. 
In \cref{sec:app-additivity}, we define a relational version of this property, and show under what conditions it is implied by the relational linear property (\cref{prop:linear-additivity}).}

\section{Linear properties shared by all distribution-equivalent models}
\label{sec:implications}

Based on \cref{thm:partial-identifiability}, we can now analyze which linear properties are shared across models expressing the same next-token distribution. 
\changed{We start from relational linearity (\cref{def:linearity}).}
To this end, pick a model $(\vf, \vg)$ that linearly represents $\vq$ on $\Gamma$, and consider the space $\Gamma_\vq := \mathrm{Im}(\vA_\vq^\top \vP_\Gamma)$. We show that under an additional condition on $\Gamma$ and $\Gamma_\vq$,
two models that are $\sim_{EL}$-equivalent share the same linear properties (results from this section are proved in \cref{sec:proof-sec5}):

\begin{restatable}{theorem}{pidlinrep}
\label{prop:part-id-lin-rep-tentative}
    For two models $(\vf, \vg), (\tilde{\vf}, \tilde{\vg}) \in \Theta$ s.t. $(\vf, \vg) \sim_{EL} (\tilde{\vf}, \tilde{\vg})$,
    if
    $\vf$ linearly represents $\vq$ on $\Gamma \subseteq \calN$, and $\Gamma_\vq \subseteq \calM$, then
    $\tilde \vf$ linearly represents $\vq$ on $\tilde \Gamma \subseteq \tilde \calN$, where $\tilde{\Gamma} = \mathrm{Im}(\vN^+ \vP_\Gamma)$ and%
    { $\vN$ is the matrix relating $\vg_0$ and $\tilde{\vg}_0$ by the equivalence relation in \cref{def:klinear-equiv}.}
\end{restatable}

\changed{This shows how, under the 
the condition that $\Gamma \subseteq \calN$ and $\Gamma_\vq \subseteq \calM$,
relational linearity (\notion)
\textit{is a property of all or none next-token predictors modeling the same conditional distribution}.} As a consequence, the same holds for \textsc{ls} (by \cref{prop:connection-lin-sub}) and \textsc{lp} (by \cref{prop:connection-probing}). %
Intuitively, the extra condition %
underlies that relational linearity of $(\vf, \vg)$ is displayed by the components of $\vf$ that contribute to the dot product with $\vg_0$. A $\sim_{EL}$-equivalent model $(\tilde \vf, \tilde \vg)$ would linearly transform these components, thus preserving relational linearity. Vice versa, since all components of $\vf$ outside $\calM$ can be arbitrarily distorted, any property of $(\vf, \vg)$ that depends on those components may not hold for $(\tilde \vf, \tilde \vg)$. The extra condition precisely avoids that.\footnote{{If $\Gamma \not \subseteq \calN$, then relational linearity (\cref{def:linearity}) would be trivially satisfied for $(\vf, \vg)$ and, in turn, also for $(\tilde{\vf}, \tilde{\vg})$, because %
$\vgamma \in \Gamma$ would give $\vgamma^\top \vf(\vx) =0$, and so $\vP_\Gamma \vf(\vx) = \mathbf{0}$.}}
Notice that the special case where the diversity condition (\cref{def:diversity-condition}) holds implies a similar conclusion because the condition that $\Gamma \subseteq \calN$ and $ \Gamma_\vq \subseteq \calM$ is then always satisfied (as $\calM = \calN = \bbR^d$). %
\changed{This testifies that (a special case of) relational linearity is shared among $\sim_{EL}$ models. In contrast,}
vector parallelism may not be preserved: Two parallel vectors in one model $(\vf, \vg)$ may not be parallel in another model $(\tilde \vf, \tilde \vg)$ with the same conditional distribution.
They remain parallel only within the subspaces $\calN$ and $\tilde \calN$, respectively:

\begin{restatable}{theorem}{parallelismnotpreserved}
    \label{prop:no-paral}
    For two models $(\vf, \vg), (\tilde{\vf}, \tilde{\vg}) \in \Theta$, such that $(\vf, \vg) \sim_{EL} (\tilde{\vf}, \tilde{\vg})$, the vectors $\vgamma, \vgamma' \in \SIM{\vg_0}$ are parallel within $\calN$ if and only if the corresponding vectors $\tilde{\vgamma}, \tilde{\vgamma}' \in \SIM{\tilde \vg_0}$ are parallel in $\tilde \calN$.
\end{restatable}

\section{Discussion}
\label{sec:discussion}

{%
\cref{prop:part-id-lin-rep-tentative} is an example of a property that all distribution-equivalent next-token predictors,
as characterized by our identifiability result~(\cref{thm:partial-identifiability}),
must share. %
One may then 
ask whether the widely observed linear properties of language models 
are indeed examples of shared properties akin to the one in~\cref{prop:part-id-lin-rep-tentative}. Tautologically, claims about the ubiquity of linear properties cannot solely be based on observations of linearity in individual model instances. One might thus
hypothesize that \textbf{only those properties shared across all equivalent models should be ubiquitously observable}.\footnote{An analogy could be made with the principle of covariance in physics~\citep{einstein1920fundamental, thorne2000gravitation}, which asserts that physical laws should be expressible as coordinate-independent and reference-frame-independent geometric
relationships between objects that represent physical entities~\citep{thorne2017modern}. Recently, \citet{villar2023towards} suggested that this principle could inspire future developments in machine learning.} %
In the following, we critically examine this hypothesis in light of empirical evidence. %
\paragraph{Do we observe properties which 
are not shared across all
{$\sim_{EL}$-equivalent} models?}
\cref{prop:no-paral} shows that vector parallelism is only preserved within a linear subspace $\calN$ of the unembedding space.
Surprisingly, a different kind of parallelism, 
 which according to our theory is not shared by all
{$\sim_{EL}$-equivalent} models,
appears to be consistently observed in language models. %
In fact, several empirical studies
apply dimensionality reduction through PCA to the embeddings and unembeddings to reveal and visualize linear properties including parallelism (\eg \cite[Figure 2]{mikolov2013distributed}; \cite[Figure 1]{marks2023geometry}).
Note that, across $\sim_{EL}$-equivalent models, the embeddings and unembeddings may not be completely contained within $\calM$ and $\calN$, respectively:
that is the case when 
 both $\calM \subsetneq \SIM{\vf}$ and $\calN \subsetneq \SIM{\vg_0}$. 
As a consequence, unbounded distortions
within the orthogonal complements of $\calM$ and of $\calN$ 
are inconsequential for the dot product between the embedding and the unembedding vectors (see, \eg~\cref{fig:illustration-kidf}, top left, for an embedding manifold not contained within any proper linear subspace).
If these distortions were sufficiently large, they would prevent the visualization of vectors parallel in the sense of \cref{prop:no-paral} through PCA, as the distortions would dominate the covariance matrix on which the PCA 
of the representations
is performed%
, and thus the first principal components would mostly reflect those. %
This suggests that, in models where PCA reveals parallelism, these distortions %
are small, and the representations live close to a proper linear subspace. 

\textbf{How can we explain this?} These observations suggest that something other than the assumptions of our~\cref{thm:partial-identifiability} determines what models are learned in practice.
One possible explanation is that some additional
assumptions and constraints are at play %
which imply
that only models in a subset of the $\sim_{EL}$ equivalence class are observed empirically.
For parallelism, this could occur, for example, if the modeler chooses a fixed $d$ for which the diversity condition happens to hold (\cref{def:diversity-condition}): in which case, the resulting equivalence class would be $\sim_{L}$, and parallelism in $\bbR^d$ (\cref{def:s-parallelism}) is a shared property across $\sim_{L}$-equivalent models with representation dimensionality $d$.
An alternative possibility is that other inductive biases, not captured by the identifiability result, are influencing the learned representations. These biases could stem from the training algorithm or architecture, steering the model toward a subset of the $\sim_{EL}$-equivalent models.
Our contribution is to provide a mathematical framework that enables a clear articulation of these questions, guiding future empirical investigation.
}

\section{Related work and future directions} 
\label{sec:related_work}

\textbf{Linear properties of next-token predictors} have attracted widespread attention, also beyond language modeling \citep{li2022emergent, nanda2023emergent, elhage2022toymodelsof}. %
More complex, non-linear properties have also been observed, such as circular token representations \citep{engels2024not}. %
Formalizing these properties and investigating whether all distribution-equivalent next-token predictors share them, in the sense we studied for linear properties%
, %
is an interesting open venue. %

\textbf{Theoretical studies on linear properties.} 
\citet{park2023linear} %
introduce binary latent concepts 
to describe several linear properties (though not relational linearity) in a unified framework.
This was also applied to study categorical and hierarchical concepts \citep{park2024geometry}. %
\citet{jiang2024origins} explain 
linear properties of language models based on assumptions on the data-generating process and latent variables underlying natural text. This allows them to reason about the {\em origins} of linearity; in this work, we instead focus on the {\em ubiquity} of linear properties, with an agnostic stance on latent concept variables.

An exciting direction for future work is to prove, within our framework, why and how linear properties emerge, if they do at all.

\textbf{Identifiability of representations} is a central theme in generative modeling
\citep{moran2022identifiable, xi2023indeterminacy}, particularly in
non-linear ICA~\citep{hyvarinen2019nonlinear,
gresele2020incomplete,
halva2020hidden, 
buchholz2022function,
hyttinen2022binary} and causal representation learning~\citep{
lippe2022citris,
ahuja2023interventional,
liang2023causal, 
von2024nonparametric,
varici2024general,
zhang2024identifiability,
rajendran2024learning, li2024disentangled, 
bortolotti2025shortcuts}. 
\cchanged{\citet{buchholz2024learning} studied when %
token partitions can be identified from their interactions; 
\citet{reizinger2024understanding} discussed what role identifiability may play
in explaining several aspects of large language models %
\citep{zhang2023trained}.}
Our work highlights the role of identifiability in explaining the ubiquity of linear properties in language models. %

\section*{Acknowledgments}

We thank
Beatrix Miranda Nielsen, Antonio Vergari, Frederik Hytting Jørgensen, Filippo Camilloni,
Adrián Javaloy, and Julius von Kügelgen 
for useful discussions. 
We acknowledge positive feedback by Stefano Teso and interesting conversations with the
participants of the 2024 Bellairs Workshop on Causality.
E.M.\@ acknowledges support from TANGO, Grant Agreement No.\@ 101120763. Funded by the European Union. Views and opinions expressed are however those of the author(s) only and do not necessarily reflect those of the European Union or the European Health and Digital Executive Agency (HaDEA). Neither the European Union nor the granting authority can be held responsible for them.
L.G.\@ was supported by Danish Data Science Academy,
 which is funded by the Novo Nordisk Foundation (NNF21SA0069429).

\bibliography{references}

\begin{thebibliography}{67}
\providecommand{\natexlab}[1]{#1}
\providecommand{\url}[1]{\texttt{#1}}
\expandafter\ifx\csname urlstyle\endcsname\relax
  \providecommand{\doi}[1]{doi: #1}\else
  \providecommand{\doi}{doi: \begingroup \urlstyle{rm}\Url}\fi

\bibitem[Paccanaro and Hinton(2001)]{paccanaro2001learning}
Alberto Paccanaro and Geoffrey~E. Hinton.
\newblock Learning distributed representations of concepts using linear
  relational embedding.
\newblock \emph{IEEE Transactions on Knowledge and Data Engineering},
  13\penalty0 (2):\penalty0 232--244, 2001.

\bibitem[Hernandez et~al.(2024)Hernandez, Sharma, Haklay, Meng, Wattenberg,
  Andreas, Belinkov, and Bau]{hernandez2023linearity}
Evan Hernandez, Arnab~Sen Sharma, Tal Haklay, Kevin Meng, Martin Wattenberg,
  Jacob Andreas, Yonatan Belinkov, and David Bau.
\newblock Linearity of relation decoding in transformer language models.
\newblock In \emph{The Twelfth International Conference on Learning
  Representations (ICLR)}, 2024.

\bibitem[Roeder et~al.(2021)Roeder, Metz, and Kingma]{roeder2021linear}
Geoffrey Roeder, Luke Metz, and Durk Kingma.
\newblock On linear identifiability of learned representations.
\newblock In \emph{International Conference on Machine Learning (ICML)}, pages
  9030--9039. PMLR, 2021.

\bibitem[Rumelhart and Abrahamson(1973)]{rumelhart1973model}
David~E. Rumelhart and Adele~A. Abrahamson.
\newblock A model for analogical reasoning.
\newblock \emph{Cognitive Psychology}, 5\penalty0 (1):\penalty0 1--28, 1973.

\bibitem[Hinton et~al.(1986{\natexlab{a}})Hinton, McClelland, and
  Rumelhart]{hinton1986distributed}
Geoffrey~E. Hinton, James~L. McClelland, and David~E. Rumelhart.
\newblock Distributed representations.
\newblock In \emph{Parallel Distributed Processing: Explorations in the
  Microstructure of Cognition, Volume 1: Foundations}, pages 77--109.
  1986{\natexlab{a}}.

\bibitem[Hinton et~al.(1986{\natexlab{b}})]{hinton1986learning}
Geoffrey~E. Hinton et~al.
\newblock Learning distributed representations of concepts.
\newblock In \emph{Proceedings of the eighth Annual Conference of the Cognitive
  Science Society}, volume~1, page~12. Amherst, MA, 1986{\natexlab{b}}.

\bibitem[Rumelhart et~al.(1986)Rumelhart, Hinton, and
  Williams]{rumelhart1986learning}
David~E. Rumelhart, Geoffrey~E. Hinton, and Ronald~J. Williams.
\newblock Learning representations by back-propagating errors.
\newblock \emph{Nature}, 323\penalty0 (6088):\penalty0 533--536, 1986.

\bibitem[Bengio et~al.(2000)Bengio, Ducharme, and Vincent]{bengio2000neural}
Yoshua Bengio, R{\'e}jean Ducharme, and Pascal Vincent.
\newblock A neural probabilistic language model.
\newblock \emph{Advances in Neural Information Processing Systems (NeurIPS)},
  13, 2000.

\bibitem[Mikolov et~al.(2013{\natexlab{a}})Mikolov, Sutskever, Chen, Corrado,
  and Dean]{mikolov2013distributed}
Tomas Mikolov, Ilya Sutskever, Kai Chen, Greg~S. Corrado, and Jeff Dean.
\newblock Distributed representations of words and phrases and their
  compositionality.
\newblock \emph{Advances in Neural Information Processing Systems (NeurIPS)},
  26, 2013{\natexlab{a}}.

\bibitem[Mikolov et~al.(2013{\natexlab{b}})Mikolov, Yih, and
  Zweig]{mikolov2013linguistic}
Tomas Mikolov, Wen-tau Yih, and Geoffrey Zweig.
\newblock Linguistic regularities in continuous space word representations.
\newblock In \emph{Proceedings of the 2013 Conference of the North American
  Chapter of the Association for Computational Linguistics: Human language
  technologies}, pages 746--751, 2013{\natexlab{b}}.

\bibitem[Pennington et~al.(2014)Pennington, Socher, and
  Manning]{pennington2014glove}
Jeffrey Pennington, Richard Socher, and Christopher~D. Manning.
\newblock Glove: Global vectors for word representation.
\newblock In \emph{Proceedings of the 2014 Conference on Empirical Methods in
  Natural Language Processing (EMNLP)}, pages 1532--1543, 2014.

\bibitem[Burns et~al.(2022)Burns, Ye, Klein, and
  Steinhardt]{burns2022discovering}
Collin Burns, Haotian Ye, Dan Klein, and Jacob Steinhardt.
\newblock Discovering latent knowledge in language models without supervision.
\newblock \emph{arXiv preprint arXiv:2212.03827}, 2022.

\bibitem[Merullo et~al.(2023)Merullo, Eickhoff, and
  Pavlick]{merullo2023language}
Jack Merullo, Carsten Eickhoff, and Ellie Pavlick.
\newblock Language models implement simple word2vec-style vector arithmetic.
\newblock \emph{arXiv preprint arXiv:2305.16130}, 2023.

\bibitem[Tigges et~al.(2023)Tigges, Hollinsworth, Geiger, and
  Nanda]{tigges2023linear}
Curt Tigges, Oskar~John Hollinsworth, Atticus Geiger, and Neel Nanda.
\newblock Linear representations of sentiment in large language models.
\newblock \emph{arXiv preprint arXiv:2310.15154}, 2023.

\bibitem[Pal et~al.(2023)Pal, Sun, Yuan, Wallace, and Bau]{pal2023future}
Koyena Pal, Jiuding Sun, Andrew Yuan, Byron~C Wallace, and David Bau.
\newblock Future lens: Anticipating subsequent tokens from a single hidden
  state.
\newblock \emph{arXiv preprint arXiv:2311.04897}, 2023.

\bibitem[Gurnee and Tegmark(2023)]{gurnee2023language}
Wes Gurnee and Max Tegmark.
\newblock Language models represent space and time.
\newblock \emph{arXiv preprint arXiv:2310.02207}, 2023.

\bibitem[Bricken et~al.(2023)Bricken, Templeton, Batson, Chen, Jermyn, Conerly,
  Turner, Anil, Denison, Askell, Lasenby, Wu, Kravec, Schiefer, Maxwell,
  Joseph, Hatfield-Dodds, Tamkin, Nguyen, McLean, Burke, Hume, Carter,
  Henighan, and Olah]{bricken2023monosemanticity}
Trenton Bricken, Adly Templeton, Joshua Batson, Brian Chen, Adam Jermyn, Tom
  Conerly, Nick Turner, Cem Anil, Carson Denison, Amanda Askell, Robert
  Lasenby, Yifan Wu, Shauna Kravec, Nicholas Schiefer, Tim Maxwell, Nicholas
  Joseph, Zac Hatfield-Dodds, Alex Tamkin, Karina Nguyen, Brayden McLean,
  Josiah~E Burke, Tristan Hume, Shan Carter, Tom Henighan, and Christopher
  Olah.
\newblock Towards monosemanticity: Decomposing language models with dictionary
  learning.
\newblock \emph{Transformer Circuits Thread}, 2023.

\bibitem[Khemakhem et~al.(2020{\natexlab{a}})Khemakhem, Monti, Kingma, and
  Hyvarinen]{khemakhem2020ice}
Ilyes Khemakhem, Ricardo Monti, Diederik Kingma, and Aapo Hyvarinen.
\newblock Ice-beem: Identifiable conditional energy-based deep models based on
  nonlinear {ICA}.
\newblock \emph{Advances in Neural Information Processing Systems (NeurIPS)},
  33:\penalty0 12768--12778, 2020{\natexlab{a}}.

\bibitem[Arora et~al.(2016)Arora, Li, Liang, Ma, and Risteski]{arora2016latent}
Sanjeev Arora, Yuanzhi Li, Yingyu Liang, Tengyu Ma, and Andrej Risteski.
\newblock A latent variable model approach to {PMI}-based word embeddings.
\newblock \emph{Transactions of the Association for Computational Linguistics},
  4:\penalty0 385--399, 2016.

\bibitem[Allen and Hospedales(2019)]{allen2019analogies}
Carl Allen and Timothy Hospedales.
\newblock Analogies explained: Towards understanding word embeddings.
\newblock In \emph{International Conference on Machine Learning (ICML)}, pages
  223--231. PMLR, 2019.

\bibitem[Park et~al.(2024{\natexlab{a}})Park, Choe, and Veitch]{park2023linear}
Kiho Park, Yo~Joong Choe, and Victor Veitch.
\newblock The linear representation hypothesis and the geometry of large
  language models.
\newblock In \emph{Proceedings of the 41st International Conference on Machine
  Learning (ICML)}, volume 235 of \emph{Proceedings of Machine Learning
  Research}, pages 39643--39666. PMLR, 21--27 Jul 2024{\natexlab{a}}.

\bibitem[Heinzerling and Inui(2024)]{heinzerling2024monotonic}
Benjamin Heinzerling and Kentaro Inui.
\newblock Monotonic representation of numeric properties in language models.
\newblock \emph{arXiv preprint arXiv:2403.10381}, 2024.

\bibitem[Khemakhem et~al.(2020{\natexlab{b}})Khemakhem, Kingma, Monti, and
  Hyvarinen]{khemakhem2020variational}
Ilyes Khemakhem, Diederik Kingma, Ricardo Monti, and Aapo Hyvarinen.
\newblock {Variational autoencoders and nonlinear ICA: A unifying framework}.
\newblock In \emph{International Conference on Artificial Intelligence and
  Statistics (AISTATS)}, pages 2207--2217. PMLR, 2020{\natexlab{b}}.

\bibitem[Lachapelle et~al.(2023)Lachapelle, Deleu, Mahajan, Mitliagkas, Bengio,
  Lacoste-Julien, and Bertrand]{lachapelle2023synergies}
S{\'e}bastien Lachapelle, Tristan Deleu, Divyat Mahajan, Ioannis Mitliagkas,
  Yoshua Bengio, Simon Lacoste-Julien, and Quentin Bertrand.
\newblock Synergies between disentanglement and sparsity: Generalization and
  identifiability in multi-task learning.
\newblock In \emph{International Conference on Machine Learning (ICML)}, pages
  18171--18206. PMLR, 2023.

\bibitem[Rehder(1980)]{rehder1980projections}
W~Rehder.
\newblock When do projections commute?
\newblock \emph{Zeitschrift f{\"u}r Naturforschung A}, 35\penalty0
  (4):\penalty0 437--441, 1980.

\bibitem[Nielsen et~al.(2024)Nielsen, Gresele, and
  Dittadi]{nielsen2024challenges}
Beatrix Miranda~Ginn Nielsen, Luigi Gresele, and Andrea Dittadi.
\newblock Challenges in explaining representational similarity through
  identifiability.
\newblock In \emph{UniReps: 2nd Edition of the Workshop on Unifying
  Representations in Neural Models}, 2024.

\bibitem[Mikolov(2013)]{mikolov2013efficient}
Tomas Mikolov.
\newblock Efficient estimation of word representations in vector space.
\newblock \emph{arXiv preprint arXiv:1301.3781}, 2013.

\bibitem[Jiang et~al.(2024)Jiang, Rajendran, Ravikumar, Aragam, and
  Veitch]{jiang2024origins}
Yibo Jiang, Goutham Rajendran, Pradeep Ravikumar, Bryon Aragam, and Victor
  Veitch.
\newblock On the origins of linear representations in large language models.
\newblock \emph{arXiv preprint arXiv:2403.03867}, 2024.

\bibitem[Chang et~al.(2022)Chang, Tu, and Bergen]{chang2022geometry}
Tyler~A. Chang, Zhuowen Tu, and Benjamin~K Bergen.
\newblock The geometry of multilingual language model representations.
\newblock \emph{arXiv preprint arXiv:2205.10964}, 2022.

\bibitem[Alain(2016)]{alain2016understanding}
Guillaume Alain.
\newblock Understanding intermediate layers using linear classifier probes.
\newblock \emph{arXiv preprint arXiv:1610.01644}, 2016.

\bibitem[Kim et~al.(2018)Kim, Wattenberg, Gilmer, Cai, Wexler, Viegas,
  et~al.]{kim2018interpretability}
Been Kim, Martin Wattenberg, Justin Gilmer, Carrie Cai, James Wexler, Fernanda
  Viegas, et~al.
\newblock {Interpretability Beyond Feature Attribution: Quantitative Testing
  with Concept Activation Vectors (TCAV)}.
\newblock In \emph{International Conference on Machine Learning (ICML)}, pages
  2668--2677. PMLR, 2018.

\bibitem[Stolfo et~al.(2024)Stolfo, Balachandran, Yousefi, Horvitz, and
  Nushi]{stolfo2024improving}
Alessandro Stolfo, Vidhisha Balachandran, Safoora Yousefi, Eric Horvitz, and
  Besmira Nushi.
\newblock Improving instruction-following in language models through activation
  steering.
\newblock \emph{arXiv preprint arXiv:2410.12877}, 2024.

\bibitem[Einstein(1920)]{einstein1920fundamental}
Albert Einstein.
\newblock Fundamental ideas and methods of the theory of relativity, presented
  in their development.
\newblock \emph{The collected papers of Albert Einstein}, 7:\penalty0
  1918--1921, 1920.

\bibitem[Thorne et~al.(2000)Thorne, Misner, and Wheeler]{thorne2000gravitation}
Kip~S. Thorne, Charles~W. Misner, and John~Archibald Wheeler.
\newblock \emph{Gravitation}.
\newblock Freeman San Francisco, 2000.

\bibitem[Thorne and Blandford(2017)]{thorne2017modern}
Kip~S. Thorne and Roger~D. Blandford.
\newblock \emph{Modern Classical Physics: Optics, Fluids, Plasmas, Elasticity,
  Relativity, and Statistical Physics}.
\newblock Princeton University Press, 2017.

\bibitem[Villar et~al.(2023)Villar, Hogg, Yao, Kevrekidis, and
  Sch{\"o}lkopf]{villar2023towards}
Soledad Villar, David~W. Hogg, Weichi Yao, George~A. Kevrekidis, and Bernhard
  Sch{\"o}lkopf.
\newblock Towards fully covariant machine learning.
\newblock \emph{{Transactions on Machine Learning Research}}, 2023.

\bibitem[Marks and Tegmark(2023)]{marks2023geometry}
Samuel Marks and Max Tegmark.
\newblock The geometry of truth: Emergent linear structure in large language
  model representations of true/false datasets.
\newblock \emph{arXiv preprint arXiv:2310.06824}, 2023.

\bibitem[Li et~al.(2022)Li, Hopkins, Bau, Vi{\'e}gas, Pfister, and
  Wattenberg]{li2022emergent}
Kenneth Li, Aspen~K. Hopkins, David Bau, Fernanda Vi{\'e}gas, Hanspeter
  Pfister, and Martin Wattenberg.
\newblock Emergent world representations: Exploring a sequence model trained on
  a synthetic task.
\newblock In \emph{The Eleventh International Conference on Learning
  Representations (ICLR)}, 2022.

\bibitem[Nanda et~al.(2023)Nanda, Lee, and Wattenberg]{nanda2023emergent}
Neel Nanda, Andrew Lee, and Martin Wattenberg.
\newblock Emergent linear representations in world models of self-supervised
  sequence models.
\newblock \emph{arXiv preprint arXiv:2309.00941}, 2023.

\bibitem[Elhage et~al.(2022)Elhage, Hume, Olsson, Schiefer, Henighan, Kravec,
  Hatfield-Dodds, Lasenby, Drain, Chen, Grosse, McCandlish, Kaplan, Amodei,
  Wattenberg, and Olah]{elhage2022toymodelsof}
Nelson Elhage, Tristan Hume, Catherine Olsson, Nicholas Schiefer, Tom Henighan,
  Shauna Kravec, Zac Hatfield-Dodds, Robert Lasenby, Dawn Drain, Carol Chen,
  Roger Grosse, Sam McCandlish, Jared Kaplan, Dario Amodei, Martin Wattenberg,
  and Christopher Olah.
\newblock Toy models of superposition.
\newblock \emph{Transformer Circuits Thread}, 2022.

\bibitem[Engels et~al.(2024)Engels, Liao, Michaud, Gurnee, and
  Tegmark]{engels2024not}
Joshua Engels, Isaac Liao, Eric~J Michaud, Wes Gurnee, and Max Tegmark.
\newblock Not all language model features are linear.
\newblock \emph{arXiv preprint arXiv:2405.14860}, 2024.

\bibitem[Park et~al.(2024{\natexlab{b}})Park, Choe, Jiang, and
  Veitch]{park2024geometry}
Kiho Park, Yo~Joong Choe, Yibo Jiang, and Victor Veitch.
\newblock The geometry of categorical and hierarchical concepts in large
  language models.
\newblock \emph{arXiv preprint arXiv:2406.01506}, 2024{\natexlab{b}}.

\bibitem[Moran et~al.(2022)Moran, Sridhar, Wang, and
  Blei]{moran2022identifiable}
Gemma~E. Moran, Dhanya Sridhar, Yixin Wang, and David~M. Blei.
\newblock Identifiable deep generative models via sparse decoding.
\newblock \emph{Transactions on Machine Learning Research}, 2022.

\bibitem[Xi and Bloem-Reddy(2023)]{xi2023indeterminacy}
Quanhan Xi and Benjamin Bloem-Reddy.
\newblock Indeterminacy in generative models: Characterization and strong
  identifiability.
\newblock In \emph{International Conference on Artificial Intelligence and
  Statistics (AISTATS)}, pages 6912--6939. PMLR, 2023.

\bibitem[Hyvarinen et~al.(2019)Hyvarinen, Sasaki, and
  Turner]{hyvarinen2019nonlinear}
Aapo Hyvarinen, Hiroaki Sasaki, and Richard Turner.
\newblock {Nonlinear ICA using auxiliary variables and generalized contrastive
  learning}.
\newblock In \emph{The 22nd International Conference on Artificial Intelligence
  and Statistics (AISTATS)}, pages 859--868. PMLR, 2019.

\bibitem[Gresele et~al.(2020)Gresele, Rubenstein, Mehrjou, Locatello, and
  Sch{\"o}lkopf]{gresele2020incomplete}
Luigi Gresele, Paul~K Rubenstein, Arash Mehrjou, Francesco Locatello, and
  Bernhard Sch{\"o}lkopf.
\newblock The incomplete {R}osetta {S}tone problem: {I}dentifiability results
  for multi-view nonlinear {ICA}.
\newblock In \emph{Uncertainty in Artificial Intelligence (UAI)}, pages
  217--227. PMLR, 2020.

\bibitem[H{\"a}lv{\"a} and Hyvarinen(2020)]{halva2020hidden}
Hermanni H{\"a}lv{\"a} and Aapo Hyvarinen.
\newblock Hidden markov nonlinear {ICA}: Unsupervised learning from
  nonstationary time series.
\newblock In \emph{Conference on Uncertainty in Artificial Intelligence (UAI)},
  pages 939--948. PMLR, 2020.

\bibitem[Buchholz et~al.(2022)Buchholz, Besserve, and
  Sch{\"o}lkopf]{buchholz2022function}
Simon Buchholz, Michel Besserve, and Bernhard Sch{\"o}lkopf.
\newblock Function classes for identifiable nonlinear independent component
  analysis.
\newblock \emph{Advances in Neural Information Processing Systems (NeurIPS)},
  35:\penalty0 16946--16961, 2022.

\bibitem[Hyttinen et~al.(2022)Hyttinen, Pacela, and
  Hyv{\"a}rinen]{hyttinen2022binary}
Antti Hyttinen, Vit{\'o}ria~Barin Pacela, and Aapo Hyv{\"a}rinen.
\newblock Binary independent component analysis: a non-stationarity-based
  approach.
\newblock In \emph{Uncertainty in Artificial Intelligence (UAI)}, pages
  874--884. PMLR, 2022.

\bibitem[Lippe et~al.(2022)Lippe, Magliacane, L{\"o}we, Asano, Cohen, and
  Gavves]{lippe2022citris}
Phillip Lippe, Sara Magliacane, Sindy L{\"o}we, Yuki~M Asano, Taco Cohen, and
  Stratis Gavves.
\newblock Citris: Causal identifiability from temporal intervened sequences.
\newblock In \emph{International Conference on Machine Learning (ICML)}, pages
  13557--13603. PMLR, 2022.

\bibitem[Ahuja et~al.(2023)Ahuja, Mahajan, Wang, and
  Bengio]{ahuja2023interventional}
Kartik Ahuja, Divyat Mahajan, Yixin Wang, and Yoshua Bengio.
\newblock Interventional causal representation learning.
\newblock In \emph{International Conference on Machine Learning (ICML)}, pages
  372--407. PMLR, 2023.

\bibitem[Liang et~al.(2024)Liang, Keki{\'c}, von K{\"u}gelgen, Buchholz,
  Besserve, Gresele, and Sch{\"o}lkopf]{liang2023causal}
Wendong Liang, Armin Keki{\'c}, Julius von K{\"u}gelgen, Simon Buchholz, Michel
  Besserve, Luigi Gresele, and Bernhard Sch{\"o}lkopf.
\newblock Causal {C}omponent {A}nalysis.
\newblock \emph{Advances in Neural Information Processing Systems (NeurIPS)},
  36, 2024.

\bibitem[von K{\"u}gelgen et~al.(2024)von K{\"u}gelgen, Besserve, Wendong,
  Gresele, Keki{\'c}, Bareinboim, Blei, and
  Sch{\"o}lkopf]{von2024nonparametric}
Julius von K{\"u}gelgen, Michel Besserve, Liang Wendong, Luigi Gresele, Armin
  Keki{\'c}, Elias Bareinboim, David Blei, and Bernhard Sch{\"o}lkopf.
\newblock Nonparametric identifiability of causal representations from unknown
  interventions.
\newblock \emph{Advances in Neural Information Processing Systems (NeurIPS)},
  36, 2024.

\bibitem[Varici et~al.(2024)Varici, Acart{\"u}rk, Shanmugam, and
  Tajer]{varici2024general}
Burak Varici, Emre Acart{\"u}rk, Karthikeyan Shanmugam, and Ali Tajer.
\newblock General identifiability and achievability for causal representation
  learning.
\newblock In \emph{International Conference on Artificial Intelligence and
  Statistics (AISTATS)}, pages 2314--2322. PMLR, 2024.

\bibitem[Zhang et~al.(2024{\natexlab{a}})Zhang, Greenewald, Squires,
  Srivastava, Shanmugam, and Uhler]{zhang2024identifiability}
Jiaqi Zhang, Kristjan Greenewald, Chandler Squires, Akash Srivastava,
  Karthikeyan Shanmugam, and Caroline Uhler.
\newblock Identifiability guarantees for causal disentanglement from soft
  interventions.
\newblock \emph{Advances in Neural Information Processing Systems (NeurIPS)},
  36, 2024{\natexlab{a}}.

\bibitem[Rajendran et~al.(2024)Rajendran, Buchholz, Aragam, Sch{\"o}lkopf, and
  Ravikumar]{rajendran2024learning}
Goutham Rajendran, Simon Buchholz, Bryon Aragam, Bernhard Sch{\"o}lkopf, and
  Pradeep Ravikumar.
\newblock Learning interpretable concepts: Unifying causal representation
  learning and foundation models.
\newblock \emph{arXiv preprint arXiv:2402.09236}, 2024.

\bibitem[Li et~al.(2024)Li, Pan, and Bareinboim]{li2024disentangled}
Adam Li, Yushu Pan, and Elias Bareinboim.
\newblock Disentangled representation learning in non-markovian causal systems.
\newblock 2024.

\bibitem[Bortolotti et~al.(2025)Bortolotti, Marconato, Morettin, Passerini, and
  Teso]{bortolotti2025shortcuts}
Samuele Bortolotti, Emanuele Marconato, Paolo Morettin, Andrea Passerini, and
  Stefano Teso.
\newblock Shortcuts and identifiability in concept-based models from a
  neuro-symbolic lens.
\newblock \emph{arXiv preprint arXiv:2502.11245}, 2025.

\bibitem[Buchholz(2024)]{buchholz2024learning}
Simon Buchholz.
\newblock Learning partitions from context.
\newblock In \emph{Advances in Neural Information Processing Systems
  (NeurIPS)}, volume~37, 2024.

\bibitem[Reizinger et~al.(2024)Reizinger, Ujv{\'a}ry, M{\'e}sz{\'a}ros,
  Kerekes, Brendel, and Husz{\'a}r]{reizinger2024understanding}
Patrik Reizinger, Szilvia Ujv{\'a}ry, Anna M{\'e}sz{\'a}ros, Anna Kerekes,
  Wieland Brendel, and Ferenc Husz{\'a}r.
\newblock {Position: Understanding LLMs Requires More Than Statistical
  Generalization}.
\newblock In \emph{Forty-first International Conference on Machine Learning
  (ICML)}, 2024.

\bibitem[Zhang et~al.(2024{\natexlab{b}})Zhang, Frei, and
  Bartlett]{zhang2023trained}
Ruiqi Zhang, Spencer Frei, and Peter~L Bartlett.
\newblock Trained transformers learn linear models in-context.
\newblock \emph{Journal of Machine Learning Research}, 25\penalty0
  (49):\penalty0 1--55, 2024{\natexlab{b}}.

\bibitem[Liu et~al.(2018)Liu, Saleh, Pot, Goodrich, Sepassi, Kaiser, and
  Shazeer]{liu2018generating}
Peter~J Liu, Mohammad Saleh, Etienne Pot, Ben Goodrich, Ryan Sepassi, Lukasz
  Kaiser, and Noam Shazeer.
\newblock Generating wikipedia by summarizing long sequences.
\newblock \emph{arXiv preprint arXiv:1801.10198}, 2018.

\bibitem[Radford et~al.(2021)Radford, Kim, Hallacy, Ramesh, Goh, Agarwal,
  Sastry, Askell, Mishkin, Clark, et~al.]{radford2021learning}
Alec Radford, Jong~Wook Kim, Chris Hallacy, Aditya Ramesh, Gabriel Goh,
  Sandhini Agarwal, Girish Sastry, Amanda Askell, Pamela Mishkin, Jack Clark,
  et~al.
\newblock Learning transferable visual models from natural language
  supervision.
\newblock In \emph{International Conference on Machine Learning (ICML)}, 2021.

\bibitem[Wang and Komatsuzaki(2021)]{wang2021gpt}
Ben Wang and Aran Komatsuzaki.
\newblock {GPT-J-6B: A 6 Billion Parameter Autoregressive Language Model}.
\newblock 2021.
\newblock URL \url{https://github. com/kingoflolz/mesh-transformer-jax}.

\bibitem[Axler(2015)]{axler2015linear}
Sheldon Axler.
\newblock \emph{{Linear Algebra Done Right}}.
\newblock Springer, 2015.

\bibitem[Hase et~al.(2024)Hase, Bansal, Kim, and Ghandeharioun]{hase2024does}
Peter Hase, Mohit Bansal, Been Kim, and Asma Ghandeharioun.
\newblock {Does localization inform editing? Surprising differences in
  causality-based localization vs. knowledge editing in language models}.
\newblock \emph{Advances in Neural Information Processing Systems (NeurIPS)},
  36, 2024.

\bibitem[Arditi et~al.(2024)Arditi, Obeso, Syed, Paleka, Rimsky, Gurnee, and
  Nanda]{arditi2024refusal}
Andy Arditi, Oscar Obeso, Aaquib Syed, Daniel Paleka, Nina Rimsky, Wes Gurnee,
  and Neel Nanda.
\newblock Refusal in language models is mediated by a single direction.
\newblock \emph{arXiv preprint arXiv:2406.11717}, 2024.

\end{thebibliography}

\section*{Checklist}

 \begin{enumerate}

 \item For all models and algorithms presented, check if you include:
 \begin{enumerate}
   \item A clear description of the mathematical setting, assumptions, algorithm, and/or model. [Yes/No/\textbf{Not Applicable}]
   \item An analysis of the properties and complexity (time, space, sample size) of any algorithm. [Yes/No/\textbf{Not Applicable}]
   \item (Optional) Anonymized source code, with specification of all dependencies, including external libraries. [Yes/No/\textbf{Not Applicable}]
 \end{enumerate}

 \item For any theoretical claim, check if you include:
 \begin{enumerate}
   \item Statements of the full set of assumptions of all theoretical results. [\textbf{Yes}/No/Not Applicable]
   \item Complete proofs of all theoretical results. [\textbf{Yes}/No/Not Applicable]
   \item Clear explanations of any assumptions. [\textbf{Yes}/No/Not Applicable]     
 \end{enumerate}

 \item For all figures and tables that present empirical results, check if you include:
 \begin{enumerate}
   \item The code, data, and instructions needed to reproduce the main experimental results (either in the supplemental material or as a URL). [Yes/No/\textbf{Not Applicable}]
   \item All the training details (e.g., data splits, hyperparameters, how they were chosen). [Yes/No/\textbf{Not Applicable}]
         \item A clear definition of the specific measure or statistics and error bars (e.g., with respect to the random seed after running experiments multiple times). [Yes/No/\textbf{Not Applicable}]
         \item A description of the computing infrastructure used. (e.g., type of GPUs, internal cluster, or cloud provider). [Yes/No/\textbf{Not Applicable}]
 \end{enumerate}

 \item If you are using existing assets (e.g., code, data, models) or curating/releasing new assets, check if you include:
 \begin{enumerate}
   \item Citations of the creator If your work uses existing assets. [Yes/No/\textbf{Not Applicable}]
   \item The license information of the assets, if applicable. [Yes/No/\textbf{Not Applicable}]
   \item New assets either in the supplemental material or as a URL, if applicable. [Yes/No/\textbf{Not Applicable}]
   \item Information about consent from data providers/curators. [Yes/No/\textbf{Not Applicable}]
   \item Discussion of sensible content if applicable, e.g., personally identifiable information or offensive content. [Yes/No/\textbf{Not Applicable}]
 \end{enumerate}

 \item If you used crowdsourcing or conducted research with human subjects, check if you include:
 \begin{enumerate}
   \item The full text of instructions given to participants and screenshots. [Yes/No/\textbf{Not Applicable}]
   \item Descriptions of potential participant risks, with links to Institutional Review Board (IRB) approvals if applicable. [Yes/No/\textbf{Not Applicable}]
   \item The estimated hourly wage paid to participants and the total amount spent on participant compensation. [Yes/No/\textbf{Not Applicable}]
 \end{enumerate}

 \end{enumerate}

\begin{minipage}[c]{0.4\textwidth}
    \color{white}
    \hypersetup{linkcolor=white}
    \doparttoc %
    \faketableofcontents %
    \part{} %
    \parttoc %
\end{minipage}

\onecolumn 

\appendix

\addcontentsline{toc}{section}{Appendices}%
\part{Appendices} %
\parttoc%

\newpage

\section{Functional Form of Decoders-only Transformers} \label{sec:app-transformers}

Decoders-only transformer models can be reduced to the form we use in the main text, as already shown in a derivation by~\citet{roeder2021linear}. We consider autoregressive, GPT-like models \citep{liu2018generating, radford2021learning}, focusing on the GPT-J model \citep{wang2021gpt}.

We denote with $\vh(\vx; \theta)$ the representation given by the a transformer model $\vh$ with trainable parameters $\theta \in \bbR^q$ for the architecture. 
We consider the transformer to represent sentences of lenght up to $C$, and by convention we pick the last $C$ tokens if the sentence has length $t(\vx)> C$.  Denote with $\tau := \max(t(\vx) -C, 1)$ it holds that: 
\[
    \vh(\vx; \theta) =  \vh(\vx_{t(\vx) - \tau: t(\vx)}; \theta) \in \bbR^{d \times \tau} \, .
\]
We focus on the representation of the last token of sentence, which is used to perform next-token prediction, and denote it with
\[
    \vh(\vx; \theta)_{-1} \in \bbR^d \, .
\]
{When predicting the next token among $K := |\calA|$ possible options, a common approach is to use a layer (or {\em head}) with weights $\vW_D \in \bbR^{K\times d}$ and a bias $\vb \in \bbR^K$. The representation of $\vg_\varphi$ is determined by the parameters $\varphi = (\vW_D, \vb)$, which are used to compute the logits as follows:}
\[
    \mathrm{logits}(\vx) := \vW_D \vh(\vx;\theta)_{-1} + \vb \, .
\]
By extending the representation of $\vh_{-1}$ to $\vf_\theta(\vx) := (1, \vh(\vx; \theta)^\top_{-1})^\top $, we can incorporate the bias by adding one column to $\vW_D$, obtaining:
\[  
    \mathrm{logits}(\vx) = \tilde{\vW}_D \vf_\theta(\vx), \quad \text{where }   \tilde{\vW}_D = \begin{pmatrix}
        \vb, \vW_D
    \end{pmatrix} \, .
\]
To predict the probability of the next  token $y$ it is then sufficient to consider
\[  \textstyle
    \log p_{\vf_\theta, \vg_\varphi}(y \mid \vx) = \mathrm{logits}(\vx)_y - \log \sum_{y' \in \calA} \mathrm{logits}(\vx_{y'}) \, .
\]
Transforming the token $y\in \calA$ to its one-hot representation, $\vy \in \{0,1\}^{K}$, we can write
\[
    \vg_\varphi(y) = \tilde{\vW}_D^\top \vy, 
\]
thereby leading to the expression:
\[
    \log p_{\vf_\theta, \vg_\varphi}(y \mid \vx) = \vg_\varphi(y)^\top \vf_\theta(\vx) -\log \sum_{y' \in \calA} \vg_\varphi(y)^\top \vf_\theta(\vx), 
\]
where we used:
\[  \textstyle
    \vf_\theta(\vx) = \begin{pmatrix}
        1 \\
        \vh (\vx_{t(\vx) - \tau: t(\vx)}; \theta)_{-1}
    \end{pmatrix},
    \qquad 
    \vg_\varphi(y) = \tilde{\vW}_D^\top \vy.
\]
This explicit form of the embedding and unembedding respectively and consistent with Equation~\eqref{eq:next-token-predictor}, as previously detailed by \citet{roeder2021linear}.

\newpage

\section{Proofs of \cref{sec:identifiability}}
\label{sec:proofs-sec4}
    \subsection{Reminder of useful properties of pseudo-inverses}

    We will often make use of the pseudo-inverse $\vA^+$ of a matrix $\vA$  \citep[page 221]{axler2015linear}. We denote with $\vT \mid_{\mathrm{ker}(\vT)^\perp }$ the restriction of $\vT$ to its orthogonal complement of the kernel \citep{axler2015linear}.

    \begin{definition}[Pseudo-inverse]
        Let $\vT \in \bbR^{m \times n}$ be a matrix.  The pseudo-inverse $\vT^+ \in \bbR^{n \times m}$ of $\vT$ is defined as the linear map:
        \[  
            \vT^+ \vw := \big( \vT \mid_{\mathrm{ker}(\vT)^\perp } \big)^{-1} \vP_{\mathrm{Im}(\vT)} \vw
        \]
        for all $\vw \in \bbR^n$.
    \end{definition}
    Accordingly, the pseudo-inverse always exists and it is unique. Notice that for any matrix $\vT \in \bbR^{m \times n}$ it holds:
    \begin{align}
        \vT \vT^+ &= \vP_{\mathrm{Im} (\vT)} 
        \label{eq:pseudoinverse-tt+}
        \\
        \vT^+ \vT &= \vP_{\mathrm{ker} (\vT)^\perp}
        \label{eq:pseudoincerse-t+t}
        \\
        \vT^\top (\vT^\top)^+ &= \vP_{\mathrm{ker} (\vT)^\perp}
        \label{eq:pseudoincerse-ttopttop+} \\
        (\vT^\top)^+ \vT^\top &= \vP_{\mathrm{Im} (\vT)} \, .
        \label{eq:pseudoincerse-ttop+ttop}
    \end{align}

    \subsection{Proof of \cref{lemma:projectors}}

        We provide here a longer version of the Lemma capturing different properties between the projectors.

\begin{lemma*}[Ref \cref{lemma:projectors}]
    Let $(\vf, \vg) \in \Theta$, and take $\calF := \SIM{\vf}$ and $\calG :=\SIM{\vg_0}$.
    For the orthogonal projectors $\vP_\calF$ and $\vP_\calG$ and the orthogonal projectors $\vP_\calM$ and $\vP_\calN$ projecting on the spaces:
    \[
        \calM = \mathrm{Im} (\vP_\calF \vP_\calG), \quad\calN = \mathrm{ker} (\vP_\calF \vP_\calG)^\perp
    \] 
    it holds:
    \begin{itemize}[leftmargin=3em]
        \item[(i)] $\mathrm{dim}(\calM) =\mathrm{dim} (\calN) = \mathrm{dim}(\calF) -\mathrm{dim}\big(
        \calF \cap \calG^\perp
        \big) $;

        \item[(ii)] The orthogonal projectors are also given by:
        \begin{align}
            \vP_\calM = (\vP_\calF \vP_\calG) (\vP_\calF \vP_\calG)^+,& \quad
            \vP_\calN = (\vP_\calF \vP_\calG)^+ (\vP_\calF \vP_\calG) \\
            \vP_\calM = (\vP_\calG \vP_\calF)^+ (\vP_\calG \vP_\calF),& \quad
            \vP_\calN = (\vP_\calG \vP_\calF) (\vP_\calG \vP_\calF)^+ 
        \end{align}

        \item[(iii)] We have:
        \begin{align}
            \vP_\calM (\vP_\calF \vP_\calG)^{\phantom{+}} = &(\vP_\calF \vP_\calG)^{\phantom{+}} = (\vP_\calF \vP_\calG)^{\phantom{+}} \vP_\calN \\ 
            \vP_\calN (\vP_\calF \vP_\calG)^+ = &(\vP_\calF \vP_\calG)^+ = (\vP_\calF \vP_\calG)^+ \vP_\calM \\
            \vP_\calN (\vP_\calG \vP_\calF)^{\phantom{+}} = &(\vP_\calG \vP_\calF)^{\phantom{+}} = (\vP_\calG \vP_\calF)^{\phantom{+}} \vP_\calM \\ 
            \vP_\calM (\vP_\calG \vP_\calF)^+ = &(\vP_\calG \vP_\calF)^+ = (\vP_\calG \vP_\calF)^+ \vP_\calN
        \end{align}

        \item[(iv)] $\calM \subseteq \calF$ and $\calN \subseteq \calG$;

        \item[(v)] $\vP_\calF \vP_\calM = \vP_\calM = \vP_\calM \vP_\calF$ and  $\vP_\calG \vP_\calN = \vP_\calN = \vP_\calN \vP_\calG$;

        \item[(vi)] It holds $\vP_\calF \vP_\calG = \vP_\calM \vP_\calN$
        and in particular:
        \[  
            \textstyle  
            \vf(\vx)^\top \vg_0(y) %
            =  (\vP_\calM \vf(\vx))^\top  \vP_\calN \vg_0(y)
        \]

    \end{itemize}
\end{lemma*}

\begin{proof}
    (i) By the rank-nullity theorem \citep[page 62]{axler2015linear}, we have that $\mathrm{dim} \, \mathrm{Im} (\vP_\calF \vP_\calG)  = d - \mathrm{dim} \, \mathrm{ker} (\vP_\calF \vP_\calG)$. Notice that $\mathrm{Im}(\vP_\calF \vP_\calG) =\calM$ and $\mathrm{ker}(\vP_\calF \vP_\calG) = \bbR^d \setminus \calN$, therefore
    \[
        \mathrm{dim}(\calM) = \mathrm{dim} \, \mathrm{Im} (\vP_\calF \vP_\calG) = d - \mathrm{dim} \, \mathrm{ker} (\vP_\calF \vP_\calG) = d - d + \mathrm{dim} \, \mathrm{ker}(\vP_\calF \vP_\calG)^\perp = \mathrm{dim}(\calN) \, .
    \]
    Next, we derive the dimensionality of $\calM$. Recall that for two matrices $\vA \in \bbR^{m \times k}$ and $\vB \in \bbR^{k \times n}$, from the rank-nullity theorem \citep[page 62]{axler2015linear}, it follows that:
    \[
        \mathrm{rank} (\vA \vB) = \mathrm{rank}(\vB) - \mathrm{dim}(\mathrm{ker}(\vA) \cap \mathrm{Im}(\vB)) \, .
    \]
    Using this for $\vA = \vP_\calG$ and $\vB = \vP_\calF$, we have that:
    \begin{align}
        \mathrm{dim}(\calM) &= \mathrm{rank} (\vP_\calG \vP_\calF) \\
        &= \mathrm{rank}(\vP_\calF) - \mathrm{dim}(\mathrm{ker}(\vP_\calG) \cap \mathrm{Im}(\vP_\calF)) \\
        &= \mathrm{dim}(\calF) - \mathrm{dim}(\calG^\perp \cap \calF) \\
        &= \mathrm{dim}(\SIM{\vf}) - \mathrm{dim}(\SIM{\vg_0}^\perp \cap \SIM{\vf}) \, . 
    \end{align}

    (ii) From the property of the pseudo-inverse, see Equation~\eqref{eq:pseudoinverse-tt+} and Equation~\eqref{eq:pseudoincerse-t+t}, we have that
    \[  \label{eq:def-of-calm}
        (\vP_\calF \vP_\calG) (\vP_\calF \vP_\calG)^+ = \vP_{\mathrm{Im}(\vP_\calF \vP_\calG)} =\vP_\calM
    \]
    and that:
    \[  \label{eq:def-of-caln}
        (\vP_\calF \vP_\calG)^+ (\vP_\calF \vP_\calG)  = \vP_{ \mathrm{ker}(\vP_\calF \vP_\calG)^\perp} = \vP_\calN \, .
    \]
    From Equation~\eqref{eq:def-of-calm}, taking the transpose of $\vP_\calM$ we get:
    \begin{align}
        \vP_\calM^\top &= ((\vP_\calF \vP_\calG)^+)^\top  (\vP_\calF \vP_\calG)^\top \\
        \vP_\calM &= (\vP_\calG^\top \vP_\calF^\top )^+(\vP_\calG^\top \vP_\calF^\top) \\
        &= (\vP_\calG \vP_\calF)^+(\vP_\calG \vP_\calF) \, ,
    \end{align}
    and from Equation~\eqref{eq:def-of-caln}, taking the transpose we obtain:
    \begin{align}
        \vP_\calN^\top &=  (\vP_\calF \vP_\calG)^\top ((\vP_\calF \vP_\calG)^+)^\top \\
        \vP_\calN &= (\vP_\calG^\top \vP_\calF^\top ) (\vP_\calG^\top \vP_\calF^\top)^+ \\
        &= (\vP_\calG \vP_\calF) (\vP_\calG \vP_\calF)^+ \, .
    \end{align}

    (iii) Denote $\vA := \vP_\calF \vP_\calG$. From the pseudo-inverse definition \citep[page 221]{axler2015linear}, it holds
    \[
        (\vA \vA^+) \vA = \vA = \vA (\vA^+ \vA)
    \]
    and substituting for $\vP_\calF \vP_\calG$ and the projectors $\vP_\calM$ and $\vP_\calN$ we get:
    \[  \label{eq:propertyFG}
        \vP_\calM( \vP_\calF \vP_\calG) = \vP_\calF \vP_\calG = (\vP_\calF \vP_\calG) \vP_\calN \, .
    \]
    Similarly, for the pseudo-inverse, consider
    \[
        (\vA^+ \vA) \vA^+ = \vA^+ = \vA^+ (\vA \vA^+)
    \]
    and substituting we get
    \[  \label{eq:propertyFG+}
        \vP_\calN (\vP_\calF \vP_\calG)^+ = (\vP_\calF \vP_\calG)^+ = (\vP_\calF \vP_\calG)^+ \vP_\calM \, .
    \]
    Similarly to point (ii), taking the transpose of Equation~\eqref{eq:propertyFG}, we obtain:
    \[
        ( \vP_\calG \vP_\calF) \vP_\calM = \vP_\calG \vP_\calF = \vP_\calN (\vP_\calG \vP_\calF)  \, ,
    \]
    and taking the transpose of Equation~\eqref{eq:propertyFG+} we obtain:
    \[
        ( \vP_\calG \vP_\calF)^+ \vP_\calN = (\vP_\calG \vP_\calF)^+ = \vP_\calM (\vP_\calG \vP_\calF)^+  \, .
    \]
    
    (iv) To show this, notice that $\calM := \mathrm{Im} (\vP_\calF \vP_\calG) = \mathrm{ker} (\vP_\calG \vP_\calF)^\perp$. Therefore, we have that
    \[
        \mathrm{ker} (\vP_\calG \vP_\calF)^\perp = \{\vP_\calF \vv \not \in \calG^\perp \mid \vv \in \bbR^d\}, \quad \calF = \{ \vP_\calF \vv \neq \mathbf{0} \mid \vv \in \bbR^d \} \, ,
    \]
    which shows that $\calM$ is a subset of $\calF$, and they are equal only when $\vP_\calG = \vI$: $\calM \subseteq \calF$. Similarly, for $\calN$, we get the same: $\calN \subseteq \calG$.

    (v) Follows by the property of orthogonal projectors. Since $\calM \subseteq \calF$, we have that $\vP_\calM$ and$ \vP_\calF$ commute, which means that $\vP_\calM \vP_\calF = \vP_{\calM \cap \calF} = \vP_\calM$. The same conclusion also holds for $\calN$ and $\calG$, because of $\calN \subseteq \calG$.

    (vi) We have to rework the following expression
    \begin{align}
        \vP_\calF \vP_\calG
        &= \vP_\calM \vP_\calF \vP_\calG \vP_\calN \\
        &=\vP_\calM \vP_\calN  \, ,
    \end{align}
    where in the second line we used that $\vP_\calF \vP_\calG = \vP_\calM (\vP_\calF \vP_\calG) \vP_\calN$, and the final step is given by using that $\vP_\calM \vP_\calF = \vP_\calM$ and $\vP_\calG \vP_\calN = \vP_\calN$. In particular:
    \begin{align}
        \vf(\vx)^\top \vg_0(y) 
        &= \vf(\vx)^\top \vP_\calM \vP_\calN \vg_0(y)  \\
        &= (\vP_\calM \vf(\vx))^\top \vP_\calN \vg_0(y) \, ,
    \end{align}
    where in the last line we used that $\vP_\calM^\top = \vP_\calM$, being an orthogonal projector.
\end{proof}

    \subsection{Extended linear equivalence}
    \label{sec:app-equivalence-relation}
    
        We show that the relation defined in \cref{def:klinear-equiv} is an equivalence relation.
        
        \klinequiv*
        \begin{proof} To prove $\sim_{EL}$ is an equivalence relation we have to show its reflexivity, symmetry, and transitivity.

\textbf{Reflexivity}. Take:
\[
    (\vf, \vg) \sim_{EL} (\vf, \vg) \, .
\]
This means that there must exist $\vM, \vN$ of rank $k := \mathrm{dim}(\calM)$ such that,
\[
    \begin{cases}
        \vP_{\calM} {\vf}(\vx) &= \vM \vP_{\calM} {\vf}(\vx)  \\
        \vP_{\calN} {\vg}_0(y) &= \vN \vP_{\calN} {\vg}_0(y) 
    \end{cases}
\]
which are given by $\vM = \vP_\calM$, $\vN = \vP_\calN$.

\textbf{Symmetry}. We have to show that:
\[
    (\vf, \vg) \sim_{EL} (\tilde \vf, \tilde \vg) \iff (\tilde \vf, \tilde \vg)  \sim_{EL} (\vf, \vg) \, .
\]
This can be seen by showing one side of the implication ($\implies$):
\[
\begin{cases}
    \vP_{\calM} {\vf}(\vx) &= \vM \vP_{\tilde \calM} \tilde{\vf}(\vx) \\
    \vP_{\calN} {\vg}_0(y) &= \vN \vP_{\tilde \calN} \tilde{\vg}_0(y) 
\end{cases} 
\implies 
\begin{cases}
    \vP_{\tilde \calM} \tilde{\vf}(\vx)  &= \tilde \vM  \vP_{\calM} {\vf}(\vx)  \\
    \vP_{\tilde \calN} \tilde{\vg}_0(y) 
     &= \tilde \vN \vP_{\calN} {\vg}_0(y) 
\end{cases}
\]
Take:
\[
\begin{cases}
    \vP_{\calM} {\vf}(\vx) &= \vM \vP_{\tilde \calM} \tilde{\vf}(\vx)\\
    \vP_{\calN} {\vg}_0(y) &= \vN \vP_{\tilde \calN} \tilde{\vg}_0(y) 
\end{cases} 
\]
and consider the pseudo-inverses $\vM^+$ and $\vN^+$ obtaining:
\[
\begin{cases}
    \vM^+ \vP_{\calM} {\vf}(\vx)  &= \vM^+ \vM \vP_{\tilde \calM} \tilde{\vf}(\vx) \\
    \vN^+ \vP_{\calN} {\vg}_0(y)  &= \vN^+ \vN \vP_{\tilde \calN} \tilde{\vg}_0(y)
\end{cases} 
\]
Notice that $\vM^+ \vM = \vP_{\tilde{\calM}}$ and $\vN^+ \vN = \vP_{\tilde{\calN}}$, which follows by the fact that $\mathrm{ker}(\vM)^\perp = \tilde \calM$ and  $\mathrm{ker}(\vN)^\perp = \tilde \calN$ 
\citep[Page 211]{axler2015linear}. Using this we have $ \vM^+ \vM \vP_{\tilde{\calM}} = \vP_{\tilde{\calM}} \vP_{\tilde{\calM}}= \vP_{\tilde{\calM}} $, and similarly $\vN^+ \vN \vP_{\tilde{\calN}} = \vP_{\tilde{\calN}}$. 
Therefore,
on the right-hand side only the projectors $\vP_{\tilde{\calM}}$ and $\vP_{\tilde{\calN}}$ remain, whereas we have to set $\tilde{\vM} = \vM^+$ and $\tilde{\vN} = \vN^+$.
Both $\tilde \vM$ and $\tilde \vN$ have range $k$, as a consequence of being pseudo-inverses.
Therefore, we get $(\tilde{\vf}, \tilde{\vg}) \sim_{EL} (\vf, \vg)$. 

The other side of the implication ($\impliedby$) follows a similar proof.

\textbf{Transitivity}. We have to show that:
\[
    (\vf, \vg) \sim_{EL} (\tilde \vf, \tilde \vg) \; 
    \land  \;
    (\tilde \vf, \tilde \vg)  \sim_{EL} (\vf^*,  \vg^*) 
    \implies 
    (\vf, \vg)  \sim_{EL} (\vf^*, \vg^*) \, .
\]
This can be verified by substitution:
\begin{align}
& \begin{cases}
    \vP_{\calM} {\vf}(\vx) = \vM \vP_{\tilde \calM} \tilde{\vf}(\vx) \\
    \vP_{\calN} {\vg}_0(y) = \vN \vP_{\tilde \calN} \tilde{\vg}_0(y)
\end{cases} \\     
& \begin{cases}
    \vP_{\calM} {\vf}(\vx) = \vM \tilde \vM \vP_{ \calM^*} {\vf^*}(\vx)  \\
    \vP_{\calN} {\vg}_0(y) = \vN \tilde \vN \vP_{ \calN^*} {\vg^*}_0(y) 
\end{cases} 
\end{align}
and by setting $\overline{\vM} = \vM \tilde{\vM}$ and $\overline{\vN} = \vN \tilde{\vN}$, it holds:
\[
\begin{cases}
    \vP_{\calM} {\vf}(\vx) = \overline{\vM} \vP_{ \calM^*} {\vf^*}(\vx) \\
    \vP_{\calN} {\vg}_0(y) = \overline \vN  \vP_{ \calN^*} {\vg^*}_0(y) 
\end{cases}
\]
Notice that, since $\vM: \tilde{\calM} \to {\calM}$ is a linear invertible transformation, and similarly $\tilde \vM: \calM^* \to \tilde{\calM}$ is also a linear invertible transformation, the composition $\overline{\vM} = \vM \tilde{\vM}$ is a linear invertible transformation from $\calM^*$ to $\calM$, with $\mathrm{rank}(\overline{\vM}) = k$.
A similar observation also applies to $\overline{\vN}: \calN^* \to \calN$. Therefore, we have shown that:
\[
    (\vf, \vg)  \sim_{EL} (\vf^*, \vg^*) \, .
\]

This concludes the proof.
\end{proof}

\paragraph{Explicit form on $\vN$.} Based on the requirement that matrices $\vM, \vN \in \bbR^{d \times \tilde d}$ obey $\vM^\top \vN = \vP_{\tilde{\calM}} \vP_{\tilde{\calN}}$, we have that:
\[
    \vN := (\vP_\calM \vP_\calN)^+ (\vM^\top)^+ (\vP_{\tilde{\calM}} \vP_{\tilde{\calN}})
\]

\begin{proof}
    Notice that the following holds:
    \begin{itemize}
        \item $\vM^\top \vP_\calF = \vM^\top$, by the fact that $\vM^\top= \vM^\top \vP_\calM = \vM^\top \vP_\calM \vP_\calF$ by \cref{lemma:projectors} (v), and
        
        \item $\vP_\calG (\vP_\calF \vP_\calG)^+ = (\vP_\calF \vP_\calG)^+$, since we have $(\vP_\calF \vP_\calG) \vP_\calG = (\vP_\calF \vP_\calG)$.
    \end{itemize}
    We use these identities to obtain:
    \begin{align}
        \vM^\top (\vP_\calF \vP_\calG)^+ (\vM^\top)^+ (\vP_{\tilde{\calF}} \vP_{\tilde{\calG}}) 
        &= \vM^\top {\vP_\calF \vP_\calG} (\vP_\calF \vP_\calG)^+ (\vM^\top)^+ (\vP_{\tilde{\calF}} \vP_{\tilde{\calG}})  
        \\
        &= \vM^\top {\vP_\calM} (\vM^\top)^+ (\vP_{\tilde{\calF}} \vP_{\tilde{\calG}}) 
        \tag{Using \cref{lemma:projectors} (ii)}
        \\
        &= \HL{\vM^\top} (\vM^\top)^+ (\vP_{\tilde{\calF}} \vP_{\tilde{\calG}}) 
        \tag{Using $\vM^\top \vP_\calM = \vM^\top$}
        \\
        &=  \HL{\vP_{\tilde \calM}} (\vP_{\tilde{\calF}} \vP_{\tilde{\calG}}) 
        \tag{$\vM^\top (\vM^\top)^+ = \vP_{\mathrm{ker}(\vM)^\perp} =\vP_\calM$}
        \\
        &=\vP_{\tilde{\calF}} \vP_{\tilde{\calG}} \, . %
    \end{align}
    where we used that $\vP_{\tilde \calM} (\vP_{\tilde{\calF}} \vP_{\tilde{\calG}}) = \vP_{\tilde{\calF}} \vP_{\tilde{\calG}}$ by \cref{lemma:projectors} (iii). 
    Notice that by \cref{lemma:projectors} (vi) we have that
    $\vP_{\tilde{\calF}} \vP_{\tilde{\calG}} = \vP_{\tilde \calM} \vP_{\tilde \calN}$.
\end{proof}

    \subsection{Proof of \cref{prop:elpreserves}}

        \elpreserves*
        \begin{proof}
    \cchanged{Starting from \cref{lemma:projectors} (vi), we have that:
    \[
        \vf(\vx)^\top\vg_0(y) = \vf(\vx)^\top \vP_\calM \vP_\calN \vg_0(y) \, .
    \]
    }
    \cchanged{Considering} the expression of $\vf$ and $\vg_0$ given by the $\sim_{EL}$ equivalence relation (\cref{def:klinear-equiv}):
    \begin{align}
        \vP_\calM \vf(\vx) &= \vM \vP_{\tilde \calM} \tilde \vf(\vx) \label{eq:eleq-top}\\
        \vP_\calN \vg_0(y) &= \vN \vP_{\tilde \calN} \tilde \vg_0(y) \label{eq:eleq-bottom}
    \end{align}
    we use also the condition that $\vM^\top \vN = \vP_{\tilde\calM} \vP_{\tilde\calN}$.
    We get
    \begin{align*}
        \vf(\vx)^\top \vP_\calM \vP_\calN \vg_0(y) &= \tilde{\vf}(\vx)^\top \vP_{\tilde \calM} \vM^\top \vN \vP_{\tilde \calN} \tilde \vg_0(y) \\
        &= \tilde{\vf}(\vx)^\top \vP_{\tilde \calM} \vP_{\tilde \calM} \vP_{\tilde \calN} \vP_{\tilde \calN} \tilde \vg_0(y) \\
        &= \tilde{\vf}(\vx)^\top \vP_{\tilde \calM} \vP_{\tilde \calN} \tilde \vg_0(y) \, ,
    \end{align*}
    where we used the {idempotency of projectors, i.e.,} that $\vP_{\tilde \calM}^2 = \vP_{\tilde \calM}$ and $\vP_{\tilde \calN}^2 = \vP_{\tilde \calN}$.
    \cchanged{To prove the claim, it is sufficient to use \cref{lemma:projectors} (vi) again, obtaining:}
    \begin{align}
        \vf(\vx)^\top\vg_0(y) &= \tilde{\vf}(\vx)^\top \vP_{\tilde \calM} \vP_{\tilde \calN} \tilde \vg_0(y) 
        = \tilde{\vf}(\vx)^\top \tilde \vg_0(y) \, .
    \end{align}
\end{proof}

    \subsection{Counterexample when diversity condition does not hold}
    \label{sec:counter-example-diversity}

\cchanged{We detail here a counter-example to linear identifiability (\cref{cor:linear-id}) when the diversity condition does not hold.}
Let $\vf: \SeqA \to \bbR^2$ and $\vg:\calA \to \bbR^2$, and $\calA = \{y_0, y_1, y_2 \}$. Let $\vg(y_0)= (1,0)^\top$, $\vg(y_1)= (1,1)^\top$, and $\vg(y_2)= (1,-1)^\top$ be unembeddings, which do not fulfill the diversity condition (\cref{def:diversity-condition}): 
In fact, these unembeddings give $\calG = \mathrm{span}(\ve_2)$, and $\dim(\calG)=1$ which is less than the dimensionality of the representation space. %
The vector $\ve_2 := (0,1)^\top$ is drawn as a \textcolor{blue}{\textbf{blue}} arrow in \cref{fig:counter-example}. 
We can construct another model where $\tilde \vg = \vg$
and choose $\tilde \vf(\vx) = (\vf_1(\vx) + 0.2 \cos (40 a_1/ \pi), \vf_2(\vx)  )^\top$. 
\cref{fig:counter-example} shows this transformation.
By construction, this model generates the same next-token distribution of the first one; however, the two model representations are not equal up to a linear transformation.

\begin{figure}[!h]
    \centering
    \begin{tabular}{ccc}
        Model $(\vf, \vg)$ & 
        & \textcolor{white}{aba}
        Model $(\tilde \vf, \tilde \vg)$
        \\
         \includegraphics[width=0.3\textwidth]{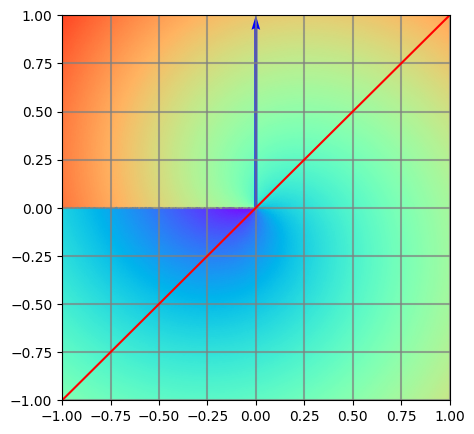}
         & & \includegraphics[width=0.3\textwidth]{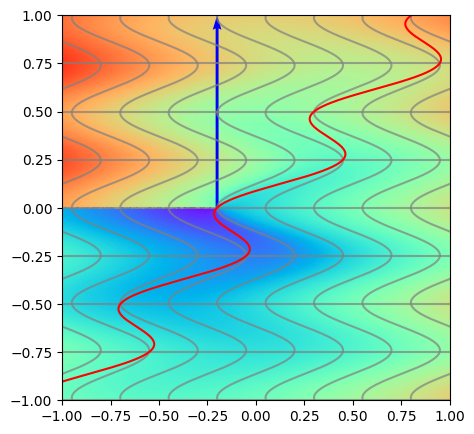}
    \end{tabular}
    \caption{\textbf{Allowed distorsions among $\sim_{EL}$-equivalent models}. 
    From the left, model embeddings $\vf$ are given different colors. The \textcolor{red}{red} segment is non-linearly transformed on the right along $\tilde \vf_1$, whereas they remain equal to the left on the component $\tilde \vf_2$. This shows that the models are not $\sim_L$-equivalent.
    }
    \label{fig:counter-example}
\end{figure}

    \subsection{Proof of \cref{thm:partial-identifiability}}

        We begin with the following lemma, which will be used in the proof of~\cref{thm:partial-identifiability}. 
        \begin{lemma} \label{lemma:rank-equality}
            Let $\calF$ be a subspace of $\bbR^d$ and $\calM \subseteq \calF$ a subspace of $\calF$. 
            Consider $D$ elements $\vv_i \in \bbR^d$, such that the matrix
            \[
                \vF := \begin{pmatrix}
                    \vv_1, \ldots, \vv_D
                \end{pmatrix}
            \]
            has range $\mathrm{Im}(\vF) = \calF$. Then, it holds $\mathrm{rank}(\vP_\calM \vF) = \mathrm{dim}(\calM)$.
        \end{lemma}

        \begin{proof}   

For two matrices $\vA \in \bbR^{m \times k}$ and $\vB \in \bbR^{k \times n}$, from the rank-nullity theorem \citep[page 62]{axler2015linear}, it follows that:
\[
    \mathrm{rank} (\vA \vB) = \mathrm{rank}(\vB) - \mathrm{dim}(\mathrm{ker}(\vA) \cap \mathrm{range}(\vB)) \, .
\]
We use this formula to obtain:
\begin{align}
    \mathrm{rank} (\vP_\calM \vF) &= \mathrm{rank} (\vF) - \mathrm{dim} (\mathrm{ker}(\vP_\calM) \cap \mathrm{range} (\vF)) \\
    &= \mathrm{dim}(\calF) - \mathrm{dim} \big( \calM^\perp \cap \calF \big) \\
    &= \mathrm{dim}(\calF) - \mathrm{dim} \big( \bbR^d \setminus \calM \cap \calF \big) \\
    &= \mathrm{dim}(\calF) - \mathrm{dim} \big( \calF \setminus \calM \big) \\
    &= \mathrm{dim}(\calF) - \mathrm{dim} (\calF) + \mathrm{dim} (\calM) \\
    &= \mathrm{dim} (\calM) 
\end{align}
where the equality on the second-last line follows from the fact that $\calM \subseteq \calF$. 
\end{proof}

        We now proceed to prove \cref{thm:partial-identifiability}.
        
        \partialidf*

\textbf{Proof sketch.}  
To prove the implication, we divide into five parts:
\begin{enumerate}

    \item Starting from log-equality of probabilities,
    we adopt a pivoting strategy to get rid of normalizing constants\footnote{A similar ``pivoting'' strategy is the starting point of several identifiability proofs in nonlinear ICA with auxiliary variables, e.g.,~\citep{hyvarinen2019nonlinear, khemakhem2020variational, khemakhem2020ice, roeder2021linear}.};

    \item We derive an explicit expression of $\vM$ 
    and $\tilde{\vM}$
    such that
    \begin{align*}
        \vP_\calM \vf(\vx) &= \vM \vP_{\tilde \calM} \tilde{\vf}(\vx) \\ 
        \vP_{\tilde \calM} \tilde{\vf}(\vx) &= \tilde{\vM} \vP_\calM \vf(\vx)
    \end{align*}    
    hold for every $\vx \in \SeqA$. We achieve that by considering $q$ points, or tokens, of $\calA$, such that the matrices $\vG := (\vg_0(y_1), \ldots, \vg_0(y_q))$ and $\tilde \vG := (\tilde{\vg}_0(y_1), \ldots, \tilde{\vg}(y_q))$ have images corresponding to $\SIM{\vg_0}$ and $\SIM{\tilde{\vg}_0}$, respectively;

    \item 
    From the linear relation obtained between $\vf$ and $\tilde \vf$, we show that, having $\mathrm{dim}(\calM) = k$, also $\mathrm{dim}(\tilde \calM) = k$ and $\mathrm{rank} (\vM) = k$.
    By considering $\ell$ points, or sequences, $x_i \in \SeqA$, such that the matrices $\vF := (\vf(\vx_q), \ldots, \vf(\vx_\ell))$ and $\tilde \vF := (\tilde \vf(\vx_q), \ldots, \tilde \vf(\vx_\ell))$ have images corresponding to $\SIM{\vf}$ and $\SIM{\tilde{\vf}}$, respectively, we will obtain that
    $$  k \leq m \leq k, \quad k \leq k' \leq k \, ;
    $$

    \item We derive an explicit expression of $\vN$ such that
    $$
        \vP_\calN \vg_0(y) = \vN \vP_{\tilde \calN} \tilde{\vg}_0 (y)
    $$
    hold for every $y \in \calA$. We achieve that by considering again $\ell$ points such that $\vF$ and $\tilde{\vF}$ have images corresponding to $\SIM{\vf}$ and $\SIM{\tilde{\vf}}$, and using the linear relation obtained between $\vf$ and $\tilde{\vf}$. This allows us to derive an expression of $\vN$ that depends on $\vM$;

    \item Finally, we show that it follows that $\mathrm{rank}(\vN)$ is equal to $k := \mathrm{dim}(\calM)$. To achieve that we use the relation between $\vg_0$ and $\tilde{\vg}_0$ to show that:
    $$ 
    k \leq \mathrm{rank}(\vN) \leq k \, .
    $$
\end{enumerate}

Recall that we are indicating with:
$$
    \calF := \SIM{\vf}, \quad \tilde{\calF} := \SIM{\tilde{\vf}}, \quad \calG := \SIM{\vg_0}, \quad  \tilde{\calG} := \SIM{\tilde{\vg}_0}\, ,
$$
and with:
\[
    \calM := \mathrm{Im} (\vP_\calF \vP_\calG), \quad \tilde \calM := \mathrm{Im} (\vP_{\tilde{\calF}} \vP_{\tilde{\calG}}), \quad
    \calN := \mathrm{ker} (\vP_\calF \vP_\calG)^\perp, \quad \tilde \calN := \mathrm{ker} (\vP_{\tilde{\calF}} \vP_{\tilde{\calG}} )^\perp \ .
\]
\begin{proof}
    We start by proving the implication ($\implies$).
    
    \textbf{Step 1: Using pivoting to get rid of normalizing constants.}
    We start from the equality between log probabilities:
    \[
        \vf(\vx)^\top\vg(y) - \log Z(\vx) = \tilde{\vf}(\vx)^\top\tilde{\vg}(y) - \log \tilde{Z}(\vx)
    \]
    Subtracting the pivot $y_0 \in \calA$ 
 to all remaining elements $y \in \calA$, and defining $\vg_0 := \vg(y) -\vg(y_0)$ and $\tilde \vg_0(y) := \tilde \vg(y) - \tilde \vg(y_0)$, we get rid of the terms containing the log of the normalizing constant, \ie
    \begin{align}
        \vf(\vx)^\top(\vg(y) - \vg(y_0))  &= \tilde{\vf}(\vx)^\top(\tilde{\vg}(y) -\tilde{\vg}(y_0))
        \tag{Subtract by $\vf(\vx)^\top\vg(y_0)$ and by $\tilde{\vf}(\vx)^\top \tilde{\vg}(y_0)$}
        \\
        \vf(\vx)^\top\vg_0(y)  &= \tilde{\vf}(\vx)^\top\tilde{\vg}_0(y) \, .
        \label{eq:log-prob-nonorm}
    \end{align}

    \textbf{Step 2: Obtaining the relation between embeddings $\vf$ and $\tilde \vf$.}
    We will now show that $\vM:= (\vP_\calG \vP_\calF)^+ (\vG^\top)^+ \tilde{\vG}^\top (\vP_{\tilde{\calG}} \vP_{\tilde{\calF}})$ and $\tilde{\vM} := (\vP_{\tilde \calG} \vP_{\tilde \calF})^+ (\tilde \vG^\top)^+ {\vG}^\top \vP_{{\calG}} \vP_{{\calF}}
    $
    satisfies Equation~\eqref{eq:linear-transf-f}, showing that $\vf$ and $\tilde{\vf}$ are linearly related. 
    For this, consider $q$ elements, or tokens, $y \in \calA$ such that the matrix
    \[
        \label{eq:g-gtilse-spans}
        \vG = \begin{pmatrix}
            \vg_0(y_1),  &  \cdots, & \vg_0(y_q)
        \end{pmatrix}, \quad
        \tilde \vG = \begin{pmatrix}
            \tilde \vg_0(y_1),  &  \cdots, & \tilde \vg_0(y_q)
        \end{pmatrix}
    \]
    span $\calG$ and $\tilde{\calG}$, respectively, \ie  $\mathrm{Im}(\vG) = \calG$ and $\mathrm{Im}(\tilde \vG) = \tilde \calG$. 
    Taking the transpose of Equation~\eqref{eq:log-prob-nonorm} and considering these $q$ elements, we can write
    \[  \label{eq:matrix-equality-G}
        \vG^\top \vf(\vx) = \tilde{\vG}^\top \tilde{\vf}(\vx) \, .
    \]
   Next, we consider the projectors on the subspace where all functions' images are contained. Indicate with $\vP_\calF$ and $\vP_\calG$ the orthogonal projectors on $\calF$ and $\calG$, respectively, and with $\vP_{\tilde{\calF}}$ and $\vP_{\tilde{\calG}}$ the orthogonal projectors on $\tilde{\calF}$ and $\tilde{\calG}$, respectively.
   It holds
    \[
        \vf(\vx) = \vP_\calF \vf(\vx), \quad \vG = \vP_\calG \vG, \quad
        \tilde \vf(\vx) = \vP_{\tilde{\calF}} \tilde \vf(\vx), \quad \tilde \vG = \vP_{\tilde{\calG}} \tilde \vG \, ,
        \label{eq:inserting-projs}
    \]
    which can be inserted in Equation~\eqref{eq:matrix-equality-G}, leading to
    \[
        \vG^\top \vP_\calG \vP_\calF \vf(\vx) = \tilde{\vG}^\top \vP_{\tilde{\calG}} \vP_{\tilde{\calF}} \tilde{\vf} (\vx) \, .
    \]
    We make use of the pseudo-inverse of $\vG^\top$ to obtain:
    \begin{align}
        \HL{(\vG^\top)^+} \vG^\top \vP_\calG \vP_\calF \vf(\vx) &=  \HL{(\vG^\top)^+} \tilde{\vG}^\top \vP_{\tilde{\calG}} \vP_{\tilde{\calF}} \tilde{\vf}(\vx) 
        \tag{Multiply on the left by $(\vG^\top)^+$}
        \\
        \HL{\vP_{\mathrm{Im}(\vG)}} \vP_\calG \vP_\calF \vf(\vx) &=  (\vG^\top)^+ \tilde{\vG}^\top \vP_{\tilde{\calG}} \vP_{\tilde{\calF}} \tilde{\vf}(\vx) 
        \tag{From Equation~\eqref{eq:pseudoincerse-ttop+ttop},  $(\vG^\top)^+ \vG^\top = \vP_{\mathrm{Im}(\vG)}$}
        \\
        \HL{\vP_\calG} \vP_\calF \vf(\vx) &=  \HL{\vA} \vP_{\tilde{\calG}} \vP_{\tilde{\calF}} \tilde{\vf}(\vx) \, ,
        \label{eq:thm5-first-step}
    \end{align}
    where in the last line we used $\vP_{\mathrm{Im}(\vG)} \vP_\calG = \vP_\calG \vP_\calG = \vP_\calG$ since $\mathrm{Im}(\vG) = \calG$, and
    we denoted with $\vA := (\vG^\top)^+ \tilde{\vG}^\top$. 
    
    Next, we insert the projectors $\vP_\calM$ and $\vP_{\tilde \calM}$ in Equation~\eqref{eq:thm5-first-step}, using \cref{lemma:projectors} (iii) we get:
    \begin{align}
        \vP_\calG \vP_\calF \HL{\vP_\calM} \vf(\vx) &=  \vA\vP_{\tilde{\calG}} \vP_{\tilde{\calF}} \HL{\vP_{\tilde{\calM}}} \tilde{\vf}(\vx) \, .
    \end{align}
    We now consider the left pseudo-inverse of $\vP_\calG \vP_\calF$ to obtain
    \begin{align}
        \HL{(\vP_\calG \vP_\calF)^+} \vP_\calG \vP_\calF \vP_\calM \vf(\vx) &= \HL{(\vP_\calG \vP_\calF)^+} \vA\vP_{\tilde{\calG}} \vP_{\tilde{\calF}} \vP_{\tilde{\calM}} \tilde{\vf}(\vx) 
        \tag{Multiply on the left by $(\vP_\calG \vP_\calF)^+$}
        \\
        \HL{\vP_\calM} \vP_\calM \vf(\vx) &= (\vP_\calG \vP_\calF)^+ \vA\vP_{\tilde{\calG}} \vP_{\tilde{\calF}} \vP_{\tilde{\calM}} \tilde{\vf}(\vx) 
        \tag{By \cref{lemma:projectors} (ii) $\vP_\calM = (\vP_\calG \vP_\calF)^+(\vP_\calG \vP_\calF)$}
        \\
        \HL{\vP_\calM} \vf(\vx) &= \HL{\vM} \vP_{\tilde{\calM}} \tilde{\vf}(\vx) \, , \label{eq:lin-subspace-equiv}
    \end{align}
    where we {used the idempotency of the $\vP_\calM$ projector, and we defined} 
    \[
        \label{eq:f-and-tild}
        {\vM} := (\vP_\calG \vP_\calF)^+ (\vG^\top)^+ \tilde{\vG}^\top \vP_{\tilde{\calG}} \vP_{\tilde{\calF}} \, .
    \]
    
    Following similar steps but starting from $\tilde{\vG}^\top \tilde \vf(\vx) = {\vG}^\top \vf(\vx)$, we arrive at a similar expression for $\vP_{\tilde \calM} \tilde{\vf}$:
    \[
        \vP_{\tilde{\calM}} \tilde \vf(\vx) = \tilde{\vM} \vP_{{\calM}} {\vf}(\vx) \, ,
        \label{eq:tilde-lin-space-equiv}
    \]
    where $\tilde{\vM} := (\vP_{\tilde \calG} \vP_{\tilde \calF})^+ (\tilde \vG^\top)^+ {\vG}^\top \vP_{{\calG}} \vP_{{\calF}}$.

    \textbf{Step 3: Showing that} $\mathrm{dim}(\tilde \calM) = \mathrm{rank}(\vM) = \mathrm{dim}(\calM)$.
    Let $k := \mathrm{dim}(\calM)$ and $k' := \mathrm{dim}(\tilde{\calM})$. Also, let $m:= \mathrm{rank}(\vM)$.
    We want to show that $k' = m = k$ and to this end we will obtain that $k \leq m \leq k$ and that $k' \leq k \leq k$. %
    This is done in three points:
    \begin{itemize}[leftmargin=2.5em]
        \item[(\textbf{I})] We show that $m \leq k$ from the definition of $\vM$ in Equation~\eqref{eq:f-and-tild};

        \item[(\textbf{II})]  We show that necessarily $m \geq k$ and $k' \geq k$ from Equation~\eqref{eq:lin-subspace-equiv};

        \item[(\textbf{III})] We show that $k' \leq k$ from Equation~\eqref{eq:tilde-lin-space-equiv}.
    \end{itemize}
    
    (\textbf{I})
    By equation~\eqref{eq:f-and-tild}, we have that $ \vM = \vP_\calM \vM$, since by Equation~\eqref{eq:f-and-tild}, in $\vM$ we have on the left $\vP_\calM (\vP_\calG \vP_\calF)^+ = (\vP_\calG \vP_\calF)^+$ by \cref{lemma:projectors} (ii). 
    Taking the rank of $\vM$ and using the fact that $\mathrm{rank} (\vA \vB) \leq \min \big( \mathrm{rank}(\vA), \mathrm{rank}(\vB) \big)$
    \citep{axler2015linear}, we obtain:
    \begin{align}
        \mathrm{rank} (\vM) &= \mathrm{rank} (\vP_\calM \vM)
        \tag{Take the rank of $\vM = \vP_\calM \vM$}
        \\
        m & \leq \min \big(\mathrm{rank}(\vP_\calM),  \mathrm{rank} (\vM ) \big)
        \tag{Using $\mathrm{rank} (\vA \vB) \leq \min \big( \mathrm{rank}(\vA), \mathrm{rank}(\vB) \big)$}
        \\
        m & \leq \min \big( m, k \big) \\
        \implies m &\leq k
        \label{eq:m-less-k}
    \end{align}

    Next, %
    we consider $\ell$ sequences $\vx_i \in \SeqA$ such that the matrices:
    \[
        \vF = \begin{pmatrix}
            \vf(\vx_1), &  \cdots, & \vf(\vx_\ell)
        \end{pmatrix}, \quad 
        \tilde \vF = \begin{pmatrix}
            \tilde \vf(\vx_1), &  \cdots, & \tilde \vf(\vx_\ell)
        \end{pmatrix}
        \label{eq:f-ftilde-span}
    \]
    span the whole $\calF$ and $\tilde{\calF}$, respectively, \ie  $\mathrm{Im} (\vF) = \calF$ and $\mathrm{Im} (\tilde \vF) = \tilde{\calF}$.  
    Moreover, since $\calM \subseteq \calF$ we have that by \cref{lemma:rank-equality} that $\mathrm{rank} (\vP_\calM \vF) =k$. Similarly, we have $\mathrm{rank} (\vP_{\tilde \calM} \tilde \vF) =k'$. 
    
    (\textbf{II})
    We consider from the Equation~\eqref{eq:lin-subspace-equiv} the condition for $\ell$ points:
    \begin{align}
        \vP_\calM \vF &= \vM \vP_{\tilde{\calM}} \tilde{\vF} \, .
        \label{eq:F-and-tildeF} 
    \end{align}  

    We evaluate the rank from Equation~\eqref{eq:F-and-tildeF} to obtain:
    \begin{align}
        \mathrm{rank} (\vP_\calM \vF) &= \mathrm{rank} (\vM \vP_{\tilde{\calM}} \tilde{\vF}) \\
        k &\leq \min \big( \mathrm{rank}(\vM), \mathrm{rank}(\vP_{\tilde{\calM}} \tilde{\vF}) \big) 
        \tag{Using $\mathrm{rank} (\vA \vB) \leq \min \big( \mathrm{rank}(\vA), \mathrm{rank}(\vB) \big)$}
        \\
         k &\leq \min \big(m, k' \big) \\
         \implies m &\geq k, \quad k' \geq k
         \label{eq.k'-greater-k}
    \end{align}
    Together with $\cref{eq:m-less-k}$, it shows that $k \leq m \leq k$, so it must be that $m = k$ and so $\mathrm{rank}(\vM) = \mathrm{dim}(\calM)$. 

    (\textbf{III}) 
    From the Equation~\eqref{eq:tilde-lin-space-equiv}, we get similarly:
    \[
        \vP_{\tilde{\calM}} \tilde{\vF} = \tilde \vM \vP_\calM \vF \, . \label{eq:tildeF-and-F} 
    \]
    Following a similar proof to (\textbf{II}), we evaluate the rank from Equation~\eqref{eq:tildeF-and-F} to obtain:
    \begin{align}
        \mathrm{rank} (\vP_{\tilde \calM} \tilde \vF) %
        &\leq \min \big(\mathrm{rank}(\tilde \vM), \mathrm{rank}(\vP_{\calM} {\vF}) \big) 
        \tag{Using $\mathrm{rank} (\vA \vB) \leq \min \big( \mathrm{rank}(\vA), \mathrm{rank}(\vB) \big)$}
        \\
        \implies k' &\leq k
    \end{align}
    which, together with \cref{eq.k'-greater-k}, it holds only as long as $k' \leq k$. This shows that $k \leq k' \leq k$, so it must be that $k' = k$. 
    
    Hence, we have shown that $\mathrm{dim}(\tilde \calM) = \mathrm{dim}(\calM) = \mathrm{rank}(\vM)$.
    In particular, it holds that $\vM$ is an invertible map from $\tilde \calM$ to $\calM$, with pseudo-inverse $\vM^+$, such that:
    \[  \label{eq:properties-of-m}
        {\mathrm{Im}(\vM)} = \calM, \quad {\mathrm{ker}(\vM)^\perp} = {\tilde \calM} \, .
    \]

    \textbf{Step 4. Obtaining the relation between unembeddings $\vg$ and $\tilde \vg$}.
    We will now show that $\vN:= (\vP_\calF \vP_\calG)^+ ( \vM^\top)^+ \vP_{\tilde{\calF}} \vP_{\tilde{\calG}}$, 
    using the matrix  $\vM$ in Equation~\eqref{eq:f-and-tild} from \textbf{Step 2} satisfies Equation~\eqref{eq:linear-transf-g}.

    Similar to \textbf{Step 2}, we take $\ell$ points, or sequences, $\vx_i \in \SeqA$ such that $\vF$ and $\tilde{\vF}$ in Equation~\eqref{eq:f-ftilde-span} span $\calF$ and $\tilde{\calF}$, respectively. 
    We then have:
    \begin{align}
         \vF^\top \vg_0(y) &= \tilde{\vF}^\top \tilde \vg_0(y)
         \tag{Considering $\ell$ points for $\vF$ and $\tilde{\vF}$}
         \\
         \vF^\top \HL{\vP_\calF \vP_\calG} \vg_0(y) &= \tilde{\vF}^\top \HL{\vP_{\tilde{\calF}} \vP_{\tilde{\calG}}} \tilde \vg_0(y) 
         \tag{Using orthogonal projectors Equation~\eqref{eq:inserting-projs}}
         \\
         \vF^\top \HL{\vP_\calM} \vP_\calF \vP_\calG \vg_0(y) &= \tilde{\vF}^\top  \vP_{\tilde \calM} \vP_{\tilde{\calF}} \vP_{\tilde{\calG}} \tilde \vg_0(y) 
         \tag{Using \cref{lemma:projectors} (iii)}
         \\
         \HL{(\vP_\calM \vF)^\top} \vP_\calF \vP_\calG \vg_0(y) &= \tilde{\vF}^\top \vP_{\tilde \calM} \vP_{\tilde{\calF}} \vP_{\tilde{\calG}} \tilde \vg_0(y) 
         \tag{Taking the transpose $\vF^\top \vP_\calM = (\vP_\calM \vF)^\top$}
         \\
         \HL{(\vM \vP_{\tilde \calM} \tilde \vF)^\top} \vP_\calF \vP_\calG \vg_0(y) &= \tilde{\vF}^\top  \HL{\vP_{\tilde \calM}} \vP_{\tilde{\calF}} \vP_{\tilde{\calG}}  \tilde \vg_0(y) \label{eq:implicit-transpose}
     \end{align}

    where in the last line we substituted the expression for $\vP_\calM \vF$ given by Equation~\eqref{eq:f-and-tild}.
    Thus, restarting from \eqref{eq:implicit-transpose}, and reworking the expression we get:
    \begin{align}             
         \HL{\tilde \vF^\top  \vP_{\tilde \calM}  \vM^\top} \vP_\calF \vP_\calG \vg_0(y) &= \tilde{\vF}^\top \vP_{\tilde \calM} \vP_{\tilde{\calF}} \vP_{\tilde{\calG}} \tilde \vg_0(y)  
         \tag{Expanding the transpose on the left}
         \label{eq:explicit-transpose}\\
         \HL{(\tilde \vF^\top)^+} {\tilde \vF^\top  \vP_{\tilde \calM}  \vM^\top} \vP_\calF \vP_\calG \vg_0(y) &= \HL{(\tilde \vF^\top)^+} \tilde{\vF}^\top \vP_{\tilde \calM} \vP_{\tilde{\calF}} \vP_{\tilde{\calG}} \tilde \vg_0(y)  
         \tag{Multiply by pseudo-inverse $(\tilde{\vF}^\top)^+$}
         \\
         \HL{\vP_{\tilde{\calF}} \vP_{\tilde{\calM}}} \vM^\top \vP_\calF \vP_\calG \vg_0(y) &= \HL{\vP_{\tilde{\calF}}} \vP_{\tilde{\calG}} \tilde \vg_0(y) 
         \tag{From Equation~\eqref{eq:pseudoincerse-ttop+ttop} we get $(\tilde{\vF}^\top)^+ (\tilde{\vF}^\top) = \vP_{\mathrm{Im}(\tilde{\vF})} = \vP_{\tilde{\calF}}$}
         \\
         \HL{\vM^\top} \vP_\calF \vP_\calG \vg_0(y) &= \vP_{\tilde{\calF}} \vP_{\tilde{\calG}} \tilde \vg_0(y) \, ,
         \label{eq:thm5-step-2}
    \end{align}
    where we used in the last line that $\vP_{\tilde{\calF}} \vP_{\tilde{\calM}}  \vM^\top = \vP_{\tilde{\calM}}  \vM^\top = \vM^\top $, and $\vP_{\tilde \calM} \vM^\top = \vM^\top$ follows by the definition of $\vM$ (Equation~\eqref{eq:f-and-tild}), containing on the right $(\vP_{\tilde \calG }\vP_{\tilde \calF})$.
    Recall that by \cref{lemma:projectors} (iii), it holds that
    \[
        \vP_\calF \vP_\calG = \vP_\calF \vP_\calG \vP_\calN, \quad
        \vP_{\tilde{\calF}} \vP_{\tilde{\calG}} = \vP_{\tilde{\calF}} \vP_{\tilde{\calG}} \vP_{\tilde{\calN}} \, .
        \label{eq:trick-n-ntilde}
    \]
    Using this in Equation~\eqref{eq:thm5-step-2}, we obtain
    \begin{align}
        \vM^\top \vP_\calF \vP_\calG \HL{\vP_\calN} \vg_0(y) &= \vP_{\tilde{\calF}} \vP_{\tilde{\calG}} \HL{\vP_{\tilde \calN}}  \tilde \vg_0(y)
        \\
        \HL{(\vM^\top)^+}  \vM^\top \vP_\calF \vP_\calG \vP_\calN \vg_0(y) &= \HL{(\vM^\top)^+} \vP_{\tilde{\calF}} \vP_{\tilde{\calG}} \vP_{\tilde \calN}  \tilde \vg_0(y) 
        \tag{Multiply on the left for the pseudo-inverse $(\vM^\top)^+$}
        \\
        \HL{\vP_\calM} \vP_\calF \vP_\calG \vP_\calN \vg_0(y) &= ( \vM^\top)^+ \vP_{\tilde{\calF}} \vP_{\tilde{\calG}} \vP_{\tilde \calN}  \tilde \vg_0(y) \, ,
    \end{align}
    where in the last line we use $(\vM^\top)^+ \vM^\top  = \vP_{\mathrm{Im}(\vM)} = \vP_\calM$, where the first equality follows by Equation~\eqref{eq:pseudoincerse-ttop+ttop} and the second one is given by $\mathrm{Im}(\vM) = \calM$ from Equation~\eqref{eq:properties-of-m}. 
    We now use that
    $
        \vP_\calM \vP_\calF \vP_\calG = \vP_\calF \vP_\calG
    $ 
    to obtain:
    \begin{align}
        \vP_\calF \vP_\calG \vP_\calN \vg_0(y) &= ( \vM^\top)^+ \vP_{\tilde{\calF}} \vP_{\tilde{\calG}} \vP_{\tilde \calN}  \tilde \vg_0(y) \\
        \HL{(\vP_\calF \vP_\calG)^+} \vP_\calF \vP_\calG \vP_\calN \vg_0(y) &= \HL{(\vP_\calF \vP_\calG)^+} ( \vM^\top)^+ \vP_{\tilde{\calF}} \vP_{\tilde{\calG}} \vP_{\tilde \calN}  \tilde \vg_0(y)
        \tag{Multiply by pseudo-inverse $(\vP_\calF \vP_\calG)^+$}
        \\
        \HL{\vP_\calN} \vP_\calN \vg_0(y) &= (\vP_\calF \vP_\calG)^+ ( \vM^\top)^+ \vP_{\tilde{\calF}} \vP_{\tilde{\calG}} \vP_{\tilde \calN}  \tilde \vg_0(y) \tag{$(\vP_\calF \vP_\calG)^+\vP_\calF \vP_\calG =\vP_\calN$, by \cref{lemma:projectors} (ii)}
        \\
        \HL{\vP_\calN} \vg_0(y) &= \HL{\vN} \vP_{\tilde \calN}  \tilde \vg_0(y) \label{eq:are-we-done}\, , 
    \end{align}
    where we set $ \vN := (\vP_\calF \vP_\calG)^+ ( \vM^\top)^+ \vP_{\tilde{\calF}} \vP_{\tilde{\calG}}$. This expression is in line with that of \cref{def:klinear-equiv}, showing the equivalence relation.

    \textbf{Step 5. Showing that} $\mathrm{rank}(\vN) = \mathrm{dim}(\calN)$. 
    It remains to show that $\vN$ has rank equal to $k:= \mathrm{dim}(\calN) = \mathrm{dim}(\calM)$. 
    Similarly to \textbf{Step 3}, we will show that $k \leq \mathrm{rank} (\vN) \leq k$ and 
    we proceed with two points:
    \begin{itemize}[leftmargin=2.5em]
        \item[(\textbf{I})] We show that by the form of $\vN$, we obtain $\mathrm{rank}(\vN) \leq k$;

        \item[(\textbf{II})]  We use Equation~\eqref{eq:are-we-done} to obtain that $\mathrm{rank}(\vN) \geq k$.
    \end{itemize}

    (\textbf{I})
    Notice that $\vN$ is left-invariant by multiplication to $\vP_\calN$, because of the term
    $(\vP_\calF \vP_\calG)^+ = \vP_\calN(\vP_\calF \vP_\calG)^+$, by \cref{lemma:projectors} (iii). Therefore:
    \begin{align}    
        \vN &= \vP_\calN \vN  \\
        \mathrm{rank}(\vN) &= \mathrm{rank}(\vP_\calN \vN) 
        \\
        \mathrm{rank}(\vN) &\leq \min \big(\mathrm{rank}(\vP_\calN), \mathrm{rank}(\vN) \big) 
        \tag{Using $\mathrm{rank} (\vA \vB) \leq \min \big( \mathrm{rank}(\vA), \mathrm{rank}(\vB) \big)$}
        \\
        \mathrm{rank}(\vN) &\leq \min \big(k, \mathrm{rank}(\vN) \big) \\
        \implies \mathrm{rank}(\vN) &\leq k
        \label{eq:N-leq-k}
    \end{align}

    (\textbf{II})
    Next, consider $q$ elements of $\calA$ such that the matrices $\vG$ and $\tilde{\vG}$ in Equation~\eqref{eq:g-gtilse-spans} have rank equal to $\mathrm{dim}(G)$ and $\mathrm{dim}(\tilde{\calG})$, respectively. 
    From Equation~\eqref{eq:are-we-done},
    it holds:
    \[  \label{eq:thm5-last-we-are-done}
        \vP_\calN \vG = \vN \vP_{\tilde{\calN}} \tilde \vG
    \]
    Notice that $\mathrm{rank}(\vP_\calN \vG) = \mathrm{dim}(\calN)$ and $\mathrm{rank}(\vP_{\tilde \calN} \tilde{\vG}) = \mathrm{dim}(\tilde{\calN})$ by \cref{lemma:rank-equality}, 
    and it also holds $\mathrm{dim}(\calN) = \mathrm{dim}(\calM) = \mathrm{dim}(\tilde{\calM}) =\mathrm{dim}(\tilde{\calN})$ by \cref{lemma:projectors} (i) and \textbf{Step 3}. Let $k:=\mathrm{dim}(\calN)$.
    Using this in Equation~\eqref{eq:thm5-last-we-are-done},
    we obtain:
    \begin{align}
        \mathrm{rank} (\vP_\calN \vG ) &= \mathrm{rank} (\vN \vP_{\tilde{\calN}} \tilde \vG) \\
        k &\leq \mathrm{min} \big( \mathrm{rank}(\vN), \mathrm{rank}(\vP_{\tilde{\calN}} \tilde \vG) \big) \\
        k &\leq \mathrm{min} \big( \mathrm{rank}(\vN), k) \\
        \implies \mathrm{rank}(\vN) &\geq k
    \end{align}
    This shows that, combined with Equation~\eqref{eq:N-leq-k} we have $k \leq \mathrm{rank} (\vN) \leq k$, which means that $\mathrm{rank}(\vN) = k$. 
    Taking \textbf{Steps 2, 3, 4, 5} together, we have that:
    \[
        (\vf, \vg) \sim_{EL} (\tilde{\vf}, \tilde{\vg}) \, .
    \]
    This shows the implication.

    \vspace{1em}

    $(\impliedby)$ To prove the other direction show that also $(\vf, \vg) \sim_{EL} (\tilde{\vf}, \tilde{\vg}) \implies p_{\vf, \vg}(y \mid \vx) = p_{\tilde \vf, \tilde \vg}(y \mid \vx)$, for all $\vx \in \SeqA$ and all $y \in \calA$.
    We start from \cref{prop:elpreserves}, which gives:
    \[
        (\vf, \vg) \sim_{EL} (\tilde{\vf}, \tilde{\vg}) \implies \vf(\vx)^\top \vg_0(y) = \tilde{\vf}(\vx)^\top \tilde \vg_0(y)
    \]
    for all $\vx \in \SeqA$ and all $y \in \calA$.
    We continue from the right-hand side to obtain:
    \begin{align}
        \vf(\vx)^\top \vg(y) - \vf(\vx)^\top \vg(y_0) &= \tilde \vf(\vx)^\top \tilde \vg(y) - \tilde \vf(\vx)^\top \tilde \vg(y_0) 
        \tag{Use explicit expression for $\vg_0$ and $\tilde{\vg}_0$}
        \\
        \vf(\vx)^\top \vg(y)  &= \tilde \vf(\vx)^\top \tilde \vg(y) + \vf(\vx)^\top \vg(y_0) - \tilde \vf(\vx)^\top \tilde \vg(y_0) 
        \tag{Reordering all $y_0$ terms on the right}
        \\
        \exp (\vf(\vx)^\top \vg(y))  &= \exp(\tilde \vf(\vx)^\top \tilde \vg(y)) \cdot \exp\big( \vf(\vx)^\top \vg(y_0) - \tilde \vf(\vx)^\top \tilde \vg(y_0)\big)
        \tag{Taking the exponential on both sides}
        \\
        \frac{\exp (\vf(\vx)^\top \vg(y))}{Z(\vx)}  &= {\exp(\tilde \vf(\vx)^\top \tilde \vg(y))} \cdot \frac{\exp\big( \vf(\vx)^\top \vg(y_0) - \tilde \vf(\vx)^\top \tilde \vg(y_0)\big)}{Z (\vx)}
        \tag{Dividing by the normalizing constant $Z(\vx)$}
        \\
        p_{\vf, \vg}(y \mid \vx)  &= {\exp(\tilde \vf(\vx)^\top \tilde \vg(y))} \cdot \frac{1}{\tilde Z (\vx)}  \, ,
        \label{eq:p-exp-ztilde}
    \end{align}
    where in the last line we
    included the expression for the conditional probability
    $p_{\vf, \vg}(y \mid \vx) = \exp(\vf(\vx)^\top \vg(y)) / Z(\vx)$ from Equation~\eqref{eq:next-token-predictor},
    and we denoted $\tilde{Z}(\vx) := {Z}(\vx) \big/ {\exp\big( \vf(\vx)^\top \vg(y_0) - \tilde \vf(\vx)^\top \tilde \vg(y_0)\big)}$. To obtain the value of $\tilde Z$ we consider the sum over all $y \in \calA$ for Equation~\eqref{eq:p-exp-ztilde}, giving:
    \begin{align}
        \sum_{y \in \calA} p_{\vf, \vg}(y \mid \vx)  &=  \sum_{y \in \calA} {\exp(\tilde \vf(\vx)^\top \tilde \vg(y))} \cdot \frac{1}{\tilde Z(\vx)} \\
        1 &= \sum_{y \in \calA} {\exp(\tilde \vf(\vx)^\top \tilde \vg(y))} \cdot \frac{1}{\tilde Z(\vx)} \\
        \tilde{Z}(\vx) &= \sum_{y \in \calA} {\exp(\tilde \vf(\vx)^\top \tilde \vg(y))}
    \end{align}
    which means that, from Equation~\eqref{eq:p-exp-ztilde} we have:
    \begin{align}
        p_{\vf, \vg}(y \mid \vx)  &=  \frac{ \exp(\tilde \vf(\vx)^\top \tilde \vg(y))
        }{
        \sum_{y \in \calA} \exp(\tilde \vf(\vx)^\top \tilde \vg(y))} \\
        &= p_{\tilde{\vf}, \tilde{\vg}}(y \mid \vx)
    \end{align}
    showing the claim. 
    This concludes the proof.
\end{proof}

    \subsection{Proof of \cref{cor:linear-id}}
    
        The following corollary constitutes a special case of \cref{thm:partial-identifiability}, which can be easily proven by setting $d= \tilde{d}$ and requiring that $\calM = \calN = \bbR^d$.%
        Here, we provide an alternative proof expanding previous results by \citet{roeder2021linear}, relaxing two assumptions that were used in that context. For comparison, we report the statement by \citet{roeder2021linear}. To this end, fix a pivot $\vx_0 \in \SeqA$ and indicate with
        \[
            \vf_0(\vx) := \vf(\vx) - \vf(\vx_0)
        \]
        the difference between embeddings and the pivot. 
        
        \begin{theorem*}[\citep{roeder2021linear}]
            Given two models $(\vf, \vg), (\vf, \vg) \in \Theta$,
            under the assumption that:
            \begin{enumerate}
                
                \item  $\SIM{\vf_0} = \SIM{\vg_0} = \bbR^d$; 
        
                \item  $\SIM{\tilde \vf_0} = \SIM{\tilde \vg_0} = \bbR^d$; 
                
            \end{enumerate}
            it holds:
            \[
                p_{\vf, \vg} = p_{\tilde{\vf}, \tilde{\vg}} 
                \implies
                (\vf, \vg) \sim_L (\tilde{\vf}, \tilde{\vg})
            \]
            where the linear equivalence relation is given by:
            \[
                (\vf, \vg) \sim_L (\tilde{\vf}, \tilde{\vg}) 
                \iff 
                \begin{cases}
                    \vf(\vx) &= \vM \tilde{\vf} (\vx)\\
                    \vg_0(y) &= \vN \tilde{\vg}_0 (y) 
                \end{cases}
            \]
            $\forall \vx \in \SeqA$ and $\forall y \in \calA$,
            where $\vM^\top \vN = \vI$ and in particular $\vN = \vM^{-\top}$.
        \end{theorem*}

        To highlight deviations, we present a proof that follows a somewhat analogous argument to the proof of \citep{roeder2021linear}; a direct proof may show $\vN=\vM^{-1}$ relying on \cref{thm:partial-identifiability}.
        We relax condition \textit{2} and use the fact the assumption that $\SIM{\vf} = \bbR^d$, which is a milder condition to requiring that $\SIM{\vf_0} = \bbR^d$. We prove the following:
        
        \roedeo*

        \begin{proof}
Our proof follows a similar technique to \citet[Theorem B.4]{lachapelle2023synergies}. 

\textbf{Identifiability of $\vf$}.
First notice that from the equivalnce of log-likelihood we can write:
\begin{align}
    \log p_{\vf, \vg} (y \mid \vx) &= \log p_{\tilde \vf, \tilde \vg} (y \mid \vx)  \\
    \vg(y)^\top \vf(\vx) - \log Z(\vx) &= \tilde{\vg}(y)^\top \tilde{\vf}(\vx) - \log \tilde Z(\vx) \\
    (\vg(y) - \vg(y_0))^\top \vf(\vx)  &= (\tilde{\vg}(y) -\tilde{\vg}(y_0))^\top \tilde{\vf}(\vx) \\
    \vg_0(y)^\top \vf(\vx)  &= \tilde{\vg}_0(y)^\top \tilde{\vf}(\vx) \label{eq:pivot-difference-roeder}
\end{align}

where in the last line we subtracted the pivot for different log-probabilities on the points $y$ and $y_0$. 
We now consider the matrix $\vG$ constructed to contain $d$ differences:
\[
    \vG = \begin{pmatrix}
        \vg_0(y_1), & \cdots &
        \vg_0(y_d)
    \end{pmatrix}
\]
such that it is invertible. Since by the diversity condition $\SIM{\vg_0} = \bbR^d$, such a marix always exists. 
Let $\tilde{\vG}$ be the corresponding matrix of differences for $\tilde{\vg}$, we obtain:
\[
    \vG^\top \vf (\vx) = \tilde{\vG}^\top \tilde{\vf}(\vx) 
\]
Then, we obtain:
\begin{align}
    \vf (\vx) &= \HL{\vG^{-\top}} \tilde{\vG}^\top \tilde{\vf}(\vx) \\
    \vf (\vx) &= \HL{\vM} \tilde{\vf}(\vx) 
\end{align}
where we denoted as $\vM =\vG^{-\top} \tilde{\vG}^\top$. 
Next, since $\SIM{\vf} = \bbR^d$ we can consider $d$ elements $\vx_i \in \SeqA$ such that
\[
    \vF = \begin{pmatrix}
        \vf(x_1), & \cdots, & \vf(\vx_d)
    \end{pmatrix}
\]
is invertible. Let $\tilde{\vF}$ be the corresponding matrix for $\tilde{\vf}$. 
In this way we obtain the following:
\[
    \vF = \vM \tilde{\vF} \, ,
\]
and since $\vF$ is invertible, it must be that also $\vM$ and $\vF$ are invertible matrices of rank $d$. This shows that
\[
    \vf (\vx) = \vM \tilde{\vf} (\vx)
    \label{eq:first-step-roeder}
\]
$\vM \in \bbR^{d \times d}$ is invertible. 

\textbf{Identifiability of $\vg_0$}.
Next, we consider the implication for $\vg_0$. We start again from the pivot difference of Equation~Equation~\eqref{eq:pivot-difference-roeder}:
\begin{align}
    \vg_0(y)^\top \vf(\vx) &= \tilde{\vg}_0(y)^\top \tilde{\vf}(\vx)  \\
    \vg_0(y)^\top \vf(\vx) &= \tilde{\vg}_0(y)^\top \vM^{-1} {\vf}(\vx)
    \label{eq:roeder-step-3}
\end{align}
where we substituted
$\tilde{\vf} (\vx) = \vM^{-1} \vf (\vx)$ from Equation~\eqref{eq:thm5-first-step}.
Therefore, taking $d$ points $\vx \in \SeqA$, such that the matrix $\vF$ is invertible, restarting from the transpose of Equation~\eqref{eq:roeder-step-3} we obtain:
\begin{align}
    \vF^\top \vg_0(y) &= \vF^\top \vM^{-\top} \tilde{\vg}_0(y) \\
    \vg_0(y) &= \HL{\vM^{-\top}} \tilde{\vg}_0(y) 
    \tag{Multiplying for the inverse of $\vF^\top$}
    \\
    \vg_0(y) &= \vM^{-\top} \tilde{\vg}_0(y) \, ,
\end{align}

This means that we have:
\begin{align}
    \vf(\vx)_{\phantom{0}} &= \vM \tilde{\vf} (\vx)_{\phantom{0}} \\
    \vg_{0}(y) &= \vN \tilde{\vg}_0 (y) \, ,
\end{align}
where we have defined $\vN := \vM^{-\top}$, such that $\vM^\top \vN = \vI$, proving the claim.
\end{proof}

    \newpage

\section{Additional Results and Proofs of \cref{sec:linearity}}
\label{sec:proofs-sec3}

    \subsection{{Relational} Linear Steering Property }%
    \label{sec:app-additivity}
    
        We here want to discuss an additional linear property
        besides those presented in \cref{sec:linearity}, termed {\em linear steering property}. 
        This behavior is also referred to as the  \textit{linear intervening property} by \citet{park2023linear}.
        
        It has been empirically observed that there exist steering vectors that influence next-token predictions \citep{park2023linear, hernandez2023linearity, hase2024does, arditi2024refusal}, in the following sense: %
        If $\vv$ encode the average difference between English to Italian embeddings, %
        adding $\vv$ to $\vf(\vs)$ for the sentence $\vs=$``\textit{The king sits on the}'' would change the most-likely next token prediction $y=$``\textit{throne}'' to $\hat y=$``\textit{trono}'', and similarly this applies for other sentences, affecting the most-likely next-token prediction to move from the English token to the Italian counterpart.

        We define this property as follows:

        \begin{definition}[Relational Linear Steering]%
            We say that a model $(\vf, \vg) \in \Theta$ possess linear relational steering for $\vq_0$ and the set of $\{\vq_1, \ldots, \vq_m \}$, for $m \geq 1$ queries $\vq_j \neq \vq_0$, if (1) it linearly represents $\vq_0$ in $\Gamma_0$
            and all $\vq_j$ on $\Gamma_j$, and (2) there exists a vector $\vv \in \bbR^d$ such that:
            \[
                \vP_{\Gamma_0} \vA_{\vq_0} \vv  \neq \mathbf{0}, \quad 
                 \vP_{\Gamma_j} \vA_{\vq_j} \vv = \mathbf{0}, \quad \forall j \in [m]
            \]
        \end{definition}
        We prove that relational linearity allows for this property: %

        \begin{restatable}%
        {proposition}{connectionli}
            \label{prop:linear-additivity}
            If (1) $(\vf, \vg)$ linearly represents $\vq_0$ on $\Gamma_0$ %
            (\cref{def:linearity}) and 
            (2) $(\vf, \vg)$ linearly represents $m \geq 1$ queries $\vq_j \neq \vq_0$ on $\Gamma_j $, %
            such that %
            $\big( \bigcup_j  \Gamma_{\vq_j} \big) \cap \Gamma_{\vq_0} \subsetneq \Gamma_{\vq_0}$, 
            then 
            the model $(\vf, \vg)$ satisfies linear relational steering for $\vq_0$ and the set of queries $\{\vq_1 ,\ldots, \vq_m\}$.
        \end{restatable}
        
        \begin{proof}
    From the assumptions (1) and (2) we have that relational linearity as per \cref{def:linearity} implies that:
    \[
    \begin{cases}
        \vP_{\Gamma_0} \vf(\vs \cat \vq_0) &= \vP_{\Gamma_0} \vA_{\vq_0} \vf (\vs) + \vP_{\Gamma_0} \va_{\vq_0} \\
        \vP_{\Gamma_j} \vf(\vs \cat \vq_j) &= \vP_{\Gamma_j} \vA_{\vq_j} \vf (\vs) + \vP_{\Gamma_j} \va_{\vq_j}, 
        \quad \forall j \in [\ell] 
    \end{cases}
    \]
    Let $\Gamma_{0,int} = \bigcup_{j=1}^\ell  \Gamma_{\vq_j} \cap \Gamma_{\vq_0}$.
    By assumption (2), $\Gamma_{0, int} \subsetneq \Gamma_0$, meaning that it exists a non-empty $\Gamma_{0, \perp} = \Gamma_{\vq_0} \setminus \Gamma_{0, int}$. 
    Let $\vv \in \Gamma_{0, \perp}$. It holds:
    \[
        \vP_{\Gamma_0} \vA_{\vq_0} \vv = \vv
    \]
    because $\Gamma_{0, \perp} \subseteq \Gamma_{\vq_0}$. Proceeding similarly, it holds $ \Gamma_{0, \perp} \cap \Gamma_{\vq_j} = \emptyset$ $\forall j \in [\ell]$, by assumption (2). Therefore, we have:
    \[
        \vP_{\Gamma_j} \vA_{\vq_j} \vv = \mathbf{0}
    \]
    showing the claim.
\end{proof}

        This means that adding $\vv$ to $\vf (\vs)$  would alter only the value of $\vP_{\Gamma_0} \vf(\vs \cat \vq_0)$ without changing that of  $\vP_{\Gamma_j} \vf(\vs \cat \vq_j)$, for $j \in [m]$. 
        As a result, the modified representation $\vf (\vs) + \vv$ would not affect the next-token prediction on other queries $\vq_j$. For example, take $\vq_0=$``\textit{Is the previous sentence written in English?}'', $\vq_1=$``\textit{Is the previous sentence written in Italian?}'', 
        and $\vq_2=\textit{``Does the previous sentence contain the symbol ``+''?''}$,
        and consider $\Gamma_0 = \Gamma_1 = \Gamma_2 = \mathrm{span}( \vg(\textit{``no''}) - \vg(\textit{``yes''}))$. 
        When assumptions of \cref{prop:linear-additivity} hold, we can alter the reply to the question $\vq_0$ by adding a vector proportional to $\vv$, \eg moving from English to another language, without affecting the representation on $\vq_1$, \ie moving to another language but not Italian, and the representation on $\vq_2$, \ie leaving the symbol \textit{``+''} in the sentence if present.

    \subsection{Proof of \cref{lemma:parallelism}}

        \lemparallelism*

\begin{proof}
    We start by considering the equality between log-ratios appearing as the ($\impliedby$) condition. 
    Writing it down we obtain:
    \begin{align}
        \log \frac{p_{\vf, \vg}(y_0 \mid \vs)}{p_{\vf, \vg}(y_1 \mid \vs)} &= \beta \cdot \log \frac{p_{\vf, \vg}(y_2 \mid \vs)}{p_{\vf, \vg}(y_3 \mid \vs)} \\
        \log \frac{\exp(\vf(\vs)^\top \vg(y_0))}{\exp(\vf(\vs)^\top \vg(y_1))} &= \beta \cdot \log \frac{\exp(\vf(\vs)^\top \vg(y_2))}{\exp(\vf(\vs)^\top \vg(y_3))} \\
        \log {\exp\big(\vf(\vs)^\top (\vg(y_0) - \vg(y_1)) \big)} &= \beta \cdot \log {\exp\big(\vf(\vs)^\top (\vg(y_2) - \vg(y_3)) \big)} \\
        \vf(\vs)^\top (\vg(y_0) - \vg(y_1)) &= \beta \cdot \vf(\vs)^\top (\vg(y_2) - \vg(y_3))
    \end{align}
    And substituting $\vg_1(y_0) := \vg(y_0) - \vg(y_1)$ and $\vg_3(y_2) := \vg(y_2) - \vg(y_3)$ we obtain:
    \[
        \vf(\vs)^\top \vg_1(y_0) = \beta \cdot \vf(\vs)^\top \vg_3(y_2)
    \]
    Consider $\ell$ elements $\vs \in \SeqA$, such that:
    \[
        \vF = (\vf(\vs_1), \ldots, \vf(\vs_\ell))
    \]
    spans $\calF := \SIM{\vf}$. We then obtain:
    \[
        \label{eq:F-T-g}
        \vF^\top \vg_1(y_0) = \beta \vF^\top \vg_3(y_2)
    \]
    and multiplying both sides of Equation~\eqref{eq:F-T-g} from the left with the pseudo-inverse of $\vF^\top$
    we get:
    \[  
        \label{eq:P-g}
        \vP_\calF \vg_1(y_0) = \beta \vP_\calF \vg_3(y_2).
    \]
    Notice that, both $\vg_1(y_0), \vg_3(y_2) \in \calG := \SIM{\vg_0}$, then it holds $\vg_1(y_0) = \vP_\calG \vg_1(y_0)$
    $\vg_3(y_2) = \vP_\calG \vg_3(y_2)$. 
    Using this we obtain:
    \begin{align}
        \vP_\calF \vP_\calG \vg_1(y_0) &= \beta \vP_\calF \vP_\calG \vg_3(y_2)
        \\
        (\vP_\calF \vP_\calG)^+ \vP_\calF \vP_\calG \vg_1(y_0) &= \beta (\vP_\calF \vP_\calG)^+ \vP_\calF \vP_\calG \vg_3(y_2) \\
        \vP_\calN \vg_1(y_0) &= \beta \vP_\calN \vg_3(y_2) \, .
    \end{align}
    Regrouping the two terms on one side we can see that to have parallelism in $\calN$ (\cref{def:s-parallelism}), we must have the following:
    \[
        \vP_\calN (\vg_1(y_0) - \beta \cdot \vg_3(y_2)) = \mathbf{0}
    \]
    \ie $\vg_1(y_0) - \beta \cdot \vg_3(y_2) \in \calN^\perp$. Therefore, $\vg_1(y_0)$ and $\vg_3(y_2)$ are parallel in $\calN$. 

    The implication ($\implies$) is given by a similar proof by starting from \cref{def:s-parallelism}, \ie that $\vP_\calN \vg_1(y_0) = \beta \vP_\calN \vg_3(y_2)$. 
    Therefore by multiplying the two for any embedding $\vf(\vs)$ we get
    \begin{align}
     \vg_1(y_0)^\top \vP_\calN \vf(\vs) &= \beta  \vg_3(y_2)^\top \vP_\calN \vf(\vs) \\
     \log \frac{p_{\vf, \vg}(y_0 \mid \vs)}{p_{\vf, \vg}(y_1 \mid \vs)} &= \beta \cdot \log \frac{p_{\vf, \vg}(y_2 \mid \vs)}{p_{\vf, \vg}(y_3 \mid \vs)} \, .
    \end{align}
    This shows the claim. 
\end{proof}

\begin{remark}
\label{remark:parallelism}
    We exclude the case $\beta = 0$ because, for $\va = 0$, a contradiction arises. Specifically, any vector $\vb \in \bbR^d$ would be trivially parallel to $\va$ (i.e., $\va = \beta \vb$ with $\beta = 0$). However, conversely, we would obtain that no scalar $\beta \in \bbR$ exists such that $\vb = \beta \va$.
\end{remark}

    \subsection{Proof of \cref{prop:connection-lin-sub}}

        \connectionls*
        \begin{proof}
    We start from the relational linearity as per \cref{def:linearity} for $\vq$ on $\Gamma \subset \bbR^d$:
    \[
        \vP_\calG \vf(\vs \cat \vq) = \vP_\calG \vA_\vq \vf (\vs) + \vP_\calG \va_\vq \, ,
    \]
    and recall that $\Gamma_\vq := \mathrm{Im}(\vA_\vq^\top \vP_\Gamma) =\{ \vA^\top_\vq \vv \mid \vv \in \Gamma \} $. 
    By assumption (2) $\Gamma_\vq \subseteq \SIM{\vg_0}$. Then, for any pair $y_i , y_j \in \calA$  such that
    $\vg_{i}(y) := \vg(y_j) - \vg(y_i) \in \Gamma_\vq$, 
    we can find a vector $\vgamma \in \Gamma$ such that
    \[
        \vA_\vq^\top \vgamma = \vg_{i}(y_j) \, .
    \]
    Notice that from this expression we can also write:
    \[
         \vA_\vq^\top \vP_\Gamma \vgamma = \vg_{i}(y_j) \, ,
    \]
    since $\vP_\Gamma \vgamma = \vgamma$. Hence,
    the vector $\vgamma$ can be obtained taking the pseudo-inverse of $\vA_\vq^\top \vP_\Gamma$:
    \begin{align}
        (\vA_\vq^\top \vP_\Gamma)^+ \vA_\vq^\top \vP_\Gamma \vgamma &= (\vA_\vq^\top \vP_\Gamma)^+  \vg_{i}(y_j) \\
        \vP_{\mathrm{Im} (\vP_\Gamma \vA_\vq )} \vgamma &= (\vA_\vq^\top \vP_\Gamma)^+  \vg_{i}(y_j)
    \end{align}
    Notice that, since $(\vA_\vq^\top \vP_\Gamma)^+ \vP_{\Gamma_\vq} = (\vA_\vq^\top \vP_\Gamma)^+$ and 
    $\vP_{\Gamma_\vq} \vg_i(y_j) =\vg_i(y_j)$, 
    we have that $(\vA_\vq^\top \vP_\Gamma)^+ \vg_i (y_j) = \mathbf{0}$ only when $\vg_i (y_j) = \mathbf{0}$. 
    As a consequence, when $\vg_i(y_j) \neq \mathbf{0}$, we have that also $\vP_{\mathrm{Im}(\vP_\Gamma \vA_\vq)} \vgamma \neq \mathbf{0}$.
    Fix this $\vgamma$ and consider:
    \begin{align}
        \vgamma^\top \vA_\vq \vf (\vs)
         &=  \vgamma^\top \big(  \vf(\vs \cat \vq) - \va_\vq \big) \\
         (\vA_\vq^\top \vgamma)^\top \vf (\vs)
         &=  \vgamma^\top \big(  \vf(\vs \cat \vq) - \va_\vq \big) \\
         \vg_{i}(y_j)^\top \vf (\vs)
         &=  \vgamma^\top \big(  \vf(\vs \cat \vq) - \va_\vq \big)  \, ,
    \end{align} 
    which shows that $(\vf, \vg)$ linearly represents $\Gamma_\vq$ related to $\vq$, showing the claim.
\end{proof}

    \subsection{Proof of \cref{prop:connection-probing}}

        \connectionlp*
        \begin{proof}
    Under the assumption (2), take a pivot $y_j \in \calY_P$, then for all remaining $y_i \in \calY_P$ denote with ${\vg_{j}(y_i) := \vg(y_i) -\vg(y_j) \in \Gamma}$. Taking the log-ratios between the conditional probabilities  
    \[
        p_{\vf, \vg} (y_i \mid \vs \cat \vq; \calY_P), \; \text{ and } p_{\vf, \vg} (y_j \mid \vs \cat \vq; \calY_P) \,  ,
    \]
    for conditional probabilities restricted to $\calY_P$, as in \cref{def:linear-probing},
    we obtain:
    \begin{align}
        \log \frac{p_{\vf, \vg} (y_i \mid \vs \cat \vq; \calY_P) }{p_{\vf, \vg} (y_j \mid \vs \cat \vq; \calY_P) } &= 
        \log \frac{ \exp \big(  \vg(y_i)^\top \vf( \vs \cat \vq) \big) }{ \exp \big(  \vg(y_j)^\top \vf( \vs \cat \vq) \big) } \label{eq:log-ratio}\\
        &= \log \exp \big(  (\vg(y_i) - \vg(y_j))^\top \vf( \vs \cat \vq) \big)  \\
        &= \big(  (\vg(y_i) - \vg(y_j))^\top \vf( \vs \cat \vq) \big)  \\
        &= \vg_{j}(y_i)^\top \vf( \vs \cat \vq ) \, . \label{eq:partial-result-logratio}
    \end{align}
     Due to relational linearity of $\vq$ onto $\Gamma$ (Assumption (1)), we can write following (see \cref{def:linearity}):
    \[
        \label{eq:consequence-def-linearity}
        \vP_\Gamma \vf(\vs \cat \vq) = \vP_\Gamma \vA_\vq \vf (\vs) + \vP_\Gamma \va_\vq \, ,
    \]
    {we can then take any $\vg_j(y_i) \in \Gamma$, and multiply their transpose times both sides of Equation~\eqref{eq:consequence-def-linearity} from the right. We then get}
    \begin{align}
        \vg_j(y_i)^\top \vP_\Gamma \vf(\vs \cat \vq) &= \vg_j(y_i)^\top \vP_\Gamma \vA_\vq \vf (\vs) + \vg_j(y_i)^\top \vP_\Gamma \va_\vq \\
        (\vP_\Gamma \vg_j(y_i))^\top  \vf(\vs \cat \vq) &= (\vP_\Gamma \vg_j(y_i))^\top  \vA_\vq \vf (\vs) + (\vP_\Gamma \vg_j(y_i))^\top \va_\vq  \\
        \vg_j(y_i)^\top \vf(\vs \cat \vq) &= \vg_j(y_i)^\top  \vA_\vq \vf (\vs) + \vg_j(y_i)^\top \va_\vq \, .
        \label{eq:relational-linearity-probing}
    \end{align}
    We can then substitute the expression on the RHS in Equation~\eqref{eq:relational-linearity-probing} to Equation~\eqref{eq:partial-result-logratio} to obtain:
    \[
        \log \frac{p_{\vf, \vg} (y_i \mid \vs \cat \vq; \calY_P) }{p_{\vf, \vg} (y_j \mid \vs \cat \vq; \calY_P) } = \vg_j(y_i)^\top  \vA_\vq \vf (\vs) + \vg_j(y_i)^\top \va_\vq \, .
    \]
    Now take the conditional probability $p_{\vf, \vg}(y_i \mid \vs \cat \vq; \calY_P)$ which can be written as
    \begin{align}
        p_{\vf, \vg}(y_i \mid \vs \cat \vq; \calY_P) &= \frac{e^{\vg(y_i)^\top \vf(\vs \cat \vq)}}{Z(\vs \cat \vq; \calY_P)} \\
        &= \frac{e^{\vg(y_i)^\top \vf(\vs \cat \vq)}}{Z(\vs \cat \vq; \calY_P)} \frac{e^{\vg(y_j)^\top\vf(\vs \cat \vq)}}{e^{\vg(y_j)^\top\vf(\vs \cat \vq)}}
        \tag{Multiply and divide by the same term}
        \\
        &= \frac{e^{(\vg(y_i)- \vg(y_j))^\top \vf(\vs \cat \vq)}}{Z(\vs \cat \vq; \calY_P)} e^{\vg(y_j)^\top\vf(\vs \cat \vq)} 
        \tag{Rearrange terms in the exponential}
        \\
        &= \frac{e^{\vg_j(y_i)^\top \vf(\vs \cat \vq)}}{Z(\vs \cat \vq; \calY_P)} e^{\vg(y_j)^\top\vf(\vs \cat \vq)} 
        \tag{Substitute $\vg_j(y_i)$}
        \\
        &= \exp \big( {\vg_j(y_i)^\top \vA_\vq \vf(\vs) +  \vg_j(y_i)^\top \va_\vq } \big) \,
        \frac{e^{(\vg(y_j)^\top\vf(\vs \cat \vq)}}{Z(\vs \cat \vq; \calY_P)}  
        \tag{Use Equation~\eqref{eq:relational-linearity-probing}}
        \\
        &= \exp \big( {\vg(y_i)^\top (\vA_\vq \vf(\vs) + \va_\vq) } \big) \,
        \exp \big(- {\vg(y_j)^\top (\vA_\vq \vf(\vs) + \va_\vq) } \big)
        \,
        \frac{e^{(\vg(y_j)\top\vf(\vs \cat \vq)}}{Z(\vs \cat \vq; \calY_P)} 
        \tag{Separate the term depending on $y_i$ to those that do not}
        \\
        &= \exp \big( {\vg(y_i)^\top (\vA_\vq \vf(\vs) + \va_\vq) \big) } 
        \,
        \frac{e^{ \vg(y_j)^\top \big(  \vf(\vs \cat \vq)- \vA_\vq \vf(\vs) + \va_\vq \big)}}{Z(\vs \cat \vq; \calY_P)} 
        \tag{Rearrange exponential on the right}
        \\
        &= \exp \big( {\vg(y_i)^\top (\vA_\vq \vf(\vs) + \va_\vq) \big) } C \, ,
        \label{eq:lin-prob-step1}
    \end{align}
    where we denoted with $C$ the scaling factor applied to the exponential, which we will treat as a constant since it does not depend on $y_i \in \calA$.
    From this expression, take the sum on $\calY_P$ to obtain that:
    \begin{align}
        \sum_{y_i \in \calY_P} \exp \big( {\vg(y_i)^\top (\vA_\vq \vf(\vs) + \va_\vq) \big) } \, C &= \sum_{y_i \in \calY_P} p_{\vf, \vg}(y_i \mid \vs \cat \vq; \calY_P) \\
        \sum_{y_i \in \calY_P} \exp \big( {\vg(y_i)^\top (\vA_\vq \vf(\vs) + \va_\vq) \big) } \, C &= 1
        \tag{The sum on the right equals to $1$}
        \\
        C &= 1 \big/ \sum_{y_i \in \calY_P} \exp \big( {\vg(y_i)^\top (\vA_\vq \vf(\vs) + \va_\vq) \big) 
        \label{eq:lin-probing-step-2}
        }
    \end{align}
    Denote with 
    $\vw{_i} := \vA_\vq^\top \vg(y_i)$ and with $b_i := \vg(y_i)^\top \va_\vq$. Then using this and Equation~\eqref{eq:lin-probing-step-2} inside Equation~\eqref{eq:lin-prob-step1} we get
    \begin{align}
        p_{\vf, \vg}(y_i \mid \vs \cat \vq; \calY_P) 
        &= \exp \big( {\vg(y_i)^\top (\vA_\vq \vf(\vs) + \va_\vq) \big) } C \\
        &= \frac{\exp \big( \vw_i^\top \vf(\vs) + b_i\big) }{\sum_{y_i \in \calY_P} \exp \big( \vw_i \vf(\vs) + b_i \big) } 
        \tag{Substitute for $\vw_i$ and $b_i$, and Equation~\eqref{eq:lin-probing-step-2}}
        \\
        &= \mathrm{softmax} \big(
        \vW \vf(\vs) + \vb \big)_i \, ,
    \end{align}
    where we defined $\vW := (\vw_1, \ldots, \vw_\ell)^\top$ and $\vb := (b_1, \ldots, b_\ell)^\top$.
    This shows that the model $(\vf, \vg)$ can be linear probed for $\vq$ in $\calY_P$ with  $\vW$ and $\vb$.
\end{proof}

\newpage

\section{Proof of \cref{sec:implications}} \label{sec:proof-sec5}

    \subsection{Proof of \cref{prop:part-id-lin-rep-tentative}}
    
        \pidlinrep*
        \textbf{Proof Sketch}. The proof is divided in two steps:
\begin{enumerate}
    \item We first prove the implication that $(\tilde{\vf}, \tilde{\vg})$ linearly represents $\vq$ on a subset $\tilde \Gamma \subseteq \bbR^d$; 

    \item Then we show that $\tilde \Gamma \subseteq \calN$ and that $\tilde \Gamma_\vq := \mathrm{Im}(\tilde{\vA}^\top_\vq \vP_\Gamma) \subseteq  \tilde \calN$.
\end{enumerate}

\begin{proof}
    \textbf{Step 1}.
    We begin from the relational linearity definition for model $(\vf, \vg)$. It holds
    \[
        \vP_\Gamma \vf(\vs \cat \vq ) = \vP_\Gamma \vA_\vq \vf(\vs) + \vP_\Gamma \va_\vq \, ,
    \]
    where $\Gamma \subseteq \calN$. This also means that $\vP_\Gamma \vP_\calN = \vP_\Gamma$.
    Denote with $\calF:= \SIM{\vf}$ and with $\calG:= \SIM{\vg_0}$.
    By assumption, it holds that $
        \Gamma_\vq := \mathrm{Im} (\vA^\top_\vq \vP_\Gamma) = \{\vA^\top_\vq \vv \mid \vv \in \Gamma \}
    $ is a subset of $\calM$. 
    This implies, in turn, that $\vP_\calM \vA^\top_\vq \vP_\Gamma=\vA^\top_\vq \vP_\Gamma$. We use this to write:
    \begin{align}
        \vP_\Gamma \vP_\calN \vf(\vs \cat \vq ) &= \vP_\Gamma \vA_\vq \vf(\vs) + \vP_\Gamma \va_\vq 
        \tag{Using $\vP_\Gamma =\vP_\Gamma \vP_\calN$}
        \\
        \vP_\Gamma \vP_\calN \vP_\calF \vf(\vs \cat \vq ) &= \vP_\Gamma \vA_\vq \vP_\calM \vf(\vs) + \vP_\Gamma \va_\vq 
        \tag{Using $\vf = \vP_\calF \vf$ on the left and $\vA_\vq = \vA_\vq \vP_\calM$ on the right}
        \\
        \vP_\Gamma \vP_\calN \vP_\calM \vf(\vs \cat \vq ) &= \vP_\Gamma \vA_\vq \vP_\calM \vf(\vs) + \vP_\Gamma \va_\vq \, ,
        \label{eq:prop14-step1}
    \end{align}
    where in the last line we used that  $\vP_\calN \vP_\calF = \vP_\calN \vP_\calG \vP_\calF =  \vP_\calN \vP_\calN \vP_\calM =\vP_\calN \vP_\calM$ by \cref{lemma:projectors} (vi).
    We now substitute the expression for $(\tilde \vf, \tilde \vg)$ based on the RHS of the equivalence relation Equation~\eqref{eq:linear-transf-f} in  Equation~\eqref{eq:prop14-step1} to get
    \[
        \vP_\Gamma \vP_\calN \vM  \vP_{\tilde \calM} \tilde{\vf}(\vs \cat \vq ) = \vP_\Gamma \vA_\vq \vP_\calN \vM \vP_{\tilde \calM} \tilde{\vf}(\vs) + \vP_\Gamma \va_\vq  \, .
        \label{eq:whatwework}
    \]
    Now starting from $\vN^\top \vM = \vP_{\tilde \calN} \vP_{\tilde \calM}$, as specified in %
    \cref{def:klinear-equiv}, we can apply the following steps:
    \begin{align}
        \vN^\top \vM &= \vP_{\tilde \calN} \vP_{\tilde \calM} \\
        (\vN^\top)^+ \vN^\top \vM &=  (\vN^\top)^+ \vP_{\tilde \calN} \vP_{\tilde \calM} 
        \tag{Multiply on the left by the pseudo-inverse $(\vN^\top)^+$}
        \\
        \vP_\calN \vM &= (\vN^\top)^+ \vP_{\tilde \calN} \vP_{\tilde \calM} 
        \tag{Using $(\vN^\top)^+ \vN^\top = \vP_\calN$}
        \\
        \vP_\calN \vM &= (\vN^\top)^+ \vP_{\tilde \calN} \vP_{\tilde \calN} \vP_{\tilde{\calM}}
        \tag{Idempotency of the orthogonal projector $\vP_{\tilde \calN}$}
        \\
        \vP_\calN \vM &= (\vN^\top)^+ \vP_{\tilde \calN} \vP_{\tilde \calG} \vP_{\tilde \calF} 
        \tag{Substitute $\vP_{\tilde \calG} \vP_{\tilde \calF} = \vP_{\tilde \calN} \vP_{\tilde \calM}$ from \cref{lemma:projectors} (vi)}
        \\
        \vP_\calN \vM &= (\vN^\top)^+\vP_{\tilde \calN} \vP_{\tilde \calF} 
        \tag{Using $\vP_{\tilde \calN} \vP_{\tilde \calG} = \vP_{\tilde \calN}$}
        \\
        \vP_\calN \vM &= (\vN^\top)^+ \vP_{\tilde \calF}  
        \tag{Using $(\vN^\top)^+ \vP_{\tilde \calN} = (\vN^\top)^+$}
        \\
        \vP_\Gamma \vP_\calN \vM \vP_{\tilde{\calM}} &= \vP_\Gamma (\vN^\top)^+ \vP_{\tilde \calF}  \, .
        \label{eq:prop14-step-3}
    \end{align}
    where in the last line we multiplied on the left by the orthogonal projector $\vP_\Gamma$ and used that $\vM = \vM \vP_{\tilde{\calM}}$.
    We can now show that we can substitute  Equation~\eqref{eq:prop14-step-3} into the left-hand side of  Equation~\eqref{eq:whatwework}: %
    \begin{align}
        \vP_\Gamma (\vN^\top)^+ \vP_{\tilde \calF}  \tilde{\vf}(\vs \cat \vq ) &=\vP_\Gamma \vA_\vq \vM \vP_{\tilde{\calM}} \tilde{\vf}(\vs) + \vP_\Gamma \va_\vq \\
        \vP_\Gamma (\vN^\top)^+ \tilde{\vf}(\vs \cat \vq ) &=\vP_\Gamma \vA_\vq \vM \vP_{\tilde{\calM}} \tilde{\vf}(\vs) + \vP_\Gamma \va_\vq 
        \tag{Use $\vP_{\tilde{\calF}} \tilde \vf = \tilde \vf$}
        \\
        (\vP_\Gamma (\vN^\top)^+)^+ \vP_\Gamma (\vN^\top)^+ \tilde{\vf}(\vs \cat \vq ) &= (\vP_\Gamma (\vN^\top)^+)^+ \vP_\Gamma \vA_\vq \vM \tilde{\vf}(\vs) +  (\vP_\Gamma (\vN^\top)^+)^+ \vP_\Gamma \va_\vq
        \tag{Multiply on the left by $(\vP_\Gamma (\vN^\top)^+)^+$}
        \\
        \vP_{\tilde{\Gamma}} \tilde{\vf}(\vs \cat \vq ) &= \tilde \vA_\vq \tilde{\vf}(\vs) + \tilde \va_\vq \, ,
    \end{align}
    where we denoted with $\vP_{\tilde \Gamma}$ the orthogonal projector on $\tilde \Gamma$ and we defined:
    \begin{align}
        \vP_{\tilde{\Gamma}} &:= (\vP_\Gamma (\vN^\top)^+)^+ \vP_\Gamma (\vN^\top)^+  
        \label{eq:projector-tildegamma}
        \\
        \tilde{\vA}_\vq &:= (\vP_\Gamma (\vN^\top)^+)^+ \vP_\Gamma \vA_\vq \vM \\
        \tilde \va_\vq &:=   (\vP_\Gamma (\vN^\top)^+)^+ \vP_\Gamma \va_\vq \, .
    \end{align}
        
    Multiplying by $\vP_{\tilde{\Gamma}}$ we arrive at the same expression for linearity for the model $(\tilde{\vf}, \tilde{\vg})$:
    \[
        \vP_{\tilde{\Gamma}} \tilde{\vf}(\vs \cat \vq ) = \vP_{\tilde{\Gamma}} \tilde \vA_\vq \tilde{\vf} + \vP_{\tilde{\Gamma}} \tilde \va_\vq \, .
    \]

    \vspace{1em}
    
    \textbf{Step 2}. We proceed to show that $\tilde{\Gamma} \subseteq \tilde{\calN}$.
    First, we have to show that $\vP_{\tilde \Gamma} \vP_{\tilde \calN} = \vP_{\tilde \Gamma}$. To this end, notice that $(\vN^\top)^+: \tilde{\calN} \to \calN$. Therefore, we have from Equation~\eqref{eq:projector-tildegamma}:
    \begin{align}
        \vP_{\tilde \Gamma} &= (\vP_\Gamma (\vN^\top)^+)^+ \vP_\Gamma (\vN^\top)^+ \\
        &= (\vP_\Gamma (\vN^\top)^+)^+ \vP_\Gamma (\vN^\top)^+ \vP_{\tilde \calN}
        \tag{Using $(\vN^\top)^+ = (\vN^\top)^+ \vP_{\tilde{\calN}}$}
        \\
        &=\vP_{\tilde \Gamma} \vP_{\tilde \calN} \, .
    \end{align}
    Taking the transpose of $\vP_{\tilde \Gamma}$ we obtain:
    \begin{align}
        \vP_{\tilde \Gamma}^\top &= (\vP_{\tilde \Gamma} \vP_{\tilde \calN})^\top \\
        \vP_{\tilde \Gamma} &=   \vP_{\tilde \calN}^\top \vP_{\tilde \Gamma}^\top \\
        \vP_{\tilde \Gamma} &=   \vP_{\tilde \calN} \vP_{\tilde \Gamma} 
    \end{align}
    where we used that $\vP_{\tilde{\calN}}^\top  = \vP_{\tilde{\calN}}$ and $\vP_{\tilde{\Gamma}}^\top  = \vP_{\tilde{\Gamma}}$ because both are symmetric matrices.
    This means, in turn, that
    $\vP_{\tilde \calN}$ and $\vP_{\tilde \Gamma}$ commute, and so 
    $\tilde \Gamma $ must be a contained in $ \tilde{\calN}$. 
    Similarly from the expression of $\tilde{\vA}_\vq$ we get:
    \begin{align}
         \tilde{\vA}_\vq &= (\vP_\Gamma (\vN^\top)^+)^+ \vP_\Gamma \vA_\vq \vM \\
         &= (\vP_\Gamma (\vN^\top)^+)^+ \vP_\Gamma \vA_\vq \vM \vP_{\tilde \calM} \\
         &= \tilde{\vA}_\vq \vP_{\tilde \calM} \, .
    \end{align}
    This means that 
    $
        \tilde \vA_\vq^\top  \vP_{\tilde \Gamma} =
        \vP_{\tilde \calM} \tilde \vA_\vq^\top  \vP_{\tilde \Gamma} 
    $,
    and so
    $\Gamma_\vq := \mathrm{Im}(\tilde{\vA}_\vq^\top \vP_{\tilde \Gamma}) \subseteq \tilde{\calM}$.
    To find the expression for $\tilde \Gamma$, we consider the following:
    \begin{align}
        \tilde{\Gamma} 
        &= \mathrm{ker}(\vP_\Gamma (\vN^\top)^+)^\perp \\
        &= \mathrm{Im} (\vN^+ \vP_\Gamma^\top )
        \tag{$\mathrm{ker}(\vA)^\perp = \mathrm{Im}(\vA^\top)$, for any matrix $\vA$ \citep{axler2015linear}}
        \\
        &= \mathrm{Im} (\vN^+ \vP_\Gamma ) \, .
    \end{align}

    This proves the claim. 
\end{proof}

\textbf{What happens if $\Gamma_\vq \not \subseteq \calM$?}. We discuss the case when the condition in \cref{prop:part-id-lin-rep-tentative} is not met due to $\Gamma_\vq := \mathrm{Im}(\vA_\vq^\top \vP_\Gamma ) \not \subseteq \calM$. 
We show that even if a model $(\vf, \vg)$ linearly represents $\vq$ on $\Gamma$, a $\sim_{EL}$-equivalent model $(\tilde{\vf}, \tilde \vg)$ may not. This is due to the fact that the information contained in $\calF \setminus \calM$, used to relationally represent $\vq$, may not be linearly transformed on another $\sim_{EL}$-equivalent model.  

In fact, when $\Gamma_\vq \not \subseteq \calM$, consider the expression for relational linearity given by \cref{def:linearity} where
\[
    \label{eq:reminder-lin}
    \vP_\Gamma \vf(\vs \cat \vq) = \vP_\Gamma \vA_\vq \vf(\vs) + \vP_\Gamma \va_\vq  \, .
\]
{We take the first term on the RHS of Equation~\eqref{eq:reminder-lin}: by inserting the projector $\vP_{\Gamma_\vq}$, we rewrite it as follows:}
\[  \label{eq:identity-minus-m}
    \vP_\Gamma \vA_\vq \vP_{\Gamma_\vq} \vf(\vs) = \vP_\Gamma \vA_\vq \vP_{\Gamma_\vq} \vP_\calM \vf(\vs) + \vP_\Gamma \vA_\vq \vP_{\Gamma_\vq} (\vI - \vP_\calM) \vf(\vs) \, ,
\]
where we used the identity $\vI = \vP_\calM + (\vI - \vP_\calM)$ to separate the contributions, inside and outside $\calM$. We can thus rewrite  Equation~\eqref{eq:reminder-lin} as:
\begin{align}
     \vP_\Gamma \vf(\vs \cat \vq) &= \vP_\Gamma \vA_\vq \vP_{\Gamma_\vq} 
     \vP_{\calM} \vf(\vs) + \vP_\Gamma \vA_\vq \vP_{\Gamma_\vq} (\vI - \vP_\calM) \vP_\calF \vf(\vs) + \vP_\Gamma \va_\vq
     \tag{Introduce $\vP_\calF \vf = \vf$}
     \\
     &\propto \vP_\Gamma \vA_\vq \vP_{\Gamma_\vq} (\vI - \vP_\calM) \vP_\calF \vf(\vs)
     \tag{$\vP_{\Gamma_\vq} 
     \vP_{\calM} = \mathbf{0}$}
     \\
     &= \vP_\Gamma \vA_\vq \vP_{\Gamma_\vq} (\vP_\calF - \vP_\calM \vP_\calF) \vf(\vs)
     \tag{Multiplying $(\vI - \vP_\calM)\vP_\calF = \vP_\calF - \vP_\calM\vP_\calF$} 
     \\
     &= \vP_\Gamma \vA_\vq \vP_{\Gamma_\vq} (\vP_\calF - \vP_\calM) \vf(\vs) \label{eq:proportionality-non-linear}
\end{align}
where the in the last equation we used that $\vP_\calM \vP_\calF = \vP_\calM$ by \cref{lemma:projectors} (v).
To show that this can lead to non-linearities {implying deviations from relational linearity}, suppose that: 
\begin{align}
    \vf(\vs) &= \vP_\calM \vf(\vs) + (\vP_\calF - \vP_\calM) \vf (\vs)  \\
    &= { \vM \vP_{\tilde{\calM}} \tilde \vf(\vs)}   + (\vP_\calF - \vP_\calM) \tilde{\vf}^2 (\vx) 
\end{align}
where the first term follows from the equivalence relation  Equation~\eqref{eq:linear-transf-f}, \ie $\vP_\calM \vf(\vs) = \vM \vP_{\tilde \calM} \tilde \vf(\vs)$, and we used $\tilde \vf^2(\vx) = (\tilde{f}_1(\vx)^2, \ldots, \tilde{f}_{\tilde d}(\vx)^2)$ to denote the square of the components of $\tilde{\vf}$. Notice that, this choice is allowed since the components of $\vf$ outside $\calM$, \ie those in $\calF \setminus \calM$, can be arbitrarily chosen and do not contribute to the dot-product with $\vg_0$. Therefore, substituting this expression to  Equation~\eqref{eq:proportionality-non-linear} we get:
\[
     \vP_\Gamma \vf(\vs \cat \vq) \propto%
     \vP_\Gamma \vA_\vq \vP_{\Gamma_\vq} (\vP_\calF - \vP_\calM) \tilde{\vf}^2 (\vs)
\]
and substituting $\vP_\Gamma \vf(\vs \cat \vq) = \vP_\Gamma \vN^+ \tilde \vf(\vs \cat \vq)$, implied by the LHS of  Equation~\eqref{eq:prop14-step1} and from the equality in  Equation~\eqref{eq:prop14-step-3}, 
we have that:
\[
    \vP_\Gamma \vN^+ \tilde \vf(\vs \cat \vq) \propto %
    \vP_\Gamma \vA_\vq \vP_{\Gamma_\vq} (\vP_\calF - \vP_\calM) \tilde{\vf}^2 (\vs) \, ,
\]
which shows a non-linear dependence of $\tilde{\vf}(\vs \cat \vq)$ on $\tilde{\vf}(\vs)$, invalidating relational linearity when $\vP_{\Gamma_\vq} (\vP_\calF - \vP_\calM) \neq \mathbf{0} $.

    \subsection{Proof of \cref{prop:no-paral}}
        \parallelismnotpreserved*
        
\begin{proof}
    Given that $\vgamma$ and $\vgamma'$ are parallel in $\calN$ (\cref{def:s-parallelism}), we have that:
    \begin{align}
        \vP_\calN \vgamma = \beta \vP_\calN \vgamma' \, , \label{eq:parallel-in-n}
    \end{align}
    where $\beta \neq 0$ is given by
    $
        \beta = {\norm{\vP_\calN \vgamma}}/{\norm{\vP_\calN \vgamma'}} 
    $ (see also \cref{remark:parallelism}. %
    We consider the components of the $\sim_{EL}$-equivalent model, given by
    \[
        \vP_\calN \vgamma =  \vN \vP_{\tilde \calN} \tilde \vgamma, \quad \vP_\calN \vgamma' =  \vN \vP_{\tilde \calN} \tilde \vgamma' \, .
    \]
    Using this in Equation~\eqref{eq:parallel-in-n} we get:
    \[
        \vN \vP_{\tilde \calN} \tilde \vgamma  = \beta \vN \vP_{\tilde \calN} \tilde \vgamma '
    \]
    and multiplying from the left by the pseudoinverse of $\vN$ we get:
    \begin{align}
        \vN^+ \vN \vP_{\tilde \calN} \tilde \vgamma &= \beta \vN^+ \vN \vP_{\tilde \calN} \tilde \vgamma ' \\
        \vP_{\tilde \calN} \vP_{\tilde \calN} \tilde \vgamma &= \beta  \vP_{\tilde \calN} \vP_{\tilde \calN} \tilde \vgamma ' 
        \tag{Using $\vN^+ \vN = \vP_{\tilde \calN}$}
        \\
        \vP_{\tilde \calN} \tilde \vgamma &= \beta  \vP_{\tilde \calN} \tilde \vgamma '
    \end{align}
    which shows that $\tilde \vgamma$ is parallel to $\tilde \vgamma'$ in $\tilde \calN$. %
    To prove the reverse implication, the same steps can be repeated by symmetry, taking:
    \[
        \vP_{\tilde \calN} \tilde \vgamma =  \vN^+ \vP_{\calN} \vgamma, \quad \vP_{\tilde \calN} \tilde \vgamma' =  \vN^+ \vP_{\calN} \vgamma' \, .
    \]
    for two vectors $\tilde \vgamma$ and $ \tilde \vgamma'$ parallel in $\tilde \calN$, giving the desired result.
\end{proof}

\newpage

\section{Context-query-reply sentences: Corner cases}
\label{sec:properties}

    \subsection{Paraphrases}

    We consider a sentence $\vq_2$ to be the paraphrase of $\vq_1$ when $\vq_2$ repeats what was written in $\vq_1$ using different words (for this definition, we only slightly adapted the one from the \href{https://dictionary.cambridge.org/dictionary/english/paraphrase}{Cambridge Dictionary}).
    To provide an example, let $\vq_1=\textit{``Is the text written in English?''}$. A paraphrase of $\vq_1$ can be $\vq_2=\textit{``Was the previous text written in English or not?''}$. These two equivalent formulations of the same question can nonetheless be treated differently by a next-token predictor: for example, given a string $\vs$, it can be that $\vf(\vs \cat \vq_1) \neq \vf(\vs \cat \vq_2)$. 
    
    Here, we analyze how paraphrastic aspects of textual data can be described with relational context-query-reply ($\vs \cat \vq \cat y$) strings for a model $(\vf, \vg) \in \Theta$. 
    We start by providing a tentative definition 
    of paraphrastic sentences 
    in terms of their entailed conditional log-probabilities for different pairs of next-tokens.

\begin{definition}[Paraphrases]
\label{def:paraphrase-test}
    We say that $\vq_2 \in \SeqA$ on $\calY_2 \subseteq \calA$ is a paraphrase of $\vq_1 \in \SeqA$ on $\calY_1 \subseteq \calA$  for the model $(\vf, \vg) \in \Theta$ if (1) there exists $\beta \neq 0$ such that, for all $y_0, y_1 \in \calY_1 \subseteq \calA$, and (2) there exist $\hat y_0, \hat y_1 \in \calY_2 \subseteq \calA$, for which it holds:
    \[
        \label{eq:paraphrasis-log-prob}
        \log \frac{p_{\vf, \vg} (y_0 \mid \vs \cat \vq_1) }{p_{\vf, \vg} (y_1 \mid \vs \cat \vq_1)} 
        = \beta \cdot
        \log \frac{p_{\vf, \vg} (\hat y_0 \mid \vs \cat \vq_2) }{p_{\vf, \vg} (\hat y_1 \mid \vs \cat \vq_2)} \, .
    \]
\end{definition}
For example,
consider the strings $\vq_1 = \textit{``Is the text written in English?''}$, with expected replies $y_0=\textit{``yes''}$ and $y_1=\textit{``no''}$, \ie $y_0, y_1 \in \calY_1$; 
and a second string $\vq_2=$\textit{``Reply with only A or B. Was the text written in (A) English or (B) not ?''}, with expected replies $\hat y_0=\textit{``A''}$ and  $ \hat y_1=\textit{``B''}$. 
\cref{def:paraphrase-test} entails that, for all input-strings $\vs$, the concatenation to the query $\vq_1$ gives a ratio of the log-probabilities of $y_0$ and $y_1$ that matches, up to a constant $\beta$, that of $\hat y_0$ and $\hat y_1$  for the concatenation to $\vq_2$.
A model that successfully recognizes between English and non-English text and considers $\vq_2$ a paraphrase of $\vq_1$, then will attribute similar conditional probabilities to both $p_{\vf,\vg}(y_0 \mid \vs \cat \vq)$ and $p_{\vf,\vg}(\hat y_0 \mid \vs \cat \vq)$, for any input-context $\vs \in \SeqA$.

We show that sentences and next-tokens as in \cref{def:paraphrase-test} induce a specific structure in the model embeddings, linearly relating the representations $\vf(\vs \cat \vq_1)$ and $\vf(\vs \cat \vq_2)$.
To this end,
we will define $\SIM{\vg_0}_\calY := \mathrm{span}\{\vg(y) - \vg(y_0) \mid \; y_0, y \in \calY \} $ and $\SIM{\vf}_\vq := \mathrm{span}\{
\vf(\vs \cat \vq) \mid \vs \in \SeqA \}$.

\begin{proposition}
    If (1) $\vq_2 \in \SeqA$ on $\calY_2 \subseteq \calA$ is a paraphrase of $\vq_1 \in \SeqA$ on $\calY_1 \subseteq \calA$  for the model $(\vf, \vg)$
    and (2) for the subspaces $\Gamma_1 := \SIM{\vg_0}_{\calY_1}$ and $\Gamma_2 := \SIM{\vg_0}_{\calY_2}$,
    it holds that
    $\Gamma_1 \subseteq \SIM{\vf}_{\vq_1} =: \calF_1$ and $\Gamma_2 \subseteq \SIM{\vf}_{\vq_2} =: \calF_2$, 
    then 
    $\mathrm{dim} (\Gamma_1) = \mathrm{dim} (\Gamma_2)$ and
    there exists a matrix $\vO \in \bbR^{d \times d}$ that defines a linear, invertible transformation from $\Gamma_2$ to $\Gamma_1$ such that
    \[  
        \label{eq:representation-tautology}
        \vP_{\Gamma_1} \vf(\vs \cat \vq_1) = \beta \vO \vP_{\Gamma_2} \vf(\vs \cat \vq_2) \, .
    \] %

\end{proposition}

\begin{proof}
    We start with the equality between logs of probabilities given by \cref{def:paraphrase-test}:
    \begin{align}
        \log \frac{p_{\vf, \vg} (y_1 \mid \vs \cat \vq_1) }{p_{\vf, \vg} (y_0 \mid \vs \cat \vq_1)} 
        &= \beta \cdot
        \log \frac{p_{\vf, \vg} (\hat y_1 \mid \vs \cat \vq_2) }{p_{\vf, \vg} (\hat y_0 \mid \vs \cat \vq_2)} 
        \\
        \log \exp \big(
        (\vg(y_1) - \vg(y_0))^\top \vf(\vs \cat \vq_1)
        \big)
        &= \beta
        \log \exp \big(
        (\vg(\hat y_1) - \vg(\hat y_0))^\top \vf(\vs \cat \vq_2) \big) 
        \tag{Substituting Equation~\eqref{eq:next-token-predictor} for the conditional probabilities}
        \\
        (\vg(y_1) - \vg(y_0))^\top \vf(\vs \cat \vq_1)
        &= \beta
        (\vg(\hat y_1) - \vg(\hat y_0))^\top \vf(\vs \cat \vq_2) \, , 
        \label{eq:step-1-parahprases}
    \end{align}
    where in the last line we removed the logarithm of the exponential on both sides.
    Define $\vg_0(y) := \vg(y) - \vg(y_0)$ for a pivot $y_0 \in \calY_1$ and define $\hat \vg_0(\hat y) := \vg(\hat y) - \vg(\hat y_0)$ for the corresponding pivot $\hat y_0 \in \calY_2$. Then consider $q$ elements $y_i \in \calY_1$ and their correspondents $\hat y_i \in \calY_2$ such that the matrices:
    \[
        \vG := \begin{pmatrix}
            \vg_0 (y_1), \ldots, \vg_0(y_q)
        \end{pmatrix}, 
        \quad
        \hat \vG := \begin{pmatrix}
            \hat \vg_0 (\hat y_1), \ldots, \hat \vg_0 (\hat y_q)
        \end{pmatrix}
        \label{eq:paraphprases-G-Ghat}
    \]
    have rank equal to $\mathrm{dim}(\Gamma_1)$ and $\mathrm{dim}(\Gamma_2)$, respectively. 
    Then, we make use of these matrices with their transpose in Equation~\eqref{eq:step-1-parahprases}, obtaining:
    \begin{align}
        \vG^\top \vf(\vs \cat \vq_1) &= \beta \hat \vG^\top \vf(\vs \cat \vq_2) \\
        (\vG^\top )^+ \vG^\top \vf(\vs \cat \vq_1) &=  \beta (\vG^\top )^+ \hat \vG^\top \vf(\vs \cat \vq_2) 
        \tag{Multiplying on the left for $(\vG^\top)^+$}
        \\
        \vP_{\Gamma_1} \vf(\vs \cat \vq_1) &= \beta \vO \vP_{\Gamma_2} \vf(\vs \cat \vq_2) \label{eq:P1f1-P2Of}
    \end{align}
    where in the last line we used $(\vG^\top)^+ \vG^\top = \vP_{\mathrm{Im}(\vG)} = \vP_{\Gamma_1}$, since $\vG$ in Equation~\eqref{eq:paraphprases-G-Ghat} spans $\Gamma_1$, and we defined $\vO := (\vG^\top )^+ {\hat \vG}^\top$. 
    Proceeding similarly but from $\hat \vG^\top \vf(\vs \cat \vq_2) = \vG^\top \vf(\vs \cat \vq_1)$ we obtain
    a similar equation to Equation~\eqref{eq:P1f1-P2Of}
    :
    \[
        \vP_{\Gamma_2} \vf(\vs \cat \vq_2) = \beta^{-1} \hat \vO \vP_{\Gamma_1} \vf(\vs \cat \vq_1)
        \label{eq:eq:P2f2-P1Of1}
    \]
    where we defined $\hat \vO := (\hat \vG^\top )^+ \vG^\top$. 

    Let $k_1 := \mathrm{dim} (\Gamma_1)$, $k_2 :=\mathrm{dim} (\Gamma_2)$, and $o := \mathrm{rank}(\vO)$.
    Notice that, by definition, $\vO = \vP_{\Gamma_1} \vO$, which means that:
    \begin{align}
        \mathrm{rank}(\vO) &= \mathrm{rank}(\vP_{\Gamma_1} \vO) \\
        o &\leq \min(k_1, o) 
        \tag{Using $\mathrm{rank} (\vA \vB) \leq \min \big( \mathrm{rank}(\vA), \mathrm{rank}(\vB) \big)$}
        \\
        \implies o &\leq k_1
    \end{align}
    Now, consider $\ell$ points, or sequences, $\vx_i \in \SeqA$ such that:
    \[
        \vF_1 := \begin{pmatrix}
            \vf(\vx_1 \cat \vq_1), \ldots, \vf(\vx_\ell \cat \vq_1)
        \end{pmatrix},
        \quad
        \vF_2 := \begin{pmatrix}
            \vf(\vx_1 \cat \vq_2), \ldots, \vf(\vx_\ell \cat \vq_2)
        \end{pmatrix},
    \]
    have rank equal to $\mathrm{dim}(\calF_1)$ and to $\mathrm{dim}(\calF_2)$, respectively. 
    Therefore,by substituting $\vF_1$ and $\vF_2$ in Equation~\eqref{eq:P1f1-P2Of} and in Equation~\eqref{eq:eq:P2f2-P1Of1}, we obtain:
    \begin{align}
         \vP_{\Gamma_1} \vF_1 &= \beta \vO \vP_{\Gamma_2} \vF_2 
         \label{eq:paraphrases-done1}
         \\
         \vP_{\Gamma_2} \vF_2 &= \beta^{-1} \hat \vO \vP_{\Gamma_1} \vF_1 \, .
         \label{eq:paraphrases-done2}
    \end{align}
    Then, by \cref{lemma:rank-equality} it holds that:
    \[
        \mathrm{rank}(\vP_{\Gamma_1} \vF_1) = \mathrm{dim} (\Gamma_1) = k_1, \quad  \mathrm{rank}(\vP_{\Gamma_2} \vF_2) = \mathrm{dim} (\Gamma_2) = k_2 \, .
    \]
    We use this in Equation~\eqref{eq:paraphrases-done1} to obtain the following:
    \begin{align}
        \mathrm{rank}(\vP_{\Gamma_1} \vF_1) &= \mathrm{rank} (\beta \vO \vP_{\Gamma_2} \vF_2 ) \\
        k_1 &\leq \min 
        \big( 
        \mathrm{rank} (\vO), \mathrm{rank}( \vP_{\Gamma_2} \vF_2)
        \big)
        \tag{Using $\mathrm{rank} (\vA \vB) \leq \min \big( \mathrm{rank}(\vA), \mathrm{rank}(\vB) \big)$}
        \\
        k_1 &\leq \min(o, k_2)
        \\
        \implies & o \geq k_1, \quad k_2 \geq k_1 \, .
    \end{align}
    This shows that $k_1 \leq o \leq k_1$, meaning that 
    $\mathrm{rank}(\vO) = \mathrm{dim}(\Gamma_1) $. 
    Similarly, taking the rank of Equation~\eqref{eq:paraphrases-done2} we obtain a similar implication for $k_1$ and $k_2$ (as obtained for $k_2$ and $k_1$ for Equation~\eqref{eq:paraphrases-done1}), that is:
    \begin{align}
        \mathrm{rank}(\vP_{\Gamma_2} \vF_2) &= \mathrm{rank} (\beta^{-1} \hat \vO \vP_{\Gamma_1} \vF_1) \\
        \implies& k_2 \leq k_1 \, ,
    \end{align}
    which shows that $k_1 \leq k_2 \leq k_1$, meaning that $\mathrm{dim}(\Gamma_2) = \mathrm{dim}(\Gamma_1)$. Moreover, the matrix $\vO$, being of rank $o = k_1 = k_2$ defines an invertible transformation from $\Gamma_2$ to $\Gamma_1$. This concludes the proof.
\end{proof}

    \subsection{Tautologies}
    
    Next, we investigate how tautological aspects can be encoded by a model. A tautology in our context can be considered as a context-independent sentence, whose reply to it is not influenced by the previous context. For example, we can consider as $\vq=\textit{``No matter what was written before. Whatever follows reply with $42$!''}$ as a tautology.
    These strings constitute peculiar cases in natural language where the previous input context does not influence the replies to the query $\vq$. Formally:

    \begin{definition}[Tautology]
    \label{def:tautology}
        We say that $\vq \in \SeqA$ is a tautology for the model $(\vf, \vg) \in \Theta$ if, for every $\vs \in \SeqA$ and all $y \in \calA$, it holds:
        \[
            \log p_{\vf, \vg}(y \mid \vs \cat \vq) = \log p_{\vf, \vg}(y \mid \vq) \, .
        \]
    \end{definition}

    Next, {we show that, for such tautologies, a ``trivial'' form of linear relational embedding holds.}

    \begin{proposition}[Tautologies]
    Let $\vq \in \SeqA$ be a tautology for the model $(\vf, \vg)$, then $\vf$ linearly represents $\vq$ on $\calG :=\SIM{\vg_0}$ with:
    \[
        \vP_{\calG}\vf(\vs \cat \vq) = \vP_\calG \va_\vq \, .
        \label{eq:linearity-tautology}
    \]
    
\end{proposition}

\begin{proof}
    From \cref{def:tautology}, consider a pivot token $y_0 \in \calA$ and define $\vg_0 := \vg(y) - \vg(y_0)$ for every $y \in \calA$. We have that:
    \begin{align}
        \log \frac{p_{\vf, \vg}(y \mid \vs \cat \vq)}{p_{\vf, \vg}(y_0 \mid \vs \cat \vq)} &= \log \frac{p_{\vf, \vg}(y \mid \vq)}{p_{\vf, \vg}(y_0 \mid \vq)} \tag{Take log of the ratio between $p_{\vf, \vg}(y \mid \cdot)$ and $p_{\vf, \vg}(y_0 \mid \cdot)$} \\
        \log \frac{\exp \vg(y)^\top \vf(\vs \cat \vq)}{\exp \vg(y_0)^\top \vf(\vs \cat \vq)} &= \log \frac{\exp \vg(y)^\top \vf( \vq)}{\exp\vg(y_0)^\top \vf(\vq)} \tag{Write with the exponentional} \\
        \log {\exp \vg_0(y)^\top \vf(\vs \cat \vq)} &= \log {\exp \vg_0(y)^\top \vf( \vq)} \tag{Use the definition of $\vg_0$} \\
        \vg_0(y)^\top \vf(\vs \cat \vq) &=  \vg_0(y)^\top \vf( \vq) \tag{Remove log and exponential} \\
        \vg_0(y)^\top \vf(\vs \cat \vq) &=  \vg_0(y)^\top \va_{\vq} \, , \label{eq:prop23-step1}
    \end{align}
    where in the last line with denoted with $\va_\vq := \vf(\vq)$. Consider $\ell$ tokens $y_i \in \calA$ such that
    \[
        \vG := \begin{pmatrix}
            \vg_0(y_1), \ldots, \vg_0(y_\ell)
        \end{pmatrix}
    \]
    spans $\calG$. We use this and consider the following expression ofr the transpose:
    \begin{align}
        \vG^\top \vf(\vs \cat \vq) &=  \vG^\top \va_{\vq} \\
        (\vG^\top)^+ \vG^\top \vf(\vs \cat \vq) &=  (\vG^\top)^+ \vG^\top \va_{\vq} \tag{Multiply by pseudo-inverse of $\vG^\top$} \\
        \vP_\calG \vf(\vs \cat \vq) &=  \vP_\calG \va_{\vq} 
    \end{align}
    where we used the fact that $(\vG^\top)^+ \vG^\top = \vP_\calG$. This shows the claim.
\end{proof}

\end{document}